%% file: main_arxiv.tex
\pdfoutput=1
\documentclass{article}
\usepackage[margin=1in]{geometry}
\usepackage[utf8]{inputenc}
\usepackage[T1]{fontenc}
\usepackage{microtype}
\usepackage{graphicx}
\usepackage{adjustbox}
\usepackage{enumitem}
\usepackage{booktabs} %
\usepackage{makecell}
\usepackage{fullpage}
\usepackage{hyperref}
\usepackage{natbib}
\usepackage{authblk}
\usepackage{xcolor}
\usepackage{algorithm}
\usepackage{algpseudocode}
\usepackage{xspace}
\usepackage{multirow} %
\usepackage{subcaption}
\usepackage{placeins}

\usepackage[suppress]{color-edits}
\addauthor{na}{purple!60!black}
\addauthor{mj}{red}

\definecolor{Darkblue}{rgb}{0,0,0.4}
\hypersetup{colorlinks=true,
            citebordercolor={.6 .6 .6},
            linkbordercolor={.6 .6 .6},
            citecolor=Darkblue,
            urlcolor=blue,
            linkcolor=black}

\usepackage{amsmath}
\usepackage{amssymb}
\usepackage{mathtools}
\usepackage{amsthm}
\usepackage{thm-restate}  %

\newcommand{\intagg}{\calA_{\text{intersect}}}
\newcommand{\addagg}{\calA_{\text{add}}}
\newcommand{\addaggfun}{\addagg(\bx^{(1)}, \ldots, \bx^{(K)}; \bw)}

\newcommand{\intaggfun}{\intagg(\bx^{(1)}, \ldots, \bx^{(K)})}

\newcommand{\aggop}{\bx^{(1)}, \ldots, \bx^{(K)} \rightarrow \bx^{(A)}}

\usepackage[capitalize, noabbrev]{cleveref}

\theoremstyle{plain}
\newtheorem{theorem}{Theorem}[section]
\newtheorem{proposition}[theorem]{Proposition}
\newtheorem{lemma}[theorem]{Lemma}
\newtheorem{corollary}[theorem]{Corollary}
\newtheorem{example}[theorem]{Example}
\theoremstyle{definition}
\newtheorem{definition}[theorem]{Definition}
\newtheorem{assumption}[theorem]{Assumption}
\theoremstyle{remark}

\usepackage{nicefrac}
\usepackage{dsfont}

\usepackage{minitoc}

\definecolor{Darkblue}{rgb}{0,0,0.4}

\input{common}

\begin{document}
\doparttoc
\faketableofcontents

\title{Power and Limitations of Aggregation in \\ Compound AI Systems}
\date{}
\author[1]{Nivasini Ananthakrishnan}
\author[2]{Meena Jagadeesan}

\affil[1]{UC Berkeley
}
\affil[2]{Stanford University 
}

\maketitle
\begin{abstract}
When designing compound AI systems, a common approach is to query multiple copies of the same model and aggregate the responses to produce a synthesized output. Given the homogeneity of these models, this raises the question of whether aggregation unlocks access to a greater set of outputs than querying a single model. In this work, we investigate the power and limitations of aggregation within a stylized principal-agent framework. This framework  models how the system designer can partially steer each agent's output through its reward function specification, but still faces limitations due to prompt engineering ability and model capabilities. Our analysis uncovers three natural mechanisms---\textit{feasibility expansion}, \textit{support expansion}, and \textit{binding set contraction}---through which aggregation expands the set of outputs that are elicitable by the system designer. We prove that any aggregation operation must implement one of these mechanisms in order to be elicitability-expanding, and that strengthened versions of these mechanisms provide necessary and sufficient conditions that fully characterize elicitability-expansion. Finally, we provide an empirical illustration of our findings for LLMs deployed in a toy reference-generation task. Altogether, our results take a step towards characterizing when compound AI systems can overcome limitations in model capabilities and in prompt engineering.
\end{abstract}

\input{arxiv_sections/introduction_arxiv}

\input{arxiv_sections/model_arxiv}

\input{arxiv_sections/merged_arxiv}

\input{arxiv_sections/empirical}

\input{arxiv_sections/discussion_arxiv}

\section{Acknowledgments}
We would like to thank Kate Donahue, Nika Haghtalab, Tatsu Hashimoto, Michael I. Jordan, and Manish Raghavan for useful feedback and discussions on the paper.
This work was partially supported by a Stanford AI Lab postdoctoral fellowship. 

\bibliographystyle{plainnat}
\bibliography{ref_arxiv.bib}

\newpage
\appendix
\part{Appendix}
\parttoc
\newpage
\input{arxiv_sections/appendix_arxiv}

\input{arxiv_sections/appendix_empirical_arxiv}

\end{document}

%% file: common.tex
\newcommand{\R}{\mathbb{R}}

\newcommand{\calA}{\mathcal{A}}
\newcommand{\calB}{\mathcal{B}}

\newcommand{\calS}{\mathcal{S}}

\newcommand{\calV}{\mathcal{V}}

\newcommand{\balpha}{\boldsymbol{\alpha}}
\newcommand{\bconstraints}{\boldsymbol{C}}
\newcommand{\bfeatures}{\boldsymbol{F}}

\newcommand{\blambda}{\boldsymbol{\lambda}}

\newcommand{\bgamma}{\boldsymbol{\gamma}}

\newcommand{\bd}{\boldsymbol{d}}

\newcommand{\bu}{\boldsymbol{u}}
\newcommand{\bv}{\boldsymbol{v}}
\newcommand{\bw}{\boldsymbol{w}}
\newcommand{\bx}{{\boldsymbol{x}}}
\newcommand{\by}{\boldsymbol{y}}

   \newcommand{\reals}{\mathbb{R}}
   \newcommand{\preals}{\mathbb{R}_{\ge 0}}

\newcommand{\Bcal}{{\mathcal B}}

\renewcommand{\hbar}{{\bar{h}}}

%% file: arxiv_sections/introduction_arxiv.tex
\section{Introduction}

Compound AI systems---which leverage multiple AI components, rather than a single model in isolation---present a powerful paradigm to tackle complex tasks \citep{BAIR2024CompoundAISystems}. In the context of large language models (LLMs), one common approach is to create many copies of the same model, give these models different prompts or access to different tools, and aggregate the outputs of these models at test-time. This approach has proven fruitful in multi-agent research systems \citep{Anthropic2025MultiAgentSystem} where a lead LLM agent delegates subtasks to different specialized agents and aggregates their outputs, in multi-agent debate protocols where different LLM agents seek consensus \citep{Du2024MultiagentDebate} or argue for different answers \citep{Khan2024PersuasiveDebate}, and in prompt ensembling approaches where the outputs from different prompts are combined \citep{Arora2023AMA}.

Given the empirical success of these compound LLM systems, this raises the question of when aggregating across multiple copies of the same model unlocks greater performance than querying a single model. At first glance, aggregation may seem redundant when the model copies are homogeneous. However, one source of improved performance is at the prompting level: complex prompt engineering approaches for a single model may be replaceable by simple but diverse prompting strategies for a set of models \citep{Arora2023AMA}, illustrating how aggregation across models can overcome limitations in prompt engineering ability. Another source of improved performance is at the output level: aggregating multiple LLM agents over multiple interactions can help correct errors such as hallucinations \citep{Du2024MultiagentDebate}, illustrating how aggregation can overcome limitations in model capabilities as well. This suggests that the extent to which aggregation overcomes these limitations in prompt engineering and model capabilities fundamentally impacts the power of compound AI systems. 

In this work, we study the power and limitations of aggregation from a theoretical perspective, building on a classical principal-agent framework \citep{Kleinberg2019Effort}. Our focus is on compound AI systems where a system designer passes reward functions (e.g., via prompts) to many copies of the same model and then aggregates their outputs. For intuition, consider a reference-generation task  where the goal is to produce a list of paper references; the system designer may benefit from prompting different copies of an LLM to focus on different topics, and then taking the union or intersection of the resulting reference lists.

In our stylized principal-agent framework (Section \ref{sec:model}), the system designer (i.e., the principal) designs reward functions to elicit $M$-dimensional outputs from each agent (i.e., each model), and aggregates these outputs to produce a synthesized output. Each agent generates an output in its feasible set that maximizes the reward function. The system designer co-designs the reward functions across models to try to produce a specific output. We capture prompt engineering limitations as the reward functions operating over a coarser $N$-dimensional feature space rather than directly over $M$-dimensional outputs, and model capability limitations as conic constraints on each agent's feasible set of outputs. 

\begin{figure*}[t]
    \centering
    \begin{subfigure}[t]{0.32\textwidth}
        \centering
        \includegraphics[width=\textwidth]{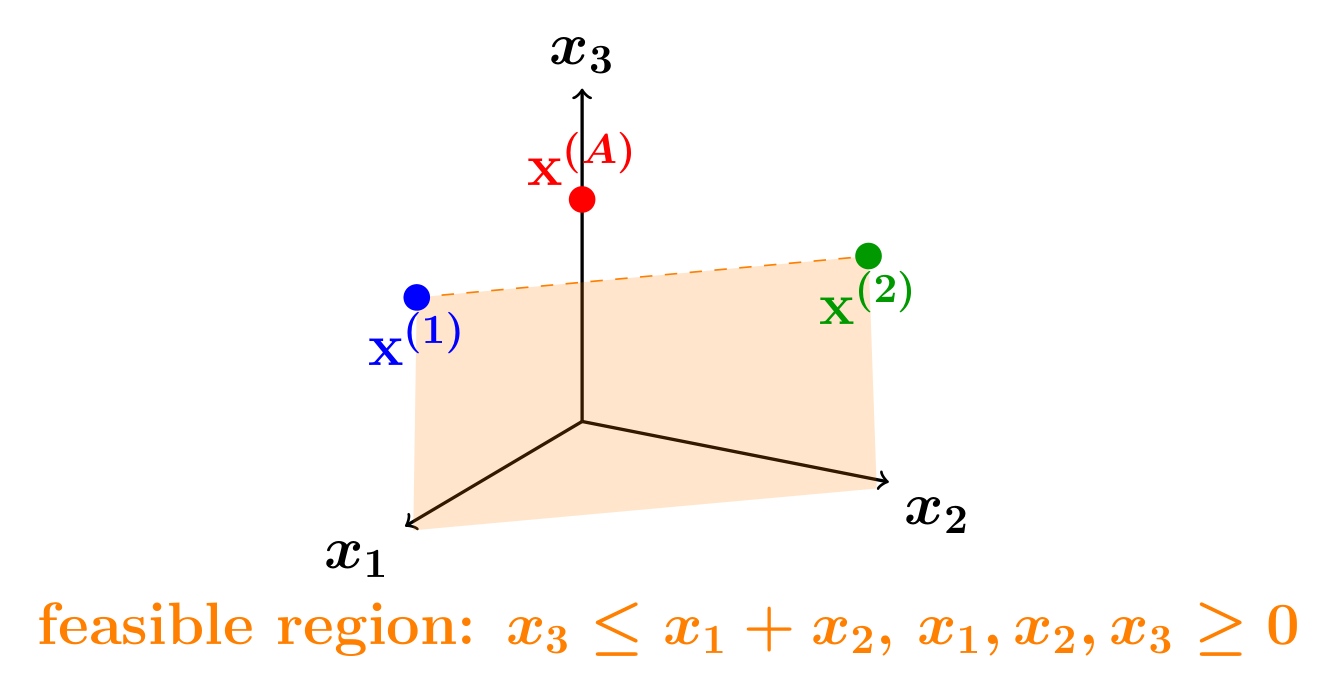}
        \caption{Feasibility expansion}
        \label{fig:feasibility-expansion}
    \end{subfigure}
    \hfill
    \begin{subfigure}[t]{0.32\textwidth}
        \centering
        \includegraphics[width=\textwidth]{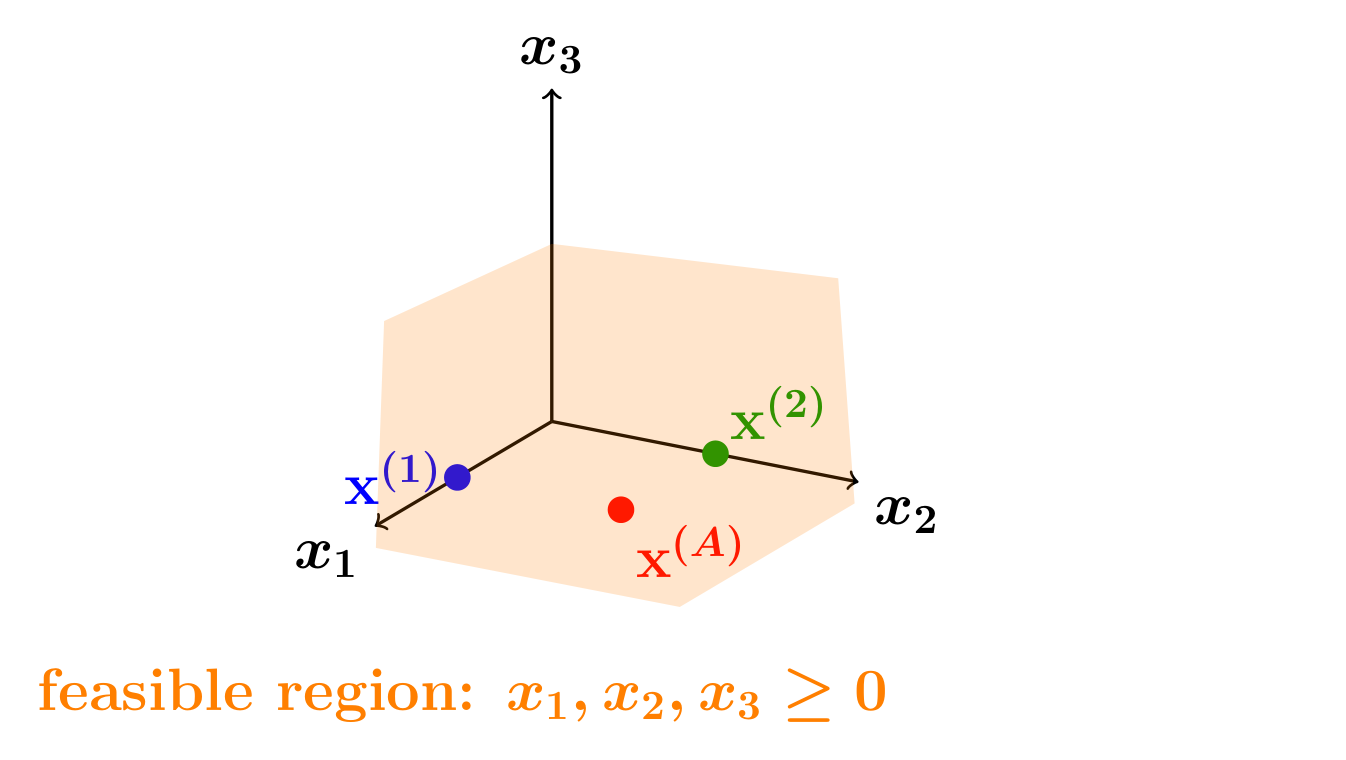}
        \caption{Support expansion}
        \label{fig:support-expansion}
    \end{subfigure}
    \hfill
    \begin{subfigure}[t]{0.32\textwidth}
        \centering
        \includegraphics[width=\textwidth]{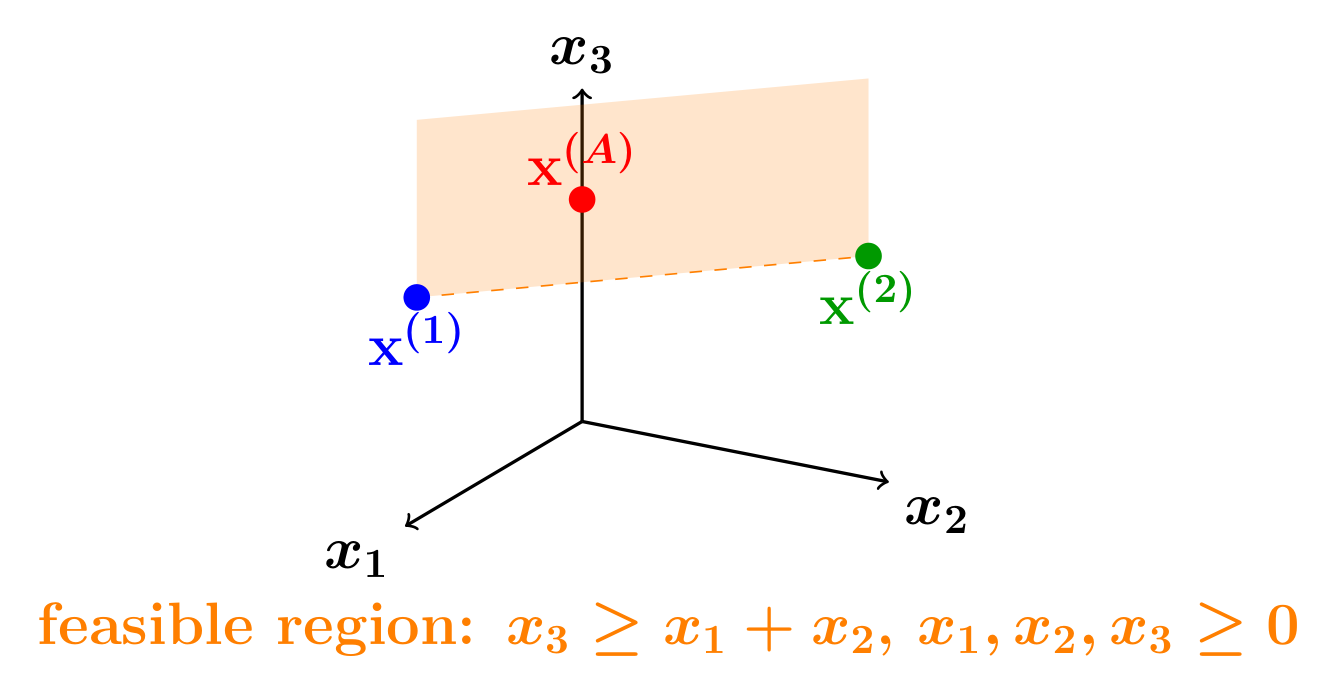}
        \caption{Binding set contraction}
        \label{fig:binding-set-contraction}
    \end{subfigure}

    \caption{Three mechanisms by which the aggregation operation $\bx^{(1)}, \bx^{(2)} \rightarrow \bx^{(A)}$ expands the set of outputs that the system designer can elicit. Feasibility expansion captures when two feasible vectors are aggregated into an infeasible vector (left; \Cref{def:feasibility_expand}). Support expansion captures when two vectors are aggregated into a vector with richer support (middle; \Cref{def:support_expand}). Binding set contraction captures when two vectors on the boundary of the feasible set are aggregated into a vector in the interior (right; \Cref{def:bind_contract}). Any aggregation operation must implement one of these mechanisms to offer power to the system designer (Theorem \ref{thm:weaker_necessary}), and strengthened versions of these mechanisms characterize when aggregation adds power (Theorem \ref{thm:necessary}, Theorem \ref{thm:sufficient}). See Figure \ref{fig:empirical-mechanisms} for an empirical illustration of these mechanisms for LLMs in a reference-generation task.}
    \label{fig:elicitability-modes}
\end{figure*}\label{fig:mechanisms_examples}

Using our theoretical framework, we formalize three natural mechanisms by which aggregating across multiple agents enables the system designer to elicit a greater set of outputs than relying on a single model (Figure~\ref{fig:mechanisms_examples}). The first mechanism is \textit{feasibility expansion}, where aggregation produces outputs outside of any agent's feasibility set. The second is \textit{support expansion}, where aggregation combines outputs with smaller supports into an output with a richer support. The third is  \textit{binding set contraction}, where aggregation combines outputs that are binding with respect to constraints into an output that falls within the interior. 

We then formally connect these mechanisms to elicitability-expansion. Specifically, we find that the power of aggregation fundamentally relies on at least one of these mechanisms being implemented: if none are implemented, then aggregation does not expand elicitability on any problem instance (Theorem \ref{thm:weaker_necessary}). This illustrates a strong form of limitation where the aggregation operation offers no power to the system designer regardless of the level of prompt engineering limitations. However, we also show that these mechanisms do not provide sufficient conditions for elicitability-expansion. 
To close this gap, we prove that strengthened versions of the mechanisms fully characterize elicitability-expansion (Theorem \ref{thm:sufficient} and Theorem \ref{thm:necessary}).

Finally, to connect our findings to LLM aggregation, we empirically illustrate the three mechanisms for LLMs (GPT-5.4, GPT-5-mini, GPT-4o-mini) deployed in a toy reference-generation task (Figure \ref{fig:empirical-mechanisms}; Table \ref{tab:empirical_results}). The system designer aims to generate a list that reflects a specific balance of topics, but can only prompt the model using higher-level topics. Within this task, we empirically construct aggregation operations that implement each mechanism (feasibility expansion, support expansion, and binding-set contraction), and that also expand the set of outputs that are elicitable to the system designer. This empirical analysis illustrates the robustness of our findings beyond the assumptions of our theoretical framework.

Altogether, our work uncovers key mechanisms that underpin the power and limitations of aggregation in compound AI systems. Our results fully characterize when aggregation expands the set of elicitable outputs. Our results offer conceptual insights about aggregation for system designers (Section \ref{subsec:conceptualinsights}), and our mechanisms connect to empirical phenomena observed for language models (Section \ref{sec:discussion}).
More broadly, our work takes a step towards understanding when aggregation of multiple copies of the same model provides benefits to system designers. 

\subsection{Related Work}

\paragraph{Aggregation across multiple models.} Aggregating outputs from multiple copies of an LLM is a common strategy for tackling complex tasks. One common approach is resampling the same model or reasoning trace and then selecting outputs via reward models \citep{Christiano2017RLHF}, self-consistency \citep{Wang2023SelfConsistency}, or synthesis \citep{Zhang2025CoTSynthesizer}; coverage is an important property for inference-time computations \citep{huang2025best}. Closest to our setting are systems with multiple copies of the same model under different reward function specifications, as in LLM debate \citep{Du2024MultiagentDebate}, consensus games between generators and discriminators \citep{Jacob2024ConsensusGame}, prompt ensembling \citep{Arora2023AMA}, and multi-agent research frameworks \citep{Anthropic2025MultiAgentSystem}. Another related approach is to combine many different LLMs, for example by routing queries across different LLMs \citep{Chen2024FrugalGPT} or adversarially combining models to generate unsafe outcomes \citep{Jones2024AdversariesCombinations}.  Motivated by these systems, we study the fundamental conceptual question of when aggregating multiple models elicits strictly more outputs than querying a single model. 

The conceptual question that we study is inspired by a rich literature on aggregation in a variety of other domains. This includes ensembling \citep{Dietterich2000Ensembles}, voting \citep{Ladha1992CondorcetCorrelation}, distributed algorithms \citep{Lynch1996DistributedAlgorithms}, and multi-agent reinforcement learning \citep{Tan1993MARL}, among other domains. 
For example, some works (e.g., \citep{gentzkow2017bayesian, collina2025emergent, bansal2021does, donahue2022human}) demonstrate the benefits of aggregating heterogeneous agents in achieving alignment. However, these works focus on a prespecified set of distinct agents, whereas we focus on eliciting different behaviors from the same agent through designing different rewards. Turning to market-level aggregation coming from users choosing between different models, 
another line of work  studies when model-providers are incentivized to train heterogeneous models \citep{ben2019regression,jagadeesan2023improved,raghavan2024competition,
xu2025heterogeneous, einavR25, Donahue2026OutputAggregation}. As another example, a growing line of work on alignment under plurality uses principles from voting rules and social choice to design methods for aggregating diverse human preferences and training models to optimize this aggregated objective (e.g., \citep{conitzer2024social,mishra2023ai, dai2024mapping, chidambaram2025direct,shirali2025direct,golz2025distortion}).

\paragraph{Principal-Agent Models and Reward Design.}
Our model extends the principal-agent framework introduced by \citet{Kleinberg2019Effort}. This model was originally developed in a strategic classification setting where an agent takes strategic efforts to improve their classification outcome. Within this framework, they characterize the type of effort that classifiers induce in agents, identifying when an agent is incentivized to exert efforts that genuinely improve their outcomes rather than merely gaming the classifier. We extend their model in two ways to study how aggregating outputs from models can overcome prompting and output limitations. First, the principal interacts with multiple agents at once and aggregates the outputs produced by different agents. Second, agents face conic constraints on their outputs, which captures output limitations. Our goal also differs: we introduce and characterize a novel notion of elicitability-expansion to capture the power afforded by aggregation for some level of prompting limitation.\footnote{ The main technical lemma (\Cref{thm:fixed_feature_map}) underlying our characterization of elicitability-expansion extends Theorem 3 of their paper—which characterizes the elicitability of a given effort profile—to the setting with additional conic constraints on agent outputs. Their characterization reveals a single-agent limitation: effort profiles with larger supports are difficult to elicit (Lemma 10 of \cite{Kleinberg2019Effort}). This limitation underlies support expansion, one of the three mechanisms through which we show aggregation overcomes single-agent limitations.} We note that \citet{alon2020multiagent}  also generalizes \citet{Kleinberg2019Effort} to multiple agents, but their model and motivation differ significantly from ours. In their model, different agents have different mappings from features to outputs but share the same reward function specification, and they aim to design a single reward function specification that incentivizes agents to all choose an output vector with a specific structure.

These models fall under the broader principal-agent framework~\citep{holmstrom1979, grossman1983, laffont2002, bolton2005}, which captures the challenge of designing rewards based on imperfect proxies. \citet{zhuang2020consequences} use this framework to study misalignment between the AI agent's reward and the human's reward, focusing on the impact of underspecification. In addition to the effects of imperfect reward functions arising from limitations in reward design, principal-agent frameworks also capture agent limitations, often in the form of costs for actions. Similar to the interdependence of output dimensions induced by our conic constraints, multi-task principal-agent settings~\citep{holmstrommilgrom1991, slade1996, bondgomes2009} study the effects of cost dependencies between tasks, which lead to phenomena such as substitutability and complementarity. Substitutability occurs when effort on one task increases the marginal cost of effort on another, while complementarity occurs when effort on one task decreases it. These dependencies among tasks are similar to the interactions among output dimensions captured by conic constraints in our model. Principal–agent theory has also considered multiple agents~\citep{holmstrom1982teams, lazearrosen1981, dasarathaGolubShah2024}, focusing mainly on the joint design of rewards, but our work differs in considering the use of aggregation to synthesize new outputs. Finally, in a related but complementary literature on strategic classification, a handful of papers \citep{hossain2024strategic,liu2022strategic} study strategic interactions between multiple agents.

%% file: arxiv_sections/model_arxiv.tex
\section{Model}\label{sec:model}

We extend the principal-agent framework in  \cite{Kleinberg2019Effort} to capture a compound AI system with $K$ agents (who represent LLMs) and a single principal (the system designer) who aggregates the outputs of the agents. The outputs are represented in $M$-dimensions, and we capture model capability limitations as conic constraints $\bconstraints$ over the output space (Section \ref{subsec:outputs}). 
The system designer designs reward function specifications and budget levels in order to elicit outputs from each agent (Section \ref{subsec:reward}). The reward function operates over a coarser feature space defined by $\balpha$, which captures prompt engineering limitations. Within this framework, we formalize when aggregation is elicitability-expanding (i.e., when it expands the set of outputs that the system designer can elicit), which captures the power of aggregation (Section \ref{subsec:elicitability}). 
As a case study, we instantiate our framework in a reference-generation task (Section \ref{subsec:casestudy}) which we will revisit in our empirical analysis in Section \ref{sec:empirical}. 
We defer a discussion of the  limitations of our framework to Section \ref{sec:discussion}.

\subsection{Output space}\label{subsec:outputs}

We embed outputs of agents into $M$-dimensional vectors with non-negative coordinates. We view each output dimension as capturing a different characteristic of the output. The vector representation $\bx$ quantifies the degree to which the output captures each characteristic. We note that some dimensions may capture undesirable characteristics (e.g., hallucinations). The system designer seeks a specific output $\bx^{(A)} \in \mathbb{R}_{\ge 0}^M$.

\paragraph{Model capability limitations as conic constraints.}
Our framework can capture restrictions on the outputs agents produce, for example due to capability limitations. We study restrictions requiring output vectors to satisfy conic constraints. These conic constraints capture the types of outputs that the agent can produce: for example, some agents may not be able to avoid producing hallucinations without facing capability degradation along other characteristics. 

We let $L$ denote the number of conic constraints, and we let $\bconstraints \in \mathbb{R}^{L \times M}$ denote the conic constraints themselves. Let $\bconstraints_i \in \mathbb{R}^M$ denote the $i$th row of $\bconstraints$ for $i \in [L]$, and let $\bconstraints_V \in \mathbb{R}^{|V| \times M}$ denote the set of rows corresponding to indices $V \subseteq [L]$. We denote by $\bconstraints_{\emptyset}$ the zero-vector, to capture how  $\{\bd: C_\emptyset \bd \le 0, \bd \ge 0 \} = \preals^M$. This restriction defines a feasible set $\mathcal{B}^{\mathrm{feasible}}(\bconstraints)$ of output vectors:
\[ \mathcal{B}^{\mathrm{feasible}}(\bconstraints) := \left\{\bx \in \mathbb{R}^M_{\ge 0} \mid \bconstraints \bx \le \mathbf{0} \right\}. \]

We assume that membership in $\mathcal{B}^{\mathrm{feasible}}(\bconstraints)$ does not implicitly require any output dimension to always be zero. This assumption on $\bconstraints$ is stated below.  
\begin{assumption}\label{ass:non-zero-feasible}
    We assume that given any output dimension $i \in [M]$, there exists an output vector $\bx \in \mathbb{R}^M_{\ge 0}$ satisfying $x_i > 0$ and $\bconstraints \bx \le 0$.
\end{assumption}

\subsection{Reward and Budget Specification}\label{subsec:reward}

The system designer designs a reward function specification $R^{(k)}$ and a budget level $E^{(k)}$ for each agent $k \in [K]$. The reward function specification represents the reward function implicit in the prompt given to the agent, and the budget level represents the level of test-time compute that the agent is allowed to use. 

\paragraph{Prompt engineering limitations as coarser features.}
To capture prompt engineering limitations, we model the reward function specification as operating over a coarser $N$-dimensional feature space than the outputs. We adopt the same form of how output vectors map to features as in~\cite{Kleinberg2019Effort}. Here, the features $\bfeatures(\bx) = [F_1(\bx), \ldots, F_N(\bx)]$ take the form 
\[F_j(\bx) = f_j \left (\sum_{i=1}^M \alpha_{ji} \bx_i \right ).\]

We call the $\alpha_{ji}$'s \emph{feature weights}. We will denote by $\balpha \in \mathbb{R}_{\ge 0}^{N \times M}$ the matrix with entries $\alpha_{ji}$ and call this the \emph{feature weights matrix}. We borrow the following assumptions on the feature mapping functions $F_j$, $j \in [N]$ from~\cite{Kleinberg2019Effort}. 

\begin{assumption}[Assumptions on feature mapping]\label{ass:feature-maps}
    We assume that each $f_j(\cdot)$ for $j \in [N]$ is strictly increasing, nonnegative, smooth, and weakly concave (i.e., diminishing returns from increasing quality on this dimension). We also assume each   $\alpha_{ji} \ge 0$, for $i \in [M], j \in [N]$. Finally, we assume that $\balpha$ has no zero rows, which rules out trivial features. 
\end{assumption}

\paragraph{Reward functions.}
We consider reward functions $R^{(1)}, \ldots, R^{(K)}: \mathbb{R}^N \rightarrow \mathbb{R}$ which operate on the features $(F_j)_{j=1}^N$. Following prior work \citep{Kleinberg2019Effort}, we restrict to \emph{monotone} reward functions $R$ satisfying the notion of monotonicity stated below.

\begin{assumption}[Monotonicity of reward functions]
    We assume reward functions $R$ are monotone. That is, $R$ does not decrease if all features are weakly increased i.e., if $F_j(\bx') \ge F_j(\bx)$ for every $j \in [N]$, then $R(x') \ge R(x)$.  Additionally, there exists a feature $F_j$ such that increasing the value of $F_j$, keeping all other features fixed, strictly increases the value of $R$. 
\end{assumption}

\paragraph{Agent optimization program.}
Given a monotone reward function $R^{(k)}$ and a positive budget level $E^{(k)} > 0$, each agent $k$ produces an output that maximizes its reward $R^{(k)}$, meets its budget constraint, and is within the feasible set $\mathcal{B}^{\text{feasible}}(\bconstraints)$:
that is,
\[ \smash{\bx \in \mathbf{X}^*(R^{(k)}, E^{(k)}; \bconstraints, \bfeatures) := \text{argmax}_{\bx \in \mathcal{B}^{\text{feasible}}(\bconstraints), \|\bx\|_1 \le E^{(k)}} R^{(k)}( \bfeatures(\bx)).} \]
 This captures how even though agents are homogeneous and solve the same optimization program, they can be given different reward functions and thus produce different outputs.

\subsection{Elicitability}\label{subsec:elicitability}

Given constraints $\bconstraints$ and features $\bfeatures$, we say that an output $\bx$ is elicitable by a single agent if there exist a monotone reward function $R$ and a positive budget level $E$ such that $\bx \in \mathbf{X}^*(R, E; \bconstraints, \bfeatures)$. 
As shown in prior work \citep{Kleinberg2019Effort} and illustrated in Section \ref{subsec:examples}, some output vectors $\bx \in \mathcal{B}^{\mathrm{feasible}}(\bconstraints)$ are not elicitable by any reward function $R$ and budget level $E$. We note that the reward function $R$ determines whether the direction $\bx / \|\bx\|_1$ is elicitable, and the specified budget $E$ determines the $\ell_1$ norm $\|\bx\|_1$ of elicitable outputs. Specifically, increasing the budget scales up the $\ell_1$ norm of the elicitable outputs, but it does not change the set of elicitable directions (\Cref{lem:only_directions}).

As we will show later in our analysis, the condition for whether $\bx$ is elicitable only depends on $\bx$ through the following sufficient statistic $(\calS(\bx),\calV(\bx))$.  The first component $\calS(\bx) = \{i \in [M]: x_i > 0\}$ denotes the support of $x$. The second component $\calV(\bx) = \{l \in [L]: C_l \bx = 0\}$ denotes the set of indices of conic constraints that are binding at $\bx$.

\paragraph{Elicitability-expansion.}
When the system designer can aggregate the outputs of different agents, this may expand the set of elicitable outputs. That is, outputs that are not directly elicitable by specifying reward functions and budgets may be obtained by aggregating a set of outputs that are directly elicitable. The following definition captures when aggregation expands the set of elicitable outputs. 
\begin{definition}[Elicitability-expansion]
We call $\bx^{(1)} \ldots, \bx^{(K)} \rightarrow \bx^{(A)}$ an \textbf{elicitability-expanding operation} \emph{relative to constraints $\bconstraints$ and features $\bfeatures$} if
\begin{itemize}
    \item There exist monotone reward functions $R^{(1)}, \ldots, R^{(K)}$ and positive budget levels $E^{(1)}, \ldots, E^{(K)}$ such that $\bx^{(k)} \in X^*(R^{(k)}, E^{(k)}; \bconstraints, \bfeatures)$ for all $k \in [K]$.  
    \item There does not exist a monotone  reward function $R$ and budget level $E > 0$ such that $x^{(A)} \in X^*(R, E; \bconstraints, \bfeatures)$.
\end{itemize}

We say that $\aggop$ is an elicitability-expanding operation relative to constraints $\bconstraints$ if there \emph{exist} features $\bfeatures$ such that $\aggop$ is elicitability-expanding relative to $\bconstraints$ and $\bfeatures$.
\end{definition}

We only consider aggregation operations $\aggop$ where each $\bx^{(k)}$, for $k \in [K]$ is feasible. Other operations are clearly not useful to the system designer. Intuitively, if an aggregation operation is elicitability-expanding, then there is a prompt engineering limitation under which this operation offers power. Namely, this operation allows the system designer to query the model using multiple distinct reward and budget specifications and put together the resulting output vectors to obtain a new output that is not directly elicitable by specifying a single reward and budget to the model. When an aggregation operation is not elicitability-expanding it is not useful for \emph{any} form of prompt engineering limitation. 
Our results (\Cref{thm:fixed_feature_map}) show that the elicitable set under features $\bfeatures$ depends on $\bfeatures$ only through the feature weights matrix $\balpha$. So we will denote the elicitability set and conditions for elicitability-expansion in terms of $\balpha$ instead of features $\bfeatures$.

\paragraph{Aggregation rules.} An aggregation rule is a mapping from a list of output vectors $(\bx^{(1)}, \ldots, \bx^{(K)})$ to an aggregated output vector $\bx^{(A)}$. There are two natural  aggregation rules we will often consider in our work. Although our results apply to more general aggregation rules, we will often use these natural aggregation rules to provide examples. 

The first is \textit{intersection aggregation}, which is defined to be the coordinate-wise minimum of the vectors: 
\begin{equation}
\label{eq:intersectionaggregation}
\smash{\intaggfun = \bx^{(1)} \wedge \ldots \wedge \bx^{(K)}}.
\end{equation}
This aggregation rule combines outputs based on commonality among different output vectors, which is conceptually similar to debate protocols \citep{Du2024MultiagentDebate} that aim to create agreement or inference scaling methods that aim to filter out incorrect information \citep{Zhang2025CoTSynthesizer}. The second is \textit{addition aggregation}, which takes a weighted sum of the vectors. For a weight vector $\bw \in \preals^K$, the rule is given by
\begin{equation}
\label{eq:additionaggregation}
\addaggfun = \sum_{k=1}^K \bw_k \bx^{(k)}.
\end{equation}
Addition aggregation interpolates among different output directions. This rule conceptually captures how system designers synthesize multiple outputs to 
delegate specialized subtasks to each agent and synthesize the outputs of these subtasks \citep{BAIR2024CompoundAISystems, Anthropic2025MultiAgentSystem}.

\input{arxiv_sections/examples}

%% file: arxiv_sections/examples.tex
\subsection{Illustrative Example: Reference-Generation Task}\label{subsec:casestudy}

To illustrate how our framework captures LLM aggregation, we consider a natural task: generating a list of papers on a given topic ~\citep{wang2023scibench, press2024citeme}. We will also revisit this task in our empirical analysis in Section \ref{sec:empirical}. At a high-level, the system designer may benefit from asking many copies of the same LLM to produce a reference list, and then aggregating these lists into a final list. For example, to increase the breadth of the final list, the system designer may prompt each LLM to focus on topics from a specific subarea and then take a union of the lists. Alternatively, to reduce hallucinations or remove irrelevant references, the system designer may prompt each LLM to focus on different measures of paper quality, and then take an intersection of the lists. 

Below, we show how different instantiations of our framework capture different aspects of this task. 
\begin{figure}[t]
    \centering
    
    \begin{subfigure}[b]{0.32\textwidth}
        \centering
        \includegraphics[width=\linewidth, height=0.4\textheight, keepaspectratio]{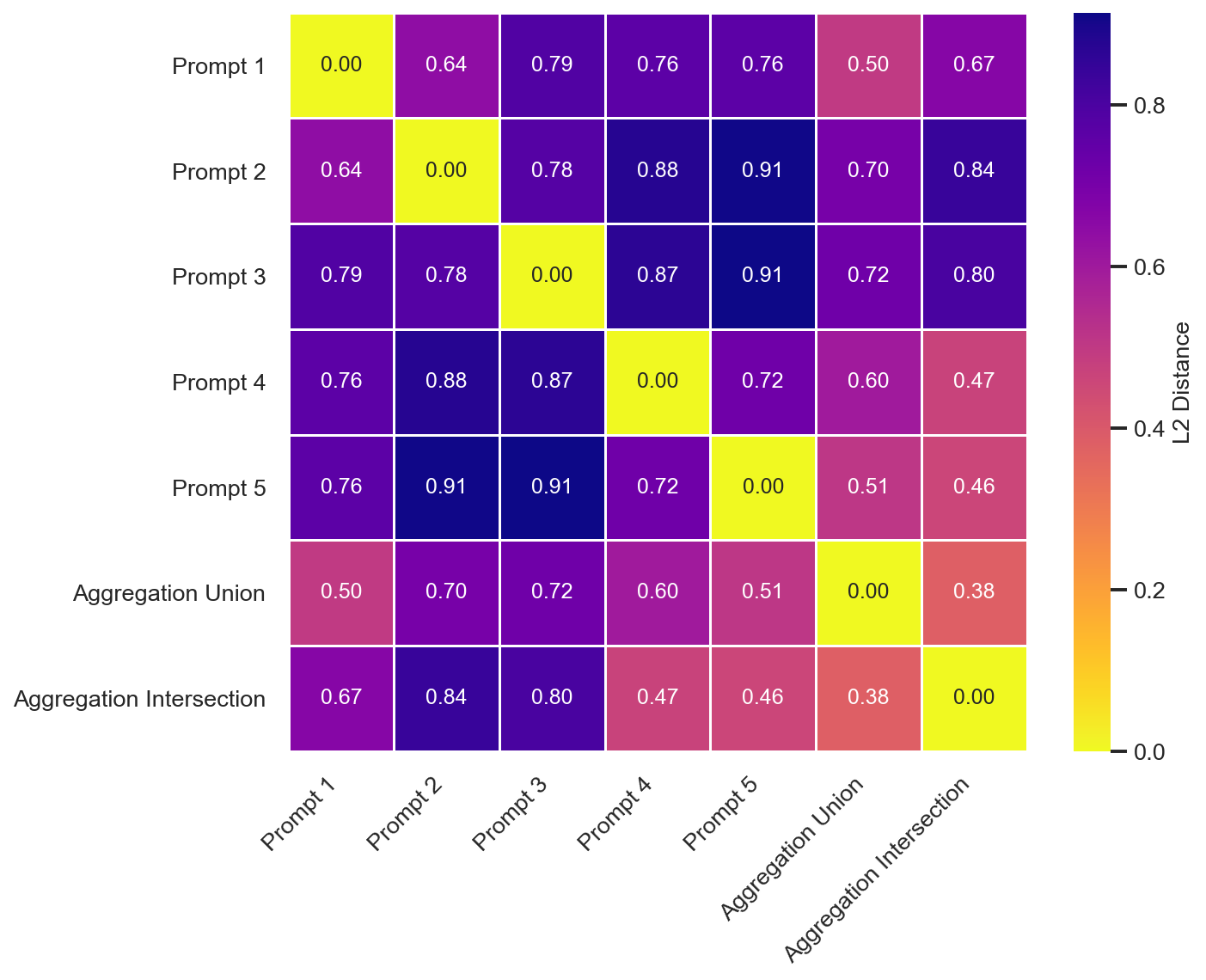}
        \caption{$\ell_2$ distances}
        \label{fig:cosine-sim}
    \end{subfigure}%
    \hfill
     \begin{subfigure}[b]{0.32\textwidth}
        \centering
        \includegraphics[width=\linewidth, height=0.4\textheight, keepaspectratio]{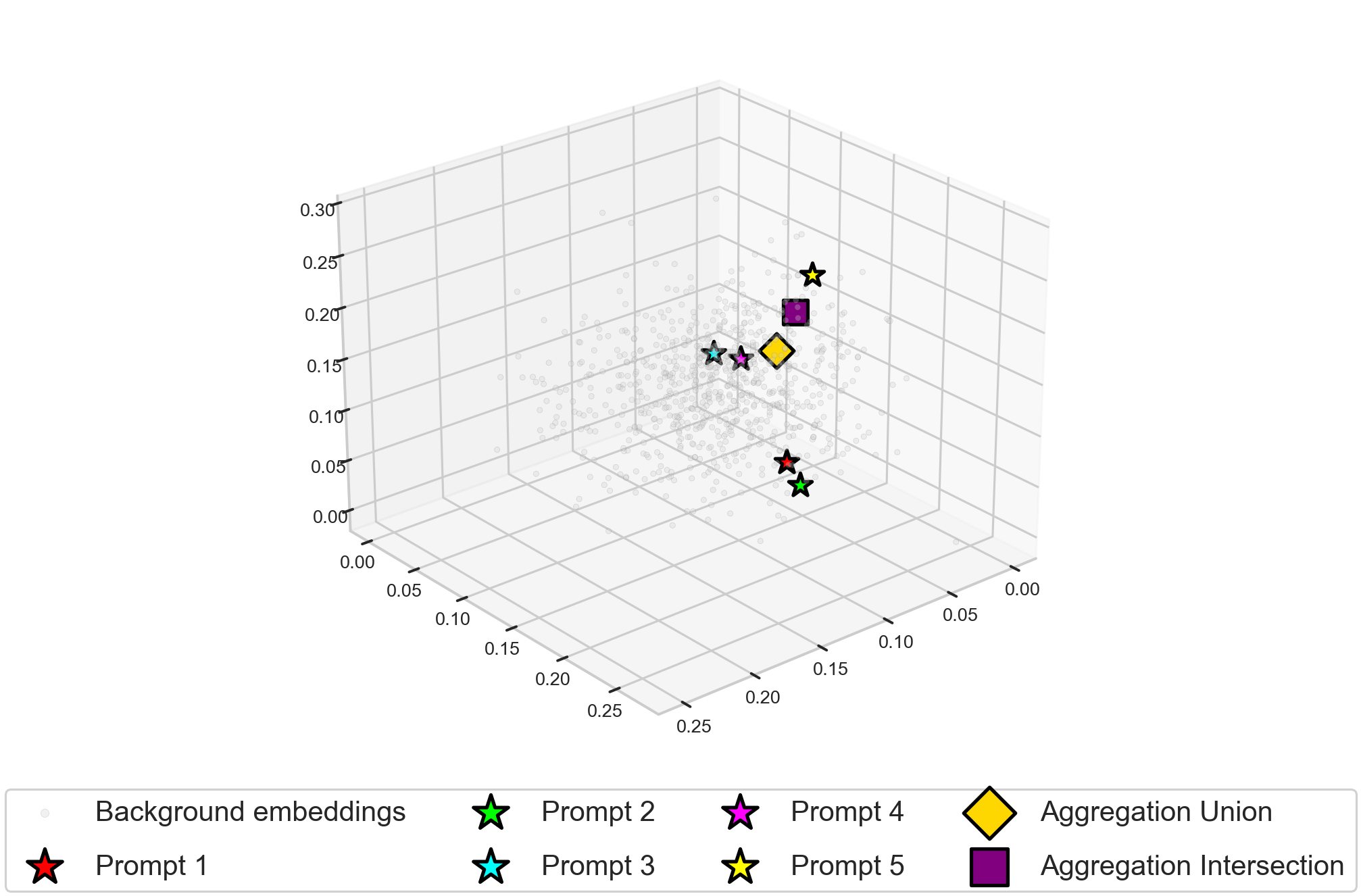}
        \caption{Projection onto top 3 dims}
        \label{fig:top3-task}
    \end{subfigure}%
    \hfill
    \begin{subfigure}[b]{0.32\textwidth}
        \centering
        \includegraphics[width=\linewidth, height=0.4\textheight, keepaspectratio]{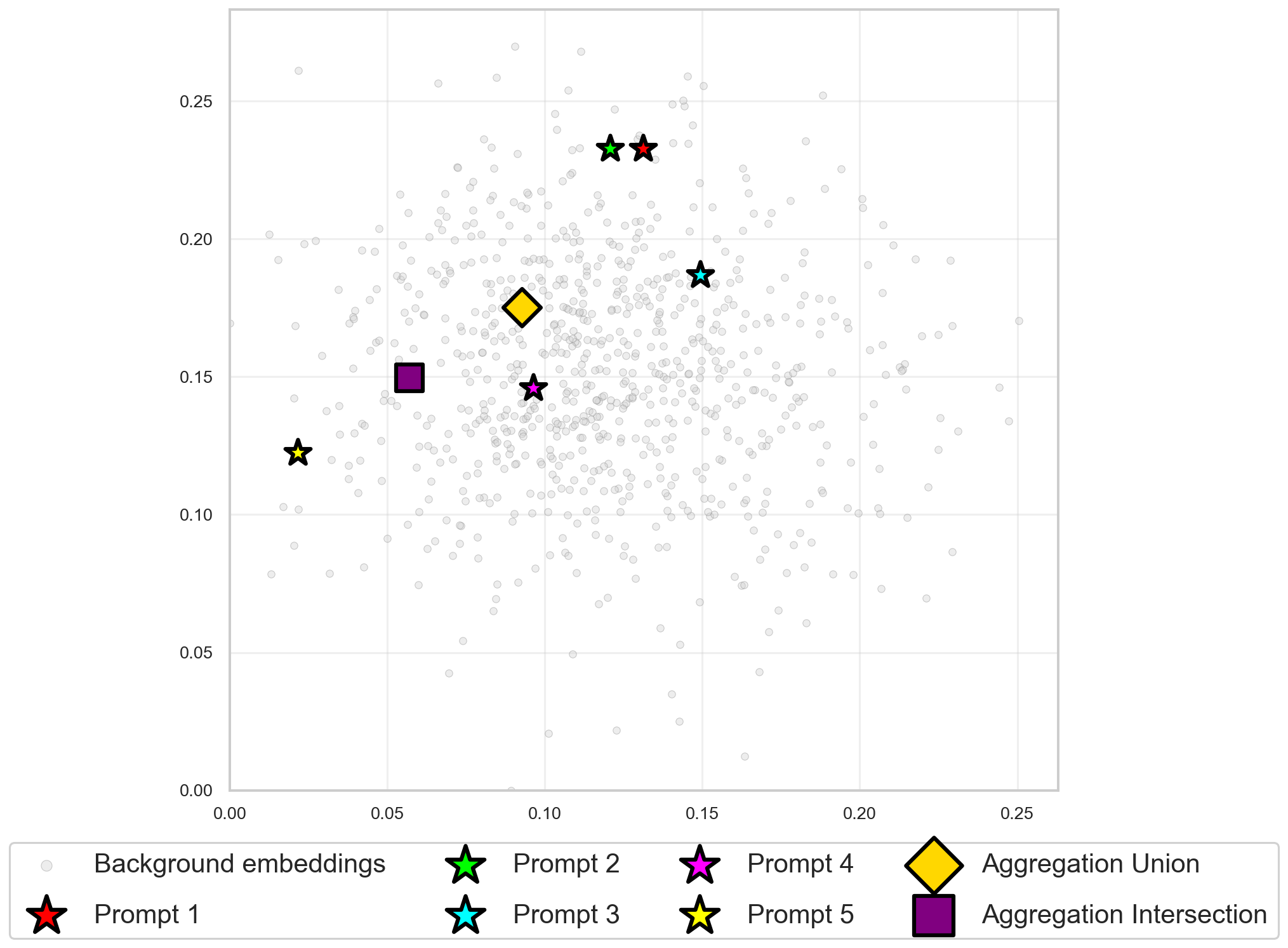}
        \caption{Projection onto top 2 dims}
        \label{fig:top2-task}
    \end{subfigure}

    \caption{{Visualization of output vectors for a reference-generation task (Section \ref{subsec:casestudy}). Output vectors are computed using the $M=768$-dimensional embeddings from all-mpnet-base-v2, shifted to be in the nonnegative orthant. Embeddings are shown for GPT-4o-mini outputs from five different prompts, and as well as two different aggregated outputs based on additional-style and intersection-style aggregation rules. The $\ell_2$-distances (left) and projections onto the top 3 highest-variance (middle) and top 2 highest-variance dimensions (right), are shown. The plots show that the five prompts produce semantically different outputs, and each aggregation operation results in a combination of the five outputs that does not resemble any output in isolation.}}
    \label{fig:embedding-grid}
\end{figure}

\paragraph{Output vector space.} 
We specify two different instantiations of the output vector space $\mathbb{R}_{\ge 0}^M$  in our framework, depending on the level of specificity of the output which the system designer aims to elicit. 
\begin{enumerate}
    \item Suppose that the system designer aims to elicit a list of papers that overall reflects a specific balance of different topics (e.g., machine learning, economics, etc.). To capture this, let each dimension $i \in [M]$ of the output capture a different topic, so the value $x_i$ captures the fraction of papers on the list that reflect the corresponding topic. We build on this instantiation in our empirical analysis in Section \ref{sec:empirical}.
\item Suppose that the system designer aims to elicit a specific list of references. To capture this, we take the output vector to be high-dimensional embedding coming from a text embedding model. We design a simple experiment using LLMs to illustrate this: we prompt GPT-4o-mini in different ways, perform aggregation also using GPT-4o-mini, and use a sentence-transformers model to produce 768-dimensional output vectors  (see \Cref{sec:visualizations} for details of the empirical setup).\footnote{Code is available at \url{https://github.com/nivasini/aggregation_compound_ai}.} 
Figure \ref{fig:embedding-grid} shows the embeddings for outputs to the five prompts, as well as the outputs produced by the intersection-style and addition-style aggregation rules. The five prompt outputs vary substantially, and the aggregated outputs differ markedly from both the originals and from each other, demonstrating how different prompts and aggregation rules can shape the embedding-space representation.
\end{enumerate}

\paragraph{Reward function specification limitations.} 
In our framework, recall that the reward function operates over coarsenings of the output dimensions (as captured by the feature weights matrix $\balpha$) rather than directly on the output dimensions.
This captures two types of prompt engineering limitations  that we describe below.  
\begin{enumerate}
    \item First, the system designer may struggle to precisely express what they truly want in the prompt, leading to underspecified prompts omitting some of the system designer's requirements \citep{Yang2025PromptUnderspecification}. For example, the system designer may prompt the model to focus on combinations of high-level topics, rather than fully specifying the fine-grained topic that they wish to see. We build on this instantiation in our empirical analysis in Section \ref{sec:empirical}. 
    \item Second, the model may not correctly interpret the system designer's prompt. For example, the model might map two dissimilar words in the prompt to the same word \citep{jones2025uncovering}. In the citation task, this could surface as the model interpreting ``papers with high attribute `X''' similarly for many different attributes ``X''. 
\end{enumerate}

%% file: arxiv_sections/merged_arxiv.tex
\section{Natural Mechanisms for Elicitability-Expansion}\label{sec:concrete}

In this section, we formalize natural mechanisms by which aggregation expands elicitability (Figure \ref{fig:mechanisms_examples}). First, we define the mechanisms, and we show how they expand elicitability via examples (Section \ref{subsec:examples}). Then, we show that these mechanisms are necessary for elicitability-expansion in general (Section \ref{sec:special_cases}). While these mechanisms are not sufficient for elicitability-expansion in general, strengthened versions of these mechanisms turn out to fully characterize elicitability-expansion, as we will show in \Cref{sec:general}.

\subsection{Formalizing the Mechanisms and Motivating Examples}\label{subsec:examples}

We formalize three natural mechanisms and demonstrate their ability to expand elicitability through examples. While our mechanisms are general, our examples focus on a 3-dimensional output space ($M = 3$) with 2-dimensional features ($N = 2$).  We focus on feature weights matrices $\balpha$ of the form 
$\balpha(q) := 
\begin{bmatrix}
    1 & 0 & q \\
    0 & 1 & q 
\end{bmatrix}$. 
Each of the output dimensions $x_1, x_2$ specialize to features $F_1, F_2$, respectively. That is, increasing the first output dimension $x_1$ only increases the first feature $F_1$, and increasing the second output dimension $x_2$ only increases the second feature $F_2$. Increasing the third output dimension $x_3$ increases both features, though the contribution is weighted by a factor of $q$. The parameter $q$ captures the extent to which it is possible to simultaneously maximize both features.

Our examples use the intersection and addition aggregation rules that we previously introduced.

Let us now formalize the mechanisms and demonstrate their ability to expand elicitability.

\paragraph{Mechanism 1: Feasibility Expansion.} 
Aggregation can help overcome the output limitations due to the conic constraints faced by each agent, producing outputs that are outside of the feasible set $\calB^{\text{feasible}}(\bconstraints)$. We formalize this through the following mechanism.  
\begin{definition}[Feasibility Expansion]\label{def:feasibility_expand}
    Given a constraint matrix $\bconstraints$, an aggregation operation $\aggop$ implements \textbf{feasibility expansion} if $\bx^{(A)}$ is infeasible i.e. $\bconstraints \bx^{(A)} \not \le 0$. 
\end{definition}
The following example, depicted in Figure \ref{fig:feasibility-expansion}, illustrates how aggregation operations which implement feasibility expansion can in turn expand elicitability. 
\begin{example}\label{ex:power_feasibility} Let the feature map be $\balpha = \balpha(2)$, so that increasing the third output dimension contributes significantly to both features.  We view the first two output dimensions as corresponding to two types of “undesirable” behavior, while dimension $3$ corresponds to “desirable” behavior. Let $\bconstraints$ be a single constraint of the form $x_3 \le x_1 + x_2$. The constraint captures how the model cannot produce the desirable dimension without also producing some of the undesirable dimension(s). 

Consider the aggregation operation $\bx^{(1)} = [1, 0, 1], \bx^{(2)} = [0, 1, 1]  \rightarrow \intagg(\bx^{(1)}, \bx^{(2)}) = [0, 0, 1]$, which implements feasibility expansion.  
The output $[0, 0, 1]$ is outside the feasibility set since it has only desirable dimensions and hence is not elicitable with any reward function specification and budget level. The outputs $[1, 0, 1]$ and $[0, 1, 1]$ turn out to be elicitable, so the aggregation operation is elicitability-expanding (Proposition \ref{prop:power_feasibility} in Appendix \ref{app:power_feasibility}). 

\end{example}

\paragraph{Mechanisms 2-3: Overcoming Reward Function Specification Limitations.} Even when an output is in the feasible set, the limitations of reward function specification still restrict which outputs are elicitable.  Aggregation can overcome the reward function specification limitations faced by the system designer, as  the next two mechanisms formalize.

\paragraph{Mechanism 2: Support Expansion.}
Reward function specification limitations can make it impossible to elicit outputs with a large support, as our results will later show.\footnote{\citet{Kleinberg2019Effort} showed this in single-agent environments without constraints.} The following mechanism formalizes how 
aggregation can combine outputs with smaller supports to produce an output with a richer support. 
\begin{definition}[Support expansion]\label{def:support_expand}
    An aggregation operation $\aggop$ implements \textbf{support expansion relative to $k$} if $\calS(\bx^{(A)}) \not \subseteq \calS(\bx^{(k)})$.
\end{definition}
Aggregation operations which implement support expansion can in turn expand elicitability, by producing outputs with richer supports than that are elicitable by a single agent. We illustrate this in the following example, which is depicted in Figure \ref{fig:support-expansion}.
\begin{example}\label{ex:power_support}
Let the feature map be $\balpha = \balpha({0.6})$. Suppose that there are no conic constraints $\bconstraints = \emptyset$, so elicitability challenges entirely stem from reward function specification limitations. We will think of the first two dimensions as two aspects we would like our output to simultaneously capture. 

Consider the aggregation operation $\bx^{(1)} = [1, 0, 0], \bx^{(2)} = [0, 1, 0]  \rightarrow \addagg(\bx^{(1)}, \bx^{(2)}; [1/2, 1/2]) = [1/2, 1/2, 0]$ which implements support expansion. The output $[1/2, 1/2,0]$ cannot be elicited directly, because reward functions that equally reward both features make dimension $3$ strictly preferred over the combination of dimensions $1$ and $2$. In contrast, 
outputs supported on just one of these two dimensions $j \in \left\{1,2\right\}$ can be elicited by only rewarding feature $F_j$.
As a result, the aggregation operation is elicitability-expanding (Prop~\ref{prop:power_support} in Appendix \ref{app:power_support}). 
\end{example}

\paragraph{Mechanism 3: Binding Set Contraction.} This next mechanism overcomes reward function specification limitations by taking advantage of the output limitations of the agent. Perhaps counterintuitively, the constraints on the output space can make it easier to elicit an output through a single reward. When a constraint is binding for an output vector, some reward-increasing directions become inaccessible to the agent, as these directions will lead to violation of the binding constraint. There are thus fewer directions of change that need to be disincentivized. Aggregation can combine outputs with binding constraints into an output with fewer binding constraints. 
\begin{definition}[Binding set contraction]\label{def:bind_contract}
 An aggregation operation $\aggop$ implements \textbf{binding set contraction relative to $k$} if $\calV(\bx^{(A)}) \not \supseteq \calV(\bx^{(k)})$.
\end{definition}

Aggregation operations which implement binding set contraction can expand elicitability. We illustrate this in the following example, which is depicted in Figure \ref{fig:binding-set-contraction}.
\begin{example}\label{ex:power_other} 
Let the feature map be $\balpha = \balpha(0.2)$. As in the first example, we will think of $x_3$ to be a ``desirable'' dimension and $x_1, x_2$ to be ``undesirable'' dimensions. Let $\mathbf{C}$ be a single constraint of the form $x_1 + x_2 \le x_3$. This constraint captures how the model cannot produce the bad dimension(s) without also producing some of the good dimension.

Consider the aggregation operation $\bx^{(1)} = [1, 0, 1], \bx^{(2)} = [0, 1, 1] \rightarrow \intagg(\bx^{(1)}, \bx^{(2)}) = [0,0,1]$, which implements binding-set contraction. The value of $q = 0.2$ is small, leading to dimension $3$ being inelicitable without the constraint. The constraint allows us to elicit a vector with some amount of $x_3$, but not a vector that has only $x_3$. The aggregation operation can allow us to get dimension $x_3$ without the additional dimensions $x_1$ or $x_2$, and is thus elicitability-expanding (Prop~\ref{prop:power_other} in Appendix \ref{app:power_other}).
\end{example}

\subsection{Connections between Elicitability-Expansion and Mechanisms} \label{sec:special_cases}
Moving beyond the examples in Section \ref{subsec:examples}, we more generally study how these mechanisms connect to elicitability-expansion. We will show that implementing one of these mechanisms is necessary for elicitability-expansion but is not sufficient. 

\paragraph{Necessity of these mechanisms.}
First, we show that if an aggregation operation expands elicitability, it must implement at least one of the three mechanisms. 
Specifically, Theorem \ref{thm:weaker_necessary} shows that either the operation must implement feasibility-expansion or it must implement at least one of support expansion or binding set contraction for every output $x^{(k)}$.  
\begin{theorem}
\label{thm:weaker_necessary}
Fix conic constraints $\bconstraints$, and any aggregation operation $\aggop$ where each $\bx^{(k)}$ is feasible (i.e., $\bconstraints \bx^{(k)} \le 0$, for every $k \in [K]$). 
If $\aggop$ is elicitability-expanding relative to constraints $\bconstraints$, then at least one of the following conditions holds:
\begin{itemize}
    \item $\aggop$ is feasibility-expanding relative to $\bconstraints$ (\Cref{def:feasibility_expand}).
    \item  For each $k \in [K]$, $\aggop$ is  either support-expanding relative to $k$ (\Cref{def:support_expand}) or binding set-contracting relative to $k$ (\Cref{def:bind_contract}).
\end{itemize}
\end{theorem}

Theorem \ref{thm:weaker_necessary} reveals a strong form of limitation for aggregation operations that do not implement at least one of the mechanisms (Definition \ref{def:feasibility_expand}, \ref{def:support_expand}, and  \ref{def:bind_contract}). Specifically, the result illustrates that if an operation does not implement the mechanisms according to the conditions in Theorem \ref{thm:weaker_necessary}, then aggregation is not elicitability-expanding regardless of the feature weights matrix $\balpha$. This result illustrates conditions under which aggregation adds no power to compound AI systems regardless of the level of prompt engineering limitations. 

\paragraph{Proof ideas.}
The full proof of \Cref{thm:weaker_necessary} is presented in \Cref{proof:weaker_necessary}. The proof builds on the technical tools we develop in Section \ref{sec:general} (i.e., Theorem \ref{thm:necessary}). The main idea is that when the condition of \Cref{thm:weaker_necessary} is violated by not implementing any of the natural mechanisms, the aggregation operation does not expand the set of viable directions of change. The notion of viable directions is formalized in \Cref{def:feasible_budget_reducing_directions}. At a high level, these are directions in which we can move a small amount from an output vector while retaining non-negative values on all coordinates and continuing to satisfy the conic feasibility constraints and budget constraints. A vector with a larger set of viable directions is harder to elicit since this vector must be preferable compared to a larger set of feasible vectors. An aggregation operation that expands the set of viable directions produces vectors that are hard to elicit through direct prompting. 

\paragraph{These mechanisms are not sufficient.} While implementing one of these three natural mechanisms (implemented in the way provided by conditions of \Cref{thm:weaker_necessary}) is necessary for aggregation to have power, it is not necessarily sufficient to guarantee power.  
Even if an aggregation  operation implements support expansion for every $k \in [K]$, the aggregation operation still may not be elicitability-expanding as shown in \Cref{prop:insufficient_support} in \Cref{sec:insufficiency}. This is a strong limitation since it means that the aggregation operation offers no power regardless of the feature weights matrix. \Cref{prop:insufficient_binding} in \Cref{sec:insufficiency} shows a similar insufficiency for binding set contraction. On the other hand, feasibility expansion on its own guarantees elicitability-expansion as shown in \Cref{prop:feasibilityexpansion} in \Cref{sec:insufficiency}.

\paragraph{Summary.} While the conditions based on the three natural mechanisms stated in \Cref{thm:weaker_necessary} are necessary for aggregation to have power, they are not sufficient for aggregation to have power. In the next section (Section \ref{sec:general}), we will show that implementation of a strengthened version of these mechanisms serves as a necessary and sufficient conditions for power of aggregation, more precisely capturing the power and limitations of aggregation.

\section{Characterizing Elicitability-Expansion}\label{sec:general}

In this section, we will provide a characterizing condition that is both necessary and sufficient for an aggregation operation to have power (i.e., expand elicitability under some feature map). This characterizing condition is based on \emph{strengthened} versions of the natural mechanisms we introduced in \Cref{subsec:examples}.

\subsection{Necessary and Sufficient conditions for elicitability-expansion}
Before stating our characterizing condition we will introduce an object that our characterizing condition depends on. This is the set of feasible and budget-reducing directions of the output vectors $\bx$ in the aggregation operation. We first define this set and how it relates to the elicitability of output vectors.

\begin{definition}[Feasible, budget-reducing directions]\label{def:feasible_budget_reducing_directions}
    The set of feasible, budget-reducing directions for an output vector $\bx$ depends on the sufficient statistics of the support and binding conic constraints indices $(\calS(\bx), \calV(\bx))$ of $\bx$. It is defined as 
    \[\mathcal{B}_{\calS(\bx),\calV(\bx)} = \underbrace{\{\bd \in \mathbb{R}^M: \bconstraints_{\calV(\bx)} \bd \le 0 \}}_{(1)} \cap \underbrace{\{ \bd \in \mathbb{R}^M : d_j \ge 0 \forall j \in (\calS(\bx))^c \}}_{(2)} \cap \underbrace{\{\mathds{1}^t \bd < 0\}}_{(3)}.\]
\end{definition}

The set $\mathcal{B}_{\calS(\bx),\calV(\bx)}$ captures the set of viable directions to move a small amount from $\bx$ while maintaining all constraints on output vectors (i.e., non-negative coordinates, feasibility according to conic constraints, and satisfying budget constraints). Note that the directions in $\mathcal{B}_{\calS(\bx),\calV(\bx)}$ are defined relative to only binding conic and non-negativity constraints rather than the set of all constraints. This is because we can move a small amount in a direction violating non-binding constraints while not violating those constraints. 

To determine if $\bx$ is elicitable with a specified reward function and budget, we search for reward-improving directions within $\mathcal{B}_{\calS(\bx),\calV(\bx)}$.  If such a reward-improving direction exists, this certifies that $\bx$ is not elicitable by providing a feasible output vector that the agent strictly prefers over $\bx$. On the other hand if no reward-improving direction exists within $\mathcal{B}_{\calS(\bx),\calV(\bx)}$, then $\bx$ is elicitable.

\paragraph{Power-characterizing condition.} The power-characterizing condition for an aggregation operation $\aggop$ is stated below based on directions in $\mathcal{B}_{\calS(\bx^{(A)}),\calV(\bx^{(A)})}$ and $\mathcal{B}_{\calS(\bx^{(k)}),\calV(\bx^{(k)})}$, $k \in [K]$.
\begin{definition}\label{def:characterizing_condition}[Power-characterizing condition]
Fix conic constraints $\bconstraints$.
We say that the \textbf{power-characterizing condition} is satisfied for $\bx^{(1)}, \ldots, \bx^{(K)} \rightarrow \bx^{(A)}$ if and only if one of the following two conditions hold:
\begin{enumerate}[leftmargin=*, nosep]
    \item  \emph{Feasibility expansion:} $\bx^{(1)}, \ldots, \bx^{(K)} \rightarrow \bx^{(A)}$ implements feasibility-expansion for $\bconstraints$ or

    \item \emph{Strengthened support expansion or strengthened binding set contraction:} There exists $\bd \in \mathcal{B}_{\calS(\bx^{(A)}), \calV(\bx^{(A)})}$ with $\bd \not \le 0$ such that for every $k \in [K]$, there exist $\bgamma^{(k)} \in \preals^{|\calV(\bx^{(k)})|}$ and $\blambda^{(k)} \in \preals^{|\calS(\bx^{(k)})^c|}$ such that
    the following inequality holds:
    \[\bw^{(k)} \bd \;+\; \bigl|\mathds{1}^\top \bd\bigr| \cdot \min_{j \in [M]} \min\!\bigl(\bw^{(k)}_j,\, 0\bigr) \;>\; 0\]
    where $\bw^{(k)} := \bigl(\bgamma^{(k)}\bigr)^{\!\top} C_{\calV(\bx^{(k)})} \;-\; \bigl(\blambda^{(k)}\bigr)^{\!\top} I_{\calS(\bx^{(k)})^c}$.
\end{enumerate}
\end{definition}

This condition requires the aggregation operation is either feasibility-expanding or for every $k \in [K]$, it either implements a strengthened version of support expansion or a strengthened version of binding set contraction. We discuss this in more detail in \Cref{sec:interpret_cond}, but briefly discuss the intuition here. The second condition implies that either $(\bgamma^{(k)})^{\!\top} C_{\calV(\bx^{(k)})} \bd > 0$ or $(\blambda^{(k)})^{\!\top} I_{\calS(\bx^{(k)})^c} < 0$. The first case turns out to imply binding-set contraction, while the second case turns out to imply support expansion. Moreover, the condition strengthens the mechanisms in two ways. First, it requires the mechanisms to jointly be implemented across all of the agents, rather than implemented for each agent individually. Second, it requires violation of (weighted) conic and non-negativity constraints by a minimum margin, whereas the original mechanisms have no margin requirements.

\paragraph{Characterization results.}
 Our main theorems state that the power-characterizing condition provides a full characterization of when an aggregation operation can expand elicitability under some feature map.  
For the first part of this characterization, we show that the power-characterization is sufficient for elicitability-expansion for some feature weights matrix. 

\begin{theorem}[Sufficient]
\label{thm:sufficient}
Fix constraints $\bconstraints$ and aggregation operation $\bx^{(1)}, \ldots, \bx^{(K)} \rightarrow \bx^{(A)}$ where each $\bx^{(k)}$ is feasible (i.e., $\bconstraints \bx^{(k)} \le 0$, for every $k \in [K]$). 
If the power-characterizing condition is satisfied for $\bx^{(1)}, \ldots, \bx^{(K)} \rightarrow \bx^{(A)}$, then  $\bx^{(1)}, \ldots, \bx^{(K)} \rightarrow \bx^{(A)}$ is elicitability-expanding relative to $\bconstraints$. 
\end{theorem}

To complete the characterization, our next main theorem states that power-characterizing condition is also necessary for an aggregation to have power. That is, if this condition does not hold, the aggregation operation faces a strong limitation of not offering elicitability-expansion for any feature weights matrix.

\begin{theorem}[Necessary]
\label{thm:necessary}
Fix constraints $\bconstraints$ and aggregation operation $\bx^{(1)}, \ldots, \bx^{(K)} \rightarrow \bx^{(A)}$ where each $\bx^{(k)}$ is feasible (i.e., $\bconstraints \bx^{(k)} \le 0$, for every $k \in [K]$). If the power-characterizing condition is not satisfied, then $\bx^{(1)}, \ldots, \bx^{(K)} \rightarrow \bx^{(A)}$ is not elicitability-expanding relative to $\bconstraints$.
\end{theorem}

\paragraph{Key proof ideas.} Our main technical tool (\Cref{thm:fixed_feature_map} in \Cref{subsec:power}) provides a geometric characterization of elicitability-expansion under a fixed feature map, extending~\cite{Kleinberg2019Effort} to handle conic feasibility constraints and aggregated outputs. This lemma reduces the problem to checking whether two sets have empty intersection: the set of feasible, budget-reducing directions for vectors in the aggregation, and the cone of \textit{feature-improving directions} $\{\bd \in \mathbb{R}^M \mid \balpha \bd \ge \mathbf{0}\}$.

To move from elicitability-expansion under a \emph{fixed} feature map to elicitability-expansion under \emph{any} feature map, we consider how to construct feature weights $\balpha$ that make a set $\bx^{(1)}, \ldots, \bx^{(K)}$ elicitable while rendering $\bx^{(A)}$ inelicitable. The key subtlety is that \textit{feature-improving directions} form a cone. So non-negative scaling and translations of improving directions remain improving. As a result, we cannot simply find a single separating direction for $\bx^{(A)}$---we need a robust geometric condition guaranteeing separation that survives all such transformations. We capture this robustness through a minimax argument, showing it is equivalent to the existence of a worst-case distribution over translations that certifies inelicitability. We elaborate on these ideas in \Cref{sec:proof_ideas}. 

\subsection{Connection of the power-characterizing condition to mechanisms}\label{sec:interpret_cond}
 The power-characterizing condition requires one of two cases to hold. The first is implementation of feasibility expansion. We can interpret the second case as unifying a strengthening of support expansion and a strengthening of binding set contraction into a single inequality. To demonstrate the connection between the power-characterizing condition and the mechanisms, we will first show how the power-characterizing condition implies implementation of one of the mechanisms. We will then show how the power-characterizing condition strengthens the mechanisms.

\paragraph{Power-characterization condition implies the mechanisms.} The following result shows that the power-characterizing condition implies that at least one of feasibility-expansion, support expansion, or binding set contraction is implemented. This result immediately implies Theorem \ref{thm:weaker_necessary} (i.e., that implementing at least one of these mechanisms is necessary for elicitability-expansion).
\begin{proposition}
\label{lem:stronger_implies_weaker}
Fix conic constraints $\bconstraints$, and any aggregation operation $\aggop$ where each $\bx^{(k)}$ is feasible i.e., $\bconstraints \bx^{(k)} \le 0$, for every $k \in [K]$.
If $\aggop$ satisfies the power-characterizing condition (\Cref{def:characterizing_condition}), then one of the following conditions holds:
\begin{itemize}[leftmargin=*]
    \item $\aggop$ is feasibility-expanding relative to $\bconstraints$ (\Cref{def:feasibility_expand}).
    \item  For each $k \in [K]$, $\aggop$ is  either support-expanding relative to $k$ (\Cref{def:support_expand}) or binding set-contracting relative to $k$ (\Cref{def:bind_contract}).
\end{itemize}
\end{proposition}

\begin{proof}
The first case for the power-characterizing condition to hold is feasibility expansion. Suppose feasibility expansion does not hold, i.e., all vectors in the aggregation are feasible. Then the second case of \Cref{def:characterizing_condition} holds: there exists $\bd \in \mathcal{B}_{\calS(\bx^{(A)}), \calV(\bx^{(A)})}$ with $\bd \not \le 0$ such that for every $k \in [K]$, there exist $\bgamma^{(k)} \in \preals^{|\calV(\bx^{(k)})|}$ and $\blambda^{(k)} \in \preals^{|\calS(\bx^{(k)})^c|}$ with
\[\bw^{(k)} \bd \;+\; |\mathds{1}^\top \bd| \cdot \min_{j \in [M]} \min\!\bigl(\bw^{(k)}_j,\, 0\bigr) \;>\; 0,\]
where $\bw^{(k)} := \bigl(\bgamma^{(k)}\bigr)^{\!\top} \bconstraints_{\calV(\bx^{(k)})} - \bigl(\blambda^{(k)}\bigr)^{\!\top} I_{\calS(\bx^{(k)})^c}$. Since $\min_{j \in [M]} \min(\bw^{(k)}_j, 0) \le 0$ and $|\mathds{1}^\top \bd| \ge 0$, the second term on the left-hand side is non-positive, so the inequality implies $\bw^{(k)} \bd > 0$, i.e.,
\[\bigl(\bgamma^{(k)}\bigr)^{\!\top} \bconstraints_{\calV(\bx^{(k)})} \bd \;>\; \bigl(\blambda^{(k)}\bigr)^{\!\top} I_{\calS(\bx^{(k)})^c} \bd.\]
We will show that for each $k \in [K]$, this strict inequality implies either binding set contraction relative to $k$ or support expansion relative to $k$.

\emph{Case 1: $\bigl(\bgamma^{(k)}\bigr)^{\!\top} \bconstraints_{\calV(\bx^{(k)})} \bd > 0$ implies binding set contraction relative to $k$.} Since $\bgamma^{(k)} \ge 0$ and $\bigl(\bgamma^{(k)}\bigr)^{\!\top} \bconstraints_{\calV(\bx^{(k)})} \bd = \sum_{\ell \in \calV(\bx^{(k)})} \bgamma^{(k)}_\ell \bconstraints_\ell \bd > 0$, there exists $\ell \in \calV(\bx^{(k)})$ with $\bconstraints_\ell \bd > 0$. However, $\bd \in \mathcal{B}_{\calS(\bx^{(A)}), \calV(\bx^{(A)})}$ requires $\bconstraints_\ell \bd \le 0$ for every $\ell \in \calV(\bx^{(A)})$. So $\ell \in \calV(\bx^{(k)}) \setminus \calV(\bx^{(A)})$, which is evidence that $\calV(\bx^{(A)}) \not \supseteq \calV(\bx^{(k)})$, i.e., binding set contraction relative to $k$.

\emph{Case 2: $\bigl(\bgamma^{(k)}\bigr)^{\!\top} \bconstraints_{\calV(\bx^{(k)})} \bd \le 0$ implies support expansion relative to $k$.} In this case, the strict inequality forces $\bigl(\blambda^{(k)}\bigr)^{\!\top} I_{\calS(\bx^{(k)})^c} \bd < 0$, i.e., $\sum_{j \in \calS(\bx^{(k)})^c} \blambda^{(k)}_j \bd_j < 0$. So there exists $j \in \calS(\bx^{(k)})^c$ with $\bd_j < 0$. Since $\bd \in \mathcal{B}_{\calS(\bx^{(A)}), \calV(\bx^{(A)})}$ requires $\bd_j \ge 0$ for every $j \in \calS(\bx^{(A)})^c$, we have $j \notin \calS(\bx^{(A)})^c$, i.e., $j \in \calS(\bx^{(A)}) \setminus \calS(\bx^{(k)})$. This is evidence that $\calS(\bx^{(A)}) \not \subseteq \calS(\bx^{(k)})$, i.e., support expansion relative to $k$.
\end{proof}

\paragraph{How the power-characterizing condition strengthens the mechanisms.} While the power-characterizing condition implies the implementation of one of the mechanisms, it is a strictly stronger condition. This is evidenced by the fact that the power-characterizing condition is necessary and sufficient for elicitability-expansion while implementing one of the mechanisms is not sufficient for elicitability-expansion as shown by propositions in \Cref{sec:insufficiency}.

The power-characterizing condition, specifically the second case of \Cref{def:characterizing_condition}, strengthens the mechanisms in two ways.
\begin{itemize}
    \item First, we require a mechanism to jointly be implemented across all of the agents, rather than implemented for each agent individually. That is, we place a \emph{joint} requirement across all $k \in [K]$, requiring that the same $\bd \in \calB_{\calS(\bx^{(A)}), \calV(\bx^{(A)})}$ witnesses the violation of binding constraints for every $k \in [K]$.

 \item Second, we strengthen the mechanism for each agent individually, requiring there to be a sufficient gap between the original outputs $\bx^{(k)}$ and the aggregated output $\bx^{(A)}$.
Implementing support expansion or binding set contraction only requires that $\bd \in \calB_{\calS(\bx^{(A)}), \calV(\bx^{(A)})}$ violate some binding non-negativity or conic constraint of each $\bx^{(k)}$, i.e., that $\bw^{(k)} \bd > 0$ for some non-negative weighting $\bw^{(k)} = (\bgamma^{(k)})^{\!\top} \bconstraints_{\calV(\bx^{(k)})} - (\blambda^{(k)})^{\!\top} I_{\calS(\bx^{(k)})^c}$ of these binding constraints. Such a violation can hold with $\bx^{(A)}$ arbitrarily close to any of the $\bx^{(k)}$ vectors. In contrast, the power-characterizing condition requires this violation to hold by a minimum margin: rearranging the inequality in \Cref{def:characterizing_condition}, it asks that
\[\bw^{(k)} \bd \;>\; \bigl|\mathds{1}^\top \bd\bigr| \cdot \Bigl|\min_{j \in [M]} \min(\bw^{(k)}_j, 0)\Bigr|,\]
which is a strictly positive margin whenever $\bw^{(k)}$ has a negative coordinate. For example, taking $\bgamma^{(k)} = \mathbf{0}$ and $\blambda^{(k)}$ supported on a single index $j \in \calS(\bx^{(k)})^c$ gives $\bw^{(k)} \bd = -\bd_j$ and recovers the requirement $-\bd_j > |\mathds{1}^\top \bd|$, so $\bd$ must violate that binding non-negativity constraint by a margin of $|\mathds{1}^\top \bd|$. This margin forces $\bx^{(A)}$ to be sufficiently distinct from each $\bx^{(k)}$.
\end{itemize}

The power-characterizing condition is in general a strictly stronger condition than implementing one of the mechanisms, as reflected by counterexamples in~\Cref{sec:insufficiency}. However, in some special cases, the power-characterizing condition corresponds exactly to implementing one of the mechanisms, instead of implementing a strengthening. One special case is when no vector in the aggregation operation has any binding conic constraints. This holds when there are no conic constraints. In this special case, even the regular, non-strengthened form of binding set contraction cannot kick in. We can show that in this special case, the power-characterizing condition reduces to either feasibility-expansion or the usual, non-strengthened support expansion as long as a particular edge case does not occur (the proof is deferred to  \Cref{proof:alllargesubsets}).

\begin{corollary}
\label{prop:alllargesubsets}
Fix conic constraints $\bconstraints$, and any $\aggop$. Suppose that $\calV(x^{(A)}) = \calV(x^{(1)}) = \ldots = \calV(x^{(K)}) = \emptyset$.
Suppose also that either $\bx^{(A)}$ is not full support (i.e., $\calS(\bx^{(A)}) \neq [M]$) or there exists $j \in [M]$ such that no $\bx^{(k)}$ has support $[M] \setminus \{j\}$.
Then $\aggop$ is elicitability-expanding relative to $\bconstraints$ if and only if at least one of the following two conditions holds:
\begin{itemize}
    \item $\aggop$ is feasibility-expanding.
    \item $\aggop$ is support-expanding relative to every $k \in [K]$.
\end{itemize}
\end{corollary}

Note that it appears harder to show an analogous result for binding-set contraction. This is due to the joint geometry of the constraints that appears in the power-characterizing condition.

\subsection{Main lemma: Elicitability-expansion for a given feature map}\label{subsec:power}

In this section, we will build tools to prove our main characterization result. The main technical lemma is the characterization of elicitability-expansion under a given feature map $\balpha$. This lemma extends the characterization of previous work~\cite{Kleinberg2019Effort} for elicitability of an output vector under a feature map $\alpha$ in a setting without conic constraints on the agent's output vector. We extend this characterization to allow for output limitations (i.e., nontrivial constraints $\bconstraints$) and aggregation of multiple outputs.

This characterization depends on the set $\calB_{\calS(\bx), \calV(\bx)}$ of feasible, budget-reducing directions (\Cref{def:feasible_budget_reducing_directions}) of the vectors $\bx$ in the aggregation operation. The condition checks for intersection between these sets and the set of feature-improving directions $\left\{\bd \in \mathbb{R}^M \mid \balpha \bd \ge \mathbf{0} \right\}$ consisting of directions that weakly increase \emph{all} feature values. 

\begin{lemma}\label{thm:fixed_feature_map}
 Fix conic constraints $\bconstraints$, feature weights matrix $\balpha$, and aggregation operation $\bx^{(1)}, \ldots, \bx^{(K)} \rightarrow \bx^{(A)}$ where each $\bx^{(k)}$ is feasible (i.e., $\bconstraints \bx^{(k)} \le 0$, for every $k \in [K]$). The aggregation operation $\bx^{(1)}, \ldots, \bx^{(K)} \rightarrow \bx^{(A)}$ 
 is elicitability-expanding relative to $\bconstraints$ and $\mathbf{F}$ if and only if both of the following conditions hold:
 \begin{itemize}
     \item For every $k \in [K]$, $\mathcal{B}_{\calS(x^{(k)}), \calV(\bx^{(k)})} \cap \left\{\bd \in \mathbb{R}^M \mid \balpha \bd \ge \mathbf{0} \right\} = \emptyset$.
     \item $\mathcal{B}_{\calS(x^{(A)}), \calV(x^{(A)})} \cap \left\{\bd \in \mathbb{R}^M \mid \balpha \bd \ge \mathbf{0} \right\} \neq \emptyset$ or $\bconstraints \bx^{(A)} \not \le 0$.
 \end{itemize} 
\end{lemma}

Lemma \ref{thm:fixed_feature_map} characterizes the power of aggregation for a given feature weights matrix $\balpha$. As a result, the characterizing condition depends on both the reward function specification limitation (which reflects prompt engineering limitations) via $\balpha$ and the output limitation (which reflect model capability limitations) via the conic constraints $\mathbf{C}$. This result illustrates how both forms of limitations influence the power of aggregation.

\paragraph{Proof ideas.} 
The full proof of \Cref{thm:fixed_feature_map} appears in \Cref{proof:fixed_feature_map}. 
The key idea is that elicitability of an output vector $\bx$ is characterized by whether the set of feasible perturbation directions,
$\calB_{\calS(\bx), \calV(\bx)}$, intersects the set
$\{\bd \in \mathbb{R}^M : \balpha \bd \ge \mathbf{0}\}$.
Lemmas~\ref{lem:single_output_elicitation_necessary} and~\ref{lem:single_output_elicitation_sufficient}
establish the necessity and sufficiency of this condition.
This characterization yields conditions under which each individual output $\bx^{(k)}$ for $k \in [K]$ is elicitable,
while the aggregate output $\bx^{(A)}$ is not, which is precisely the notion of elicitability expansion.

To prove the lemma, one direction is straightforward: if the intersection is nonempty, moving in any direction $\bd$ in the intersection maintains feasibility and strictly increases every monotone reward over the features, resulting in a feasible vector strictly preferred over $\bx$, providing a certificate that $\bx$ cannot be elicitable.

The other direction of the characterization is more involved and shows that whenever the intersection is empty, there is a reward function that elicits $\bx$. We can write the empty intersection as infeasibility of a system of linear inequalities. We certify emptiness via duality, and the dual certificate directly yields a reward function that elicits the desired input vectors. Specifically, using Motzkin's transposition, we can equivalently write this as a certificate in terms of dual variables. The dual variables from this certificate along with a linear reward function constructed based on these variables demonstrate optimality of $\bx$ for the reward-maximization program by satisfying the KKT conditions of this program. 

\subsection{Proof ideas for main theorems}\label{sec:proof_ideas}
In this section, we will provide proof sketches for our main theorems:  Theorem \ref{thm:necessary} and Theorem \ref{thm:sufficient} that provide the characterization of elicitability-expansion. We will build on the technical lemma (\Cref{thm:fixed_feature_map}), but move beyond specific feature weights matrices $\balpha$ to reason about elicitability-expansion under \textit{some} feature weights matrix $\balpha$. 

Based on \Cref{thm:fixed_feature_map}, one might expect that elicitability-expansion occurs exactly when $\aggop$ either implements feasibility-expansion or if there exists a direction $\bd \in \calB_{\calS(\bx^{(A)}), \calV(\bx^{(A)})}$ that does not exist in $\calB_{\calS(\bx^{(k)}), \calV(\bx^{(k)})}$ for any $k \in [K]$. We might hope that the existence of such a $\bd$ allows us to construct feature weights matrix $\balpha$ where $\bd$ is the only feature-improving direction, which would make $\bx^{(A)}$ not elicitable while each $\bx^{(k)}$ is elicitable.

However, it is not possible to construct $\balpha$ with $\bd$ being the only feature-improving direction. If $\bd$ is a feature-improving direction (i.e., $\balpha \bd \ge 0$), then every direction $\bu + \lambda \bd$ for $\bu, \lambda > 0$ is also a feature-improving direction. Due to this property, a stronger empty intersection condition than each $\calB_{\calS(\bx^{(k)}), \calV(\bx^{(k)})}$ and $\calB_{\calS(\bx^{(A)}), \calV(\bx^{(A)})}$ having empty intersection  turns out to be necessary for elicitability expansion. This stronger condition is stated below and we will show that this is actually equivalent to the power-characterizing condition.

\begin{definition}\label{def:alt_condition}
 Fix constraints $\bconstraints$ and aggregation operation $\bx^{(1)}, \ldots, \bx^{(K)} \rightarrow \bx^{(A)}$. We say that the \textbf{alternate power-characterizing condition} is satisfied for $\bx^{(1)}, \ldots, \bx^{(K)} \rightarrow \bx^{(A)}$ if (1) $\aggop$ implements feasibility-expansion, or (2) there exists $\bd^{(A)} \in \mathcal{B}_{\calS(x^{(A)}), \calV(x^{(A)})}$ with $\bd^{(A)} \not \le 0$ such that: 
\[\left\{\bu + \lambda \bd^{(A)} \mid \bu \in \mathbb{R}_{\ge 0}^M, \lambda \ge 0 \right\} \bigcap \left(\bigcup_{k \in [K]} \mathcal{B}_{\calS(x^{(k)}), \calV(x^{(k)})}\right) = \varnothing.\]  
\end{definition}

\paragraph{Equivalence of both forms of power-characterizing condition.} The alternate power-characterizing condition turns out to be exactly equivalent to the power-characterizing condition in \Cref{def:characterizing_condition}. as stated in~\Cref{prop:equivalence}.

We provide a brief proof sketch of ~\Cref{prop:equivalence} here. The full proof is in~\Cref{proof:equivalence}. Note that \Cref{def:characterizing_condition} is stated in terms of a $\bd \in \calB_{\calS{\bx^{(A)}}, \calV{\bx^{(A)}}}$ violating constraints by a margin, whereas \Cref{def:alt_condition} is in terms of all non-negative scaled translations $\bu + \lambda \bd$ violating constraints. Let us now describe how violation by a margin relates to violation after non-negative translation. For simplicity, let us assume there is a single constraint. That is $\bconstraints$ has a single row.

To prevent against any budget-reducing $\bu + \lambda \bd$ satisfying the constraint, we need to prevent against the optimal choice of $\bu$ picked for constraint satisfaction. This is the $\bu$ that minimizes $\bconstraints \bu$ subject to the budget-decreasing constraint, which is that $\|\bu\|_1 \le \lambda |\mathds{1}^t d|$. The minimizing $\bu$ is the zero vector if $\bconstraints$ has all positive coordinates. Otherwise, it places all of its weight on the most negative coordinate of $\bconstraints$. Ensuring that the minimum value of $\bconstraints \bu$, which is $\lambda |\mathds{1}^t \bd|\min_{j \in [n]}(0, \bconstraints_j)$, is lower than $\lambda \bconstraints \bd$ is exactly the violation by a margin condition from the power-characterization condition.  

Extending to multiple constraints requires ensuring that no single translation $\bu$ can satisfy all constraints simultaneously. Rather than maximizing satisfaction of a single constraint, we must now consider the max-min problem: choosing $\bu$ to maximize the minimum level of satisfaction across all constraints. By the minimax theorem, this max-min value equals the min-max value obtained by first choosing a weighting over constraints and then choosing $\bu$ to maximize satisfaction of the resulting weighted constraint. This duality reduces the multi-constraint case to the single-constraint analysis above, applied to the worst-case weighted combination.

\paragraph{Proof sketch of Theorem \ref{thm:necessary}.} The full proof is provided in~\Cref{proof:necessary}. It uses the ideas from the proof of \Cref{thm:weaker_necessary} which shows the need for a feasible, budget-reducing direction of $\bx^{(A)}$ that is not a feasible, budget-reducing direction for any other $\bx^{(k)}$ for $k \in [K]$. Only then can there be a feature-improving direction that is viable for $\bx^{(A)}$ but not any of the $\bx^{(k)}$'s to result in elicitability of each $\bx^{(k)}$ but not of $\bx^{(A)}$. However, since all non-negative scaling and translations of feature-improving directions continue to be feature-improving, we further require positive scaling and translation of some direction $\bd^{(A)} \in \calB_{\calS(\bx^{(A)}), \calV(\bx^{(A)})}$ to not be present in any $\calB_{\calS(\bx^{(k)}), \calV(\bx^{(k)})}$, $k \in [K]$. This is precisely the condition of the alternate power-characterizing condition.

\paragraph{Proof sketch of Theorem \ref{thm:sufficient}.} The full proof is provided in~\Cref{proof:sufficient}. We have shown why a stronger condition is still necessary for elicitability-expansion. Now let us show that this stronger condition is also sufficient. We show this by explicitly constructing a feature map $\balpha$ for which the set $\{\bu + \lambda \bd^{(A)} \mid \bu \in \mathbb{R}_{\ge 0}^M, \lambda \ge 0 \}$ contains all feature-improving directions i.e., the set $\{\bd: \balpha \bd \ge 0\}$. In the construction, feature improving directions can have any value on the positive coordinates of $\bd^{(A)}$. But their negative coordinates cannot be too large compared to the positive coordinates.

\subsection{Conceptual insights for system designers}\label{subsec:conceptualinsights}

Our characterization results offer conceptual insights into when system designers benefit from specific aggregation operations in compound AI systems.

\paragraph{General takeaways.} Our results characterize how the interplay between prompt engineering limitations and model capability limitations affects which types of aggregation operations are useful. Specifically, aggregation not only overcomes model capability limitations (feasibility expansion), but also overcomes prompt engineering limitations through combining multiple output characteristics (support expansion) and through taking advantage of output-level limitations (binding set-contraction). Notably, even as model capabilities continue to improve, the latter two mechanisms mean that aggregation can still be useful to system designers.  On the flip side, our results illustrate how aggregation operations that do not take advantage of these mechanisms offer no power, regardless of whether the system designer employs sophisticated or unsophisticated prompt engineering practices. 

 \paragraph{Illustrative example.} As an illustrative example, consider intersection aggregation $\intagg$ and addition aggregation $\addagg$.

 \begin{itemize}[leftmargin=*]
     \item \textbf{Intersection aggregation:} Intersection aggregation combines outputs based on commonality among different output vectors, which is conceptually similar to debate protocols \citep{Du2024MultiagentDebate} that aim to create agreement or inference scaling methods that aim to filter out incorrect information \citep{Zhang2025CoTSynthesizer}. Intersection aggregation can implement feasibility expansion and binding-set contraction, but not
    support expansion (Table  \ref{tab:aggregation-expansion} in \Cref{sec:aggregation_rules}).
    However, feasibility expansion and binding-set contraction fundamentally rely on model capability limitations. If models are very capable (i.e., if they face no conic constraints), Theorem \ref{thm:necessary}  demonstrates that intersection aggregation never adds power, regardless of whether the system designer employs sophisticated or unsophisticated prompt engineering practices.

   \item \textbf{Addition aggregation:}
    Addition aggregation interpolates among different output directions, which conceptually captures how system designers synthesize multiple outputs to
delegate specialized subtasks to each agent and synthesize the outputs of these subtasks \citep{BAIR2024CompoundAISystems, Anthropic2025MultiAgentSystem}. Addition aggregation can implement support expansion (Table \ref{tab:aggregation-expansion} in \Cref{sec:aggregation_rules}) and can actually implement the strengthened form of support expansion as well (Example \ref{ex:power_support}). Theorem \ref{thm:sufficient} shows that addition aggregation adds power even when models face no capability limitations (i.e, no conic constraints), at least for some level of prompt engineering limitations.
\end{itemize}

%% file: arxiv_sections/empirical.tex
\section{Empirical Illustration using LLMs}\label{sec:empirical}

\begin{figure*}[t]
\centering
\begin{subfigure}[t]{0.32\textwidth}
    \centering
    \includegraphics[width=\textwidth]{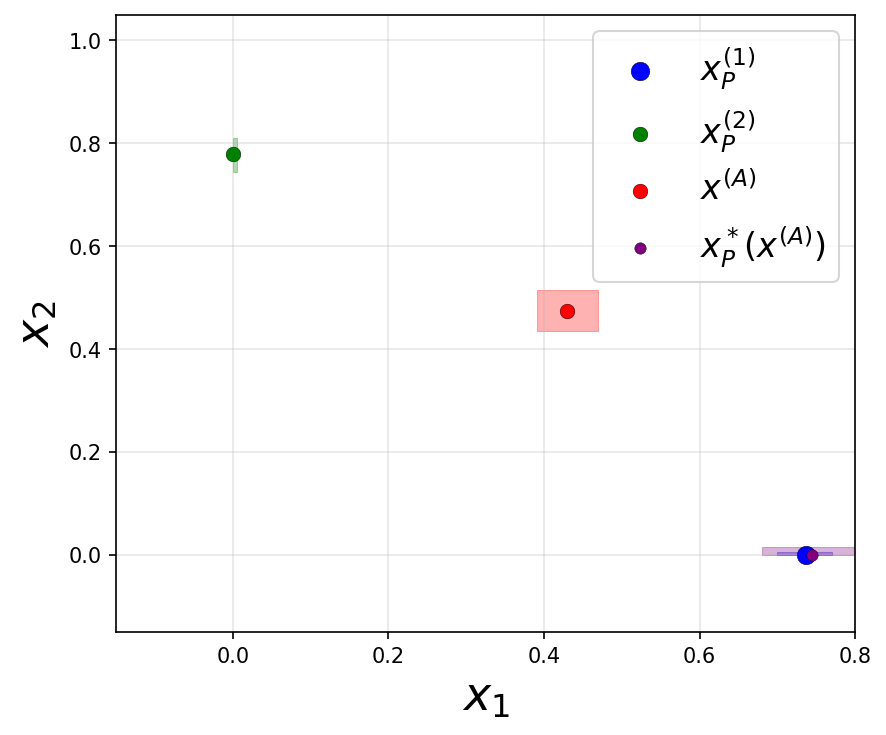}
    \caption{Support expansion}
    \label{fig:empirical-support-expansion}
\end{subfigure}
\hfill
\begin{subfigure}[t]{0.32\textwidth}
    \centering
    \includegraphics[width=\textwidth]{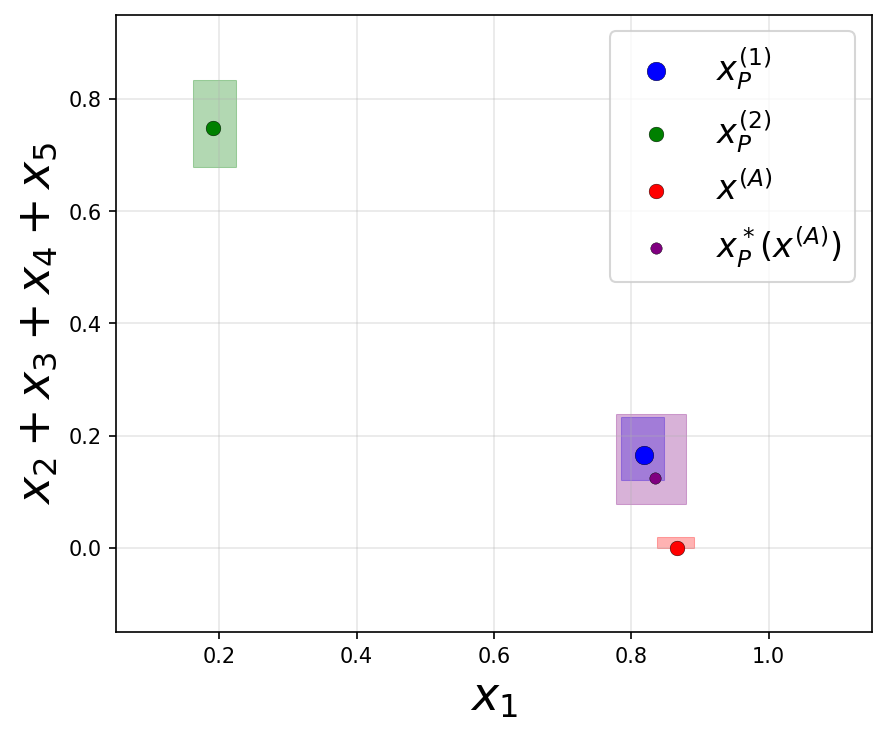}
    \caption{Binding set contraction}
    \label{fig:empirical-binding-contraction}
\end{subfigure}
\hfill
\begin{subfigure}[t]{0.32\textwidth}
    \centering
    \includegraphics[width=\textwidth]{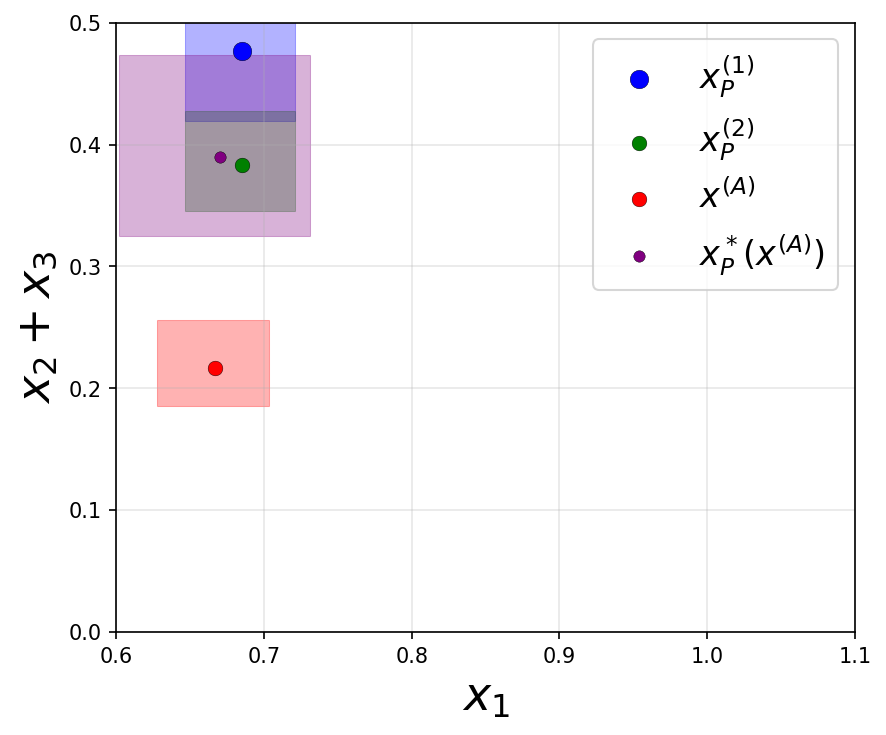}
    \caption{Feasibility expansion}
    \label{fig:empirical-feasibility-expansion}
\end{subfigure}
\caption{Empirical illustration of the three mechanisms in a toy reference-generation task with GPT-4o-mini at temperature 0.7 (Section \ref{sec:empirical}). Output dimensions capture paper topics, prompt dimensions capture (potentially higher-level) topics, and aggregation takes a union or an intersection of the lists. The plots show output vectors generated by the individual prompts ($\bx^{(1)}$ and $\bx^{(2)}$) and by the aggregation operation ($\bx^{(1)}, \bx^{(2)} \rightarrow \bx^{(A)}$), along with the output vector closest to $\bx^{(A)}$ that is elicitable by a single model with prompt topics ($\bx^*_P(\bx^{(A)})$). The shaded regions show confidence sets. Each plot shows an aggregation operation that implements one of the mechanisms---support expansion (left), binding-set contraction (middle), and feasibility expansion (right). The plots show that these aggregation operations are all elicitability-expanding. The numerical values are shown in Table \ref{tab:empirical_results}.
This plot is an empirical analogue of Figure \ref{fig:mechanisms_examples}.}
\label{fig:empirical-mechanisms}
\end{figure*}

In this section, we provide empirical support using LLMs for each of these three mechanisms and show they give power to the system designer.\footnote{Code is available at \url{https://github.com/nivasini/aggregation_compound_ai}.} Specifically, we aim to construct an aggregation operation for each operation that implements each mechanism and expands elicitability. Like in \Cref{subsec:casestudy}, we focus on a reference generation task. These experiments illustrate the robustness of our findings beyond the assumptions of our theoretical framework (e.g., reward-maximization and conic constraints).

\begin{table}[h]
\centering
\begin{minipage}{0.48\textwidth}
\centering
\resizebox{\textwidth}{!}{%
\begin{tabular}{lcccc}
\toprule
& $x_1$ & $x_2$ & $x_3$ \\
\midrule
$\bx^{(1)}$  & $[0.70, 0.77]$ & $[0.00, 0.00]$ & $[0.04, 0.08]$ \\
$\bx^{(2)}$  & $[0.00, 0.00]$ & $[0.74, 0.81]$ & $[0.03, 0.06]$ \\
\midrule
$\bx^{(A)}$  & $[0.39, 0.47]$ & $[0.43, 0.51]$ & $[0.03, 0.06]$ \\
$\bx^*_P(\bx^{(A)})$  & $[0.68, 0.80]$ & $[0.00, 0.01]$ & $[0.05, 0.12]$ \\
\bottomrule
\end{tabular}%
}
\subcaption{Support Expansion}
\label{tab:support_expansion}
\end{minipage}%
\hfill
\begin{minipage}{0.48\textwidth}
\centering
\resizebox{\textwidth}{!}{%
\begin{tabular}{lcccc}
\toprule
& $x_1$ & $x_2$ & $x_3$ \\
\midrule
$\bx^{(1)}$  & $[0.65, 0.72]$ & $[0.36, 0.44]$ & $[0.06, 0.10]$ \\
$\bx^{(2)}$ & $[0.65, 0.72]$ & $[0.00, 0.00]$ & $[0.35, 0.42]$ \\
\midrule
$\bx^{(A)}$  & $[0.63, 0.70]$ & $[0.00, 0.00]$ & $[0.19, 0.25]$ \\
$\bx^*_P(\bx^{(A)})$  & $[0.60, 0.73]$ & $[0.00, 0.01]$ & $[0.33, 0.46]$ \\
\bottomrule
\end{tabular}%
}
\subcaption{Feasibility Expansion}
\label{tab:feasibility_expansion}
\end{minipage}

\vspace{1em}
\begin{minipage}{0.6\textwidth}
\centering
\resizebox{\textwidth}{!}{%
\begin{tabular}{lccccc}
\toprule
& $x_1$ & $x_2$ & $x_3$ & $x_4$ & $x_5$ \\
\midrule
$\bx^{(1)}$ & $[0.79, 0.85]$ & $[0.09, 0.14]$ & $[0.00, 0.02]$ & $[0.00, 0.02]$ & $[0.03, 0.06]$ \\
$\bx^{(2)}$ & $[0.16, 0.22]$ & $[0.00, 0.01]$ & $[0.61, 0.69]$ & $[0.04, 0.08]$ & $[0.02, 0.05]$ \\
\midrule
$\bx^{(A)}$  & $[0.84, 0.89]$ & $[0.00, 0.00]$ & $[0.00, 0.00]$ & $[0.00, 0.00]$ & $[0.00, 0.00]$ \\
$\bx^*_P(\bx^{(A)})$ & $[0.78, 0.88]$ & $[0.07, 0.16]$ & $[0.00, 0.01]$ & $[0.00, 0.03]$ & $[0.00, 0.04]$ \\
\bottomrule
\end{tabular}%
}
\subcaption{Binding-set contraction}
\label{tab:binding_contraction}
\end{minipage}
\hfill
\begin{minipage}{0.38\textwidth}
\centering
\resizebox{\textwidth}{!}{%
\begin{tabular}{lc}
\toprule
\shortstack[l]{Experiment} & \shortstack{$\|\bx^*_P(\bx^{(A)}) - \bx^{(A)}\|_1$} \\
\midrule
Support expansion & [0.63, 1.02] \\
Feasibility expansion & [0.07, 0.39] \\
Binding-set contraction & [0.07, 0.35] \\
\bottomrule
\end{tabular}%
}
\subcaption{$\ell_1$ distance between $\bx^{(A)} $ and $\bx^*_P(\bx^{(A)})$ }
\label{tab:empirical_angular_distance_summary}
\end{minipage}
\caption{Empirical illustration of the three mechanisms in a toy reference-generation task with GPT-4o-mini at temperature 0.7 (Section \ref{sec:empirical}). We show the exact values from Figure \ref{fig:empirical-mechanisms} for support expansion, feasibility expansion, and binding-set contraction. We also show the $\ell_1$ distance between $\bx^{(A)}$ and $\bx^*_P(\bx^{(A)})$. All of our quantities are reported as 95\% confidence intervals.
}
\label{tab:empirical_results}
\end{table}

\normalsize{}

\paragraph{Instantiation of our model.}
The task is to generate a list of titles of papers that reflects specific topics. We instantiate this task to capture key aspects of our model such as reward functions operating over coarser features, but we depart from some of the specifics such as linearity of the features and the monotonicity assumption on the reward functions.  Each problem instance is defined by output dimension topics $\mathcal{T}^O$, 
prompt topics $\mathcal{T}^P$, and an aggregation operation $\mathcal{A}^{\cup}$ or $\mathcal{A}^{\cap}$ as we describe below. 
\begin{itemize}
    \item \textit{Outputs:} Given a reference list $L$, we let the output vector $\bx \in \mathbb{R}_{\ge 0}^M$ capture the extent to which the list is perceived to reflect different topics $\mathcal{T}^O := \left\{\mathcal{T}^{O}_i\right\}_{i \in [M]}$. That is, each output dimension $i \in [M]$ corresponds to a topic $T^{O}_i$, and $x_i$ captures the fraction of papers in $L$ that are perceived to belong to the paper topic $x_i$. 
    \item \textit{Reward function specification:} Prompting\footnote{For simplicity, in this setup we do not consider budget specification.} is also based on paper topics $\mathcal{T}^P := \left\{\mathcal{T}^{P}_j\right\}_{j \in [N]}$, which tend to be broader in scope than the output dimension topics $\left\{\mathcal{T}^{O}_i\right\}_{i \in [M]}$ in many of our setups. Each prompt is specified by a set of topics to include $G^{\text{inc}} \subseteq \mathcal{T}^P$, a connector operation $\text{op}^{\text{inc}} \in \left\{\text{and}, \text{or}\right\}$ between the included topics, a disjoint set of topics $G^{\text{exc}} \subseteq \mathcal{T}^P$ (where $G^{\text{inc}} \cap G^{\text{exc}} = \emptyset)$ to explicitly exclude, and a connector operation $\text{op}^{\text{exc}} \in \left\{\text{and}, \text{or}\right\}$ between the excluded topics. If an inclusion or exclusion set is empty or size one, we omit the corresponding connector 
    operation.  
    \item \textit{Aggregation:} We consider two types of aggregation rules: $\mathcal{A}^{\cup}$ takes a union of the lists (based on $\addagg$), and $\mathcal{A}^{\cap}$ takes a intersection of the lists (based on $\intagg$). 
\end{itemize}

We build on our model to capture elicitability, but relax these concepts to allow for closeness in output vectors rather than necessarily requiring an exact match. We define the closeness of two output vectors $\bx \in \mathbb{R}^M_{\ge 0}$ and $\bx' \in \mathbb{R}^M_{\ge 0}$ in terms of their $\ell_1$ norm i.e., $\|\bx - \bx'\|_1$. We capture elicitability using prompting over the prompt topics $\mathcal{T}^P$.  An output vector $\bx$ is \textit{elicitable} by a single model if there exists a prompt specification $(G^{\text{inc}} \subseteq \mathcal{T}^P, \text{op}^{\text{inc}}, G^{\text{exc}} \subseteq \mathcal{T}^P, \text{op}^{\text{exc}})$ over prompt topics that produces an output vector close to $\bx$. Given a target vector $\bx_{\text{target}}$, let $\bx^*_P(\bx_{\text{target}})$ denote the output vector closest to $\bx_{\text{target}}$ that can be produced with some prompt specification over prompt topics.

\paragraph{Empirical setup.} All of our experiments use GPT-4o-mini at temperature 0.7. For prompting, we convert a specification $(G^{\text{inc}}, \text{op}^{\text{inc}}, G^{\text{exc}}, \text{op}^{\text{exc}})$ into natural language by connecting the included topics using $\text{op}^{\text{inc}}$, and then saying ``excluding'' followed by connecting the excluded topics using $\text{op}^{\text{exc}}$. We use an LLM-as-a-judge (GPT-4o-mini at temperature 0) to classify whether each paper title falls within the topic, and then we compute the fraction of papers that are classified as positive.\footnote{Note this reflects the perception of the paper's category given the title, rather than whether the paper truly belongs to the category at hand.} To study elicitability, we do a brute force search over the space of prompt specifications. We average over 30 trials to determine the output vector $\bx \in \mathbb{R}_{\ge 0}^M$ associated with a given prompt specification. We defer other details of our setup to Appendix \ref{appendix:empiricaldetails}.

In Appendix \ref{appendix:additionalexperiments}, we show that our empirical results readily generalize to different model configurations. Specifically, we consider different models (GPT-5.4, GPT-5-mini), different temperatures, and aggregation of heterogeneous models. Our experiments provide empirical support for all three mechanisms in almost all of the experimental variations. This robustness illustrates that our findings hold beyond the specific assumptions of our theoretical framework.

For each mechanism, we aim to construct an aggregation operation $\bx^{(1)}, \bx^{(2)} \rightarrow \bx^{(A)}$ that implements the mechanism and also is elicitability-expanding. 

\subsection{Results}

We show our results in Figure \ref{fig:empirical-mechanisms} and Table \ref{tab:empirical_results}, and we explain them in detail below. 

\paragraph{Support Expansion.} We construct a setup that resembles \Cref{ex:power_support} shown in Figure \ref{fig:support-expansion}.  
Let there be $M=3$ output dimensions and $N=2$ prompt features. The output dimension topics are:
\begin{align*}
    \mathcal{T}^O_1 &: \text{complexity theory (computational complexity, P vs NP, complexity classes, hardness results,} \\
    &\text{reductions, circuit complexity, space/time complexity bounds)} \\
    \mathcal{T}^O_2 &: \text{macroeconomics (GDP, inflation, monetary policy, fiscal policy, business cycles, economic growth,} \\
    &\text{unemployment, central banking, aggregate demand/supply)} \\
    \mathcal{T}^O_3 &: \text{mechanism design (auction design, incentive compatibility, social choice, matching markets,} \\
    &\text{market design, algorithmic game theory, incentive mechanisms)}
\end{align*}
The prompt topics are $\mathcal{T}^P_1: \text{computer science theory}$ and $\mathcal{T}^P_2: \text{economics}$. We take the aggregation rule to be $\mathcal{A}^{\cup}$. We let $\bx^{(1)}$ capture the output from prompt specification $(G_1^{\text{inc}} = \left\{\mathcal{T}^P_1\right\}, G_1^{\text{exc}} = \emptyset)$, and we let $\bx^{(2)}$ capture the output from prompt specification $(G_2^{\text{inc}} = \left\{\mathcal{T}^P_2 \right\}, G_1^{\text{exc}} = \emptyset)$. 

The results are shown in Figure \ref{fig:empirical-support-expansion} and \Cref{tab:support_expansion}. We find on average that $\bx^{(A)} = [0.43, 0.47, 0.04]$, that $\bx^*_P(\bx^{(A)}) = [0.74, 0.00, 0.07]$.
We obtain a confidence bound of $[0.63, 1.02]$ for $\ell_1$ distance between $\bx^{(A)}$ and $\bx^*_P(\bx^{(A)})$.
The vector $\bx_P^*(\bx^{(A)})$ is obtained via brute-force search over prompt specifications.

These results suggest that  $\bx^{(1)}, \bx^{(2)} \rightarrow \bx^{(A)}$  implements support expansion and is elicitability-expanding. Specifically, $\bx^{(1)}$ has most of its nonzero weight on topic $\mathcal{T}^O_1$,  $\bx^{(2)}$ has most of its nonzero weight on topic
$\mathcal{T}^O_2$, and $\bx^{(A)}$ reflects both topics $\mathcal{T}^O_1$ and $\mathcal{T}^O_2$. $\bx^{(A)}$ is not elicitable by a single model, since $\bx^*_P(\bx^{(A)})$ misses out on $\mathcal{T}^O_2$. This illustrates that no prompt is able to simultaneously activate topics $\mathcal{T}^O_1$ and $\mathcal{T}^O_2$, but aggregation is able to overcomes this limitation by separately considering $\mathcal{T}^O_1$ and $\mathcal{T}^O_2$. On the other hand, if we prompt using output dimension topics $\mathcal{T}^O$ as opposed to the prompt topics $\mathcal{T}^P$, it is possible to simultaneously activate both topics:  we find on average that the output vector equals $[0.52, 0.51, 0.00]$ if we use the
prompt specification $(G^{inc}=\{\mathcal{T}^O_1, \mathcal{T}^O_2\}, \text{op}^{\text{inc}} = \text{or}, G^{exc} = \emptyset)$.

\paragraph{Binding set contraction.} We construct an example that resembles Example \ref{ex:power_other} shown in Figure \ref{fig:binding-set-contraction}, but with a greater number of output dimensions.  
Let there be $M=5$ output dimension topics and $N=2$ prompt dimension topics. The output dimension topics are:
\begin{align*}
    \mathcal{T}^O_1 &: \text{deep learning (transformers, attention mechanisms, deep neural networks,} \\
    &\text{modern architectures like BERT, GPT, ViT)} \\
    \mathcal{T}^O_2 &: \text{non-transformer NLP methods (RNN, LSTM, GRU, word2vec, seq2seq without attention,}\\
    & \text{traditional NLP)} \\
    \mathcal{T}^O_3 &: \text{non-transformer CV methods (CNN, ConvNets, ResNet, VGG, pooling, convolutional architectures)} \\
    \mathcal{T}^O_4 &: \text{multimodal methods bridging multiple modalities (combining text and images,} \\
    &\text{vision-language models, image captioning, VQA)} \\
    \mathcal{T}^O_5 &: \text{statistical machine learning (non-neural methods like SVM, random forests, logistic regression,}\\
    &\text{Bayesian methods, traditional ML))}
\end{align*}
The prompt dimension topics are $\mathcal{T}^P_1: \text{natural language processing (NLP)}$ and  $\mathcal{T}^P_2: \text{computer vision (CV)}$. We take the aggregation rule to be $\mathcal{A}^{\cap}$. We let $\bx^{(1)}$ be the output from prompt specification $(G^{inc}_1 = \{\mathcal{T}^P_1\}, G^{exc}_1 = \{\mathcal{T}^P_2\})$, and we let $\bx^{(2)}$ be the output from prompt specification  $(G^{inc}_2 = \{\mathcal{T}^P_2\}, G^{exc}_2 = \{\mathcal{T}^P_1\})$. 

The results are shown in Figure \ref{fig:empirical-binding-contraction} and Table \ref{tab:binding_contraction}. We find on average that ${\bx}^{(A)} = [0.87, 0.00, 0.00, 0.00, 0.00]$, that $\bx^*_P(\bx^{(A)}) = [0.83, 0.11, 0.00, 0.01, 0.01]$. We get a confidence bound of $[0.07, 0.35]$ for $\ell_1$ distance between $\bx^{(A)}$ and $\bx^*_P(\bx^{(A)})$.
The vector $\bx^*_P(\bx^{(A)})$ is obtained via brute-force search over prompt specifications.

The results suggest that $\bx^{(1)}, \bx^{(2)} \rightarrow \bx^{(A)}$ implements binding-set contraction and it is elicitability-expanding. 
Specifically, since $\mathcal{T}^O_1: $ deep learning is an umbrella topic covering many papers in topics $\mathcal{T}^O_2, \mathcal{T}^O_3, \mathcal{T}^O_4$, we expect that there is a semantic constraint capturing how lists of papers which include topics $\mathcal{T}^O_2$, $\mathcal{T}^O_3$ or  $\mathcal{T}^O_4$ also tend to include topic $\mathcal{T}^O_1$. 
Both $\bx^{(1)}$ and $\bx^{(2)}$
appear to be binding relative to this constraint, while ${\bx}^{(A)}$ is in the interior. 
$\bx^{(A)}$ is not elicitable by a single model: while ${\bx}^{(A)}$ predominantly reflects topic $\mathcal{T}^O_1$, the output vector $\bx^*_P(\bx^{(A)})$ also reflects some topics outside of $\mathcal{T}^O_1$. On the other hand, if we prompt using output dimension topics $\mathcal{T}^O$ rather than the prompt topics $\mathcal{T}^P$, it is possible to slightly reduce the fraction of papers reflecting topics outside of $\mathcal{T}^O_1$: we find on average that the output vector equals $[0.96, 0.00, 0.01, 0.01, 0.00]$  if we use the
prompt specification $(G^{inc} = \{\mathcal{T}^O_1\}, G^{exc} = \{\mathcal{T}^O_2, \mathcal{T}^O_3, \mathcal{T}^O_4, \mathcal{T}^O_5\}, \text{op}^{\text{exc}} = \text{or})$.

\paragraph{Feasibility Expansion.} We construct a setup that resembles Example \ref{ex:power_feasibility} shown in Figure \ref{fig:feasibility-expansion}, but without the reward specification limitations. Let there be $M=3$ output dimensions, corresponding to topics: 
\begin{align*}
    \mathcal{T}^O_1 &: \text{blockchain(consensus protocols, smart contracts, decentralized ledgers, proof-of-work, proof-of-stake)} \\
  \mathcal{T}^O_2&: \text{cryptography (encryption, zero-knowledge proofs, hash functions, digital signatures, key exchange)} \\
\mathcal{T}^O_3 &: \text{distributed systems (consensus algorithms, fault tolerance, replication, distributed databases,}\\
 &\text{CAP theorem)}
\end{align*}
We take $\mathcal{T}^P = \mathcal{T}^O$, which means that $M = N = 3$ and also that feasibility and elicitability are equivalent. We take the aggregation rule to be $\mathcal{A}^{\cap}$. We let $\bx^{(1)}$ capture the output from prompt specification $(G_1^{\text{inc}} = \left\{\mathcal{T}^O_1, \mathcal{T}^O_2 \right\}, \text{op}^{\text{inc}} = \text{or}, G_1^{\text{exc}} = \left\{ \mathcal{T}^O_3 \right\})$, and we let $\bx^{(2)}$ capture the output from prompt specification $(G_2^{\text{inc}} = \left\{\mathcal{T}^O_1, \mathcal{T}^O_3 \right\}, \text{op}^{\text{inc}} = \text{or}, G_2^{\text{exc}} = \left\{ \mathcal{T}^O_2 \right\})$.  

The results are shown in Figure \ref{fig:empirical-feasibility-expansion} and \Cref{tab:feasibility_expansion}. We find on average that $\bx^{(A)} = [0.67, 0.00, 0.22]$, that $\bx_P^* (\bx^{(A)}) = [0.67, 0.00, 0.39]$. We get a confidence bound of $[0.07, 0.39]$ for the $\ell_1$ distance between $\bx^{(A)}$ and $\bx_P^* (\bx^{(A)})$, which notably does not contain $0$, illustrating a clear gap.
The vector $\bx_P^* (\bx^{(A)})$ is obtained via brute-force search over prompt specifications.

The results suggest that  $\bx^{(1)}, \bx^{(2)} \rightarrow \bx^{(A)}$ implements feasibility expansion and is elicitability-expanding.
 Intuitively, we expect there to be the following constraint faced by agents (whether humans or LLMs) in this task: many blockchain papers ($\mathcal{T}^O_1$) are related to cryptography ($\mathcal{T}^O_2$) or to distributed systems ($\mathcal{T}^O_3$). Both $\bx^{(1)}$ and $\bx^{(2)}$ reflect this constraint, but $\bx^{(A)}$ circumvents this constraint through aggregation.
 $\bx^{(A)}$ is not elicitable by a single model: $\bx^*_P(\bx^{(A)})$ has substantially more weight on $\mathcal{T}^O_3$ (distributed systems) than $\bx^{(A)}$ does, while matching on $\mathcal{T}^O_1$. The aggregation operation produces a vector whose distributed-systems content is markedly lower than what any single prompt can achieve.

%% file: arxiv_sections/discussion_arxiv.tex
\section{Discussion}\label{sec:discussion}

In this work, we theoretically study how aggregating multiple copies of the same model gives access to a greater set of outputs than using only a single model. Building on a principal-agent framework, our results show how aggregation must implement one of three mechanisms---feasibility-expansion, support expansion, and binding set contraction---in order to expand the set of elicitable outputs. Although these mechanisms are not sufficient to ensure that aggregation adds power, we show a more precise condition formed from strengthening the mechanisms is sufficient. Finally, we empirically illustrate these mechanisms by deploying LLMs in a toy reference-generation task, demonstrating the robustness of our findings beyond the assumptions of our stylized model.

\paragraph{Connecting our mechanisms to empirical phenomena.} To complement our empirical support for the three mechanisms in the toy reference-generation task, we discuss how the mechanisms connect with more broadly observed empirical phenomena, which we hope inspire future empirical work. Since aggregation is only powerful when individual models are limited on their own, we begin by outlining the single-model limitations underlying each mechanism and the empirical phenomena supporting them.
\begin{itemize}[leftmargin=*,nosep]
\item The power of feasibility expansion traces back to limitations in the types of outputs that individual models can generate: specifically, when models can’t exhibit certain (desirable) dimensions without exhibiting other (undesirable) dimensions as a side effect. This side effect has been empirically observed for safety versus overrefusal, where models which refuse a larger fraction of toxic outputs tend to refuse a larger fraction of safe outputs as a side effect \citep{Cui2025ORBench}. Similar side effects have been observed for alignment and hedging \citep{ouyang2022training}, and theoretically studied for creativity and factuality \citep{Sinha2023CreativityReality}.
\item The power of support expansion traces back to challenges with eliciting outputs that perform along multiple dimensions at once in single-agent settings. This limitation has been empirically observed in cases where each dimension corresponds to a distinct user requirement. For example, prompts are often underspecified, since users may not include all of the requirements that they care about in the prompt \citep{Yang2025PromptUnderspecification}. Moreover, even when users specify all their requirements, LLMs struggle to satisfy many requirements simultaneously \citep{Wen2024ComplexBench,Guo2025RECAST}.
\end{itemize}
We leave pinpointing empirical phenomena which support binding set contraction---whose emergence depends on the interaction between prompt-engineering and model limitations---to future work. More broadly, since our results identify when aggregation enables these mechanisms, an important direction is to connect them to practice by testing whether real aggregation methods (e.g., debate \citep{Du2024MultiagentDebate}, prompt ensembling \citep{Arora2023AMA}) exhibit them. The single-model limitations discussed above suggest promising empirical settings where aggregation should add power.

\paragraph{Model limitations and extensions. }Our stylized model, which builds on a classical principal-agent framework \citep{Kleinberg2019Effort}, makes simplifying assumptions for tractability. First, we assume for simplicity that the agent chooses an output deterministically. Nonetheless, our empirical findings robustly generalize to LLMs with nonzero temperature which exhibit stochasticity (Section \ref{sec:empirical}). It would be interesting to extend our theoretical results to capture this randomness and to allow the system designer to choose a temperature for each agent. Moreover, while our analysis allows for nonlinear rewards $R$, we restrict the output constraints (i.e., model limitations) and the feature map (i.e., prompt engineering limitations) to linear functional forms. Extending our model to allow for nonlinear limitations, which would complicate the structure of the agent’s optimization program, is an interesting direction for future work. Furthermore, we also assume each agent’s reward depends only on its own outputs, though richer interdependencies may arise in repeated, multi-turn interactions \citep{Du2024MultiagentDebate}. Finally, while our analysis focuses on steering agents through reward design, it would be interesting to incorporate other choices, such as tool use and fine-tuning, that enable specialization in compound AI systems \citep{BAIR2024CompoundAISystems}.

%% file: arxiv_sections/appendix_arxiv.tex
\section{Additional details and proofs for Section \ref{sec:concrete}}

\subsection{Insufficiency of natural mechanisms implementation for elicitability-expansion}\label{sec:insufficiency}

While implementing one of feasibility-expansion, support expansion, or binding set contraction is necessary for an aggregation operation to be elicitability-expanding as shown by \Cref{thm:weaker_necessary}, in this section, we will show that it is not sufficient.

Support expansion on its own or binding set contraction on its own is not sufficient for power. The following proposition states the insufficiency of support expansion relative to every input vector of the aggregation operation for elicitability-expansion.

\begin{proposition}\label{prop:insufficient_support}
Fix $\bconstraints = \emptyset$. There exists an aggregation operation $\bx^{(1)}, \bx^{(2)}, \bx^{(3)} \rightarrow \bx^{(A)}$ such that $\bx^{(1)}, \bx^{(2)}, \bx^{(3)} \rightarrow \bx^{(A)}$ 
implements support expansion relative to $i$ for $i \in \{1,2\}$. However, $\bx^{(1)}, \bx^{(2)}, \bx^{(3)} \rightarrow \bx^{(A)}$ is not elicitability-expanding. 
\end{proposition}
\begin{proof}
This follows mainly from \Cref{prop:alllargesubsets}. This Corollary shows that an additional condition beyond support expansion is required for an operation $\aggop$ to be elicitability-expanding when $\calV(\bx^{(1)}) = \ldots = \calV(\bx^{(K)}) = \calV(\bx^{(A)})$. This is the condition that if $\bx^{(A)}$ has full support, then there is a $j \in [M]$ and $k \in [K]$ such that $[M] \setminus \{j\} \not \subseteq \calS(\bx^{(k)})$.

We will now construct a problem instance with three output dimensions and an aggregation operation that is support expanding but fails the additional condition. Additionally this problem instance has no conic constraints making the binding constraint set the empty set for all vectors. Consider the aggregation operation $\bx^{(1)} = [0,1,1], \bx^{(2)} = [1,0,1], \bx^{(3)} = [1,1,0] \rightarrow \bx^{(A)} = [1,1,1]$. This aggregation operation fails the necessary condition for elicitability-expansion stated in \Cref{prop:alllargesubsets}. 
\end{proof}

Similarly, binding set contraction relative to every input vector of the aggregation operation also does not guarantee that aggregation has power. The next proposition formalizes this.

\begin{proposition}\label{prop:insufficient_binding}
There exists an aggregation operation $\bx^{(1)}, \bx^{(2)} \rightarrow \bx^{(A)}$ and a set of conic constraints $\bconstraints$ such that $\aggop$ 
implements binding set contraction relative to $i$ for $i \in \{1,2\}$. However, $\bx^{(1)}, \bx^{(2)} \rightarrow \bx^{(A)}$  is not elicitability-expanding. 
\end{proposition}
\begin{proof}
Consider a problem with two output dimensions having the following two constraints: 1) $c_1: x_1-x_2 \le 0$, 2) $c_2: -2x_1 + x_2 \le 0$. Consider an aggregation operation $\bx^{(1)} = [1,1], \bx^{(2)} = [1,2] \rightarrow \bx^{(A)} = [5, 7]$, where the binding constraints sets are $\calV_{\bx^{(i)}} = \{c_i\}$ for $i \in \{1,2\}$ and $\calV_{\bx^{(A)}} = \emptyset$ and the supports are the entire set of output dimensions. Hence the feasibility-improving and budget-reducing directions are $\mathcal{B}_{\calS(\bx^{(A)}), \calV(\bx^{(A)})} =\{\bd: \mathds{1}^\top \bd < 0\}, \mathcal{B}_{\calS(\bx^{(1)}), \calV(\bx^{(1)})} =\{\bd: \bd_1 - \bd_2 \le 0, \mathds{1}^\top \bd < 0\}, \mathcal{B}_{\calS(\bx^{(2)}), \calV(\bx^{(2)})} =\{\bd: -2\bd_1 + \bd_2 \le 0, \mathds{1}^\top \bd < 0\}$. 

First note that this operation is not feasibility-expanding since $\bx^{(A)}$ satisfies the conic constraints. We will show that there is no $\bd \in \mathcal{B}_{\calS(\bx^{(A)}), \calV(\bx^{(A)})}$ such that $\bd \not 
\le 0$ and $\bd$ is in neither $\mathcal{B}_{\calS(\bx^{(1)}), \calV(\bx^{(1)})}$ nor in $\mathcal{B}_{\calS(\bx^{(2)}), \calV(\bx^{(2)})}$. This along with the lack of feasibility-expansion implies violation of the alternate power-characterizing condition in \Cref{def:alt_condition} which implies that the aggregation operation cannot be elicitability-expanding by \Cref{thm:necessary} and \Cref{prop:equivalence}.

Any $\bd \in \mathcal{B}_{\calS(\bx^{(A)}), \calV(\bx^{(A)})}$ such that $\bd \not \in \mathcal{B}_{\calS(\bx^{(1)}), \calV(\bx^{(1)})} \cup \mathcal{B}_{\calS(\bx^{(2)}), \calV(\bx^{(2)})}$ satisfies $\bd_1 - \bd_2 > 0$ and $-2\bd_1 + \bd_2 > 0$. These inequalities satisfy $\bd_2 < \bd_1  < 0$. Hence such a $\bd$ cannot satisfy $\bd \not \le 0$.
 
\end{proof}

On the other hand, feasibility expansion on its own guarantees power under some feature mapping as stated by the following proposition.  
\begin{proposition}\label{prop:feasibilityexpansion}
Fix conic constraints $\bconstraints$. If an aggregation operation $\aggop$ 
implements feasibility-expansion, then it is elicitability-expanding.
\end{proposition}
\begin{proof}
    From the definition of the power-characterizing condition (\Cref{def:characterizing_condition}), this condition is satisfied if the aggregation operation implements feasibility-expansion. \Cref{thm:sufficient} showing the sufficiency of the power-characterizing condition for elicitability-expansion implies that implementing feasibility-expansion is sufficient for elicitability-expansion. 
\end{proof}

\subsection{Proofs for Section \ref{subsec:examples}}

Recall that the examples in this section use the feature weights matrix
$\balpha(q) := 
\begin{bmatrix}
    1 & 0 & q \\
    0 & 1 & q 
\end{bmatrix}$.
\subsubsection{Analysis of Example \ref{ex:power_feasibility}}\label{app:power_feasibility}

\begin{proposition}
\label{prop:power_feasibility} For the feature weights matrix $\balpha_2$ and constraint matrix with the row $x_3 \le x_1 + x_2$ in \Cref{ex:power_feasibility}, the aggregation operation $\bx^{(1)} = [1, 0, 1], \bx^{(2)} = [0, 1, 1]  \rightarrow \bx^{(A)} = [0, 0, 1]$ is elicitability-expanding.
\end{proposition}
\begin{proof}
From the construction, it is easy to see that $x^{(1)}$ can be elicited with a linear reward function $[1, 0, 0]$ equal to the $F_1$ and budget level $E = 1$ and $x^{(2)}$ can be elicited with a linear reward function $[0, 1, 0]$ equal to the $F_1$ and budget level $E = 1$. 

Let us use our characterization \Cref{thm:fixed_feature_map} to formally demonstrate elicitability-expansion. 
    
The support sets of $\bx^{(1)}, \bx^{(2)}$ are $\calS(\bx^{(1)}) = \{1,3\}$ and $\calS(\bx^{(2)}) = \{2,3\}$ respectively. The single conic constraint is binding for both $x^{(1)}, x^{(2)}$. The set $\calB_{\calS(\bx^{(1)}),\calV(x^{(1)})}$ is $\{\bd: \bd_3 \le \bd_1 + \bd_2, \bd_2 \ge 0, \bd_1 + \bd_2 + \bd_3 < 0\}$. For any $\bd$ in this set, $\bd_3 < 0$ and $\bd_1 + \bd_2 < -\bd_3$. The set $\calB_{\calS(\bx^{(2)}),\calV(x^{(2)})}$ is $\{\bd: \bd_3 \le \bd_1 + \bd_2, \bd_1 \ge 0, \bd_1 + \bd_2 + \bd_3 < 0\}$. For any $\bd$ in this set, $\bd_3 < 0$ and $\bd_1 + \bd_2 < -\bd_3$. 

Now consider the set of feature-improving direction $\{\bd: \bd_1 + 2\bd_3 \ge 0, \bd_2 + 2 \bd_3 \ge 0\}$. For any $\bd$ in this set, $\bd_1 + \bd_2 \ge -4d_3$.

All three conditions $\bd_3 < 0$, $\bd_1 + \bd_2 < -\bd_3$, and $\bd_1 + \bd_2 \ge -4d_3$ cannot be satisfied since for $\bd_3 < 0$, $-\bd_3 < -4\bd_3$. Hence there is no intersection between feasibility improving directions and features improving directions and $x^{(1)}$ is elicitable. Similarly, $x^{(2)}$ is also elicitable.

$\bx^{(A)}$ is not feasible and hence not elicitable. This shows that $\bx^{(1)}, \bx^{(2)} \rightarrow \bx^{(A)}$ is elicitability-expanding by implementing feasibility-expansion.
\end{proof}

\subsubsection{Analysis of Example \ref{ex:power_support}}\label{app:power_support}

\begin{proposition}\label{prop:power_support}For the feature weights matrix $\balpha({0.6})$ and null constraint matrix in  \Cref{ex:power_support}, the aggregation operation $\bx^{(1)} = [1, 0, 0], \bx^{(2)} = [0, 1, 0]  \rightarrow \bx^{(A)} = [1/2, 1/2, 0]$ is elicitability-expanding.

\end{proposition}
\begin{proof}
    The set of directions $\Bcal_{\mathcal{S}(\bx^{(1)}), \mathcal{V}(\bx^{(1)})}  = \{\bd: \bd_2 \ge 0, \bd_3 \ge 0, \bd_1 + \bd_2 + \bd_3 < 0 \}$. And the set of feature-improving directions is $\calA = \{\bd: \bd_1 + 0.6 \bd_3 \ge 0, \bd_2 + 0.6 \bd_3 \ge 0\}$.

    $\bd \in \calB_{\mathcal{S}(\bx^{(1)}), \mathcal{V}(\bx^{(1)})}$ means that $\bd_1 < -(\bd_2 + \bd_3) < -\bd_3$ and $\bd_3 \ge 0$. $\bd \in \calA_1$ means that $\bd_1 \ge - 0.6 \bd_3$. These three conditions cannot be simultaneously showing that $\bx^{(1)}$ is elicitable due to empty intersection of $\calA$ and $\calB_1$. Symmetrically, we can also show that $\bx^{(2)}$ is also elicitable. 

    Now let us argue that $\bx^{(A)} = [1/2,1/2,0]$ is not elicitable. $\calB_{\mathcal{S}(\bx^{(A)}), \mathcal{V}(\bx^{(A)})} = \{\bd: \bd_3 \ge 0, \bd_1 + \bd_2 + \bd_3 < 0\}$. Consider $\bd = [-0.6, -0.6, 1]$. $\bd \in \calA \cap \calB_A$. This shows that $\bx^{(A)}$ is not elicitable.

\end{proof}

\subsubsection{Analysis of Example \ref{ex:power_other}}\label{app:power_other}

\begin{proposition}
\label{prop:power_other}
For the feature weights matrix $\balpha({0.2})$ and conic constraint matrix with one constraint $x_1 + x_2 \le x_3$ from  \Cref{ex:power_other}, the aggregation operation $\bx^{(1)} = [1, 0, 1], \bx^{(2)} = [0, 1, 1]  \rightarrow \bx^{(A)} = [0, 0, 1]$ is elicitability-expanding.
\end{proposition}

\begin{proof}[Proof of \Cref{prop:power_other}]
    The feature-improving directions are the set $\calA = \{\bd: \bd_1 + 0.2 \bd_3 \ge 0, \bd_2 + 0.2 \bd_3 \ge 0\}$. 

    The constraint is binding at both $\bx^{(1)}$ and $\bx^{(2)}$. The feasibility improving directions are $\mathcal{B}_{(\mathcal{S}(\bx^{(1)}), \mathcal{V}(\bx^{(1)})} = \{\bd: \bd_1 + \bd_2 \le \bd_3, \bd_2 \ge 0, \bd_1 + \bd_2 + \bd_3 < 0 \}$.

    If $\bd \in \calB_{\mathcal{S}(\bx^{(1)}), \mathcal{V}(\bx^{(1)})} \cap \mathcal{A}$, then $\bd_1 + \bd_2 + \bd_3 < 0$, $\bd_1 + \bd_2 + \bd_3 \le 2\bd_3$, $\bd \in \calA$, $\bd_1 \ge -0.2 \bd_3$, and $\bd_2 \ge -0.2 \bd_3$. This implies that $\bd_3 < 0$. If all the conditions are satisfied simultaneously, then $\bd_1 > 0$ and $\bd_2 > 0$. This contradicts $\bd_1 + \bd_2 \le \bd_3 < 0$.

    The conic constraint is not binding at $\bx^{(A)}$. Now consider the feasibility improving directions of $\bd^{(A)}$: $\calB_{(\mathcal{S}(\bx^{(A)}), \mathcal{V}(\bx^{(A)})} = \{\bd: \bd_1 \ge 0, \bd_2 \ge 0, \bd_1 + \bd_2 + \bd_3 < 0 \}$. The vector $\bd = (0.2, 0.2, -1) \in \calA \cap \calB_A$ demonstrating that $\bx^{(A)}$ is not elicitable. 
\end{proof}

\input{arxiv_sections/proofs_special_cases}

\section{Additional details and proofs for Section \ref{sec:general} }

\subsection{Characterization of single output elicitation }

A key technical tool for our characterization of elicitability through aggregation (provided by \Cref{thm:fixed_feature_map}) is the characterization of direct elicitability of a vector $\bx$ with a single agent, given a feature weights matrix $\balpha$. This characterization is in terms of the intersection of feasible perturbation directions $\calB_{\calS(\bx), \calV(\bx)} = \{\bd: C_{\calV(\bx)} \bd \le 0\} \cap \{\bd: \bd_{\calS(\bx)^c} \ge 0\} \cap \{\mathds{1}^t\bd < 0\}$ and feature-improving directions $d \in \reals^M$: $\{\bd: \balpha \bd \ge 0\}$. It is stated in the following proposition and it generalizes the characterization results in \citet{Kleinberg2019Effort} to allow for conic constraints $\bconstraints$.

\begin{proposition}[Single output elicitation characterization]\label{prop:single_output_characterization}
Given a feature weights matrix $\balpha$ and conic constraints $\bconstraints$, an output vector $\bx \gneq 0$ is elicitable if and only if it is feasible i.e., $\bconstraints \bx \le 0$ is and  $\calB_{\calS(\bx), \calV(\bx)} \cap \{\bd: \balpha \bd \ge 0\}$ is empty.
\end{proposition}

We will prove this proposition by proving the necessary and sufficient directions through \Cref{lem:single_output_elicitation_necessary} and \Cref{lem:single_output_elicitation_sufficient} respectively. 

To prove these lemmas we will make use of the monotonicity property we assume our reward functions satisfy. Recall that we assume that reward functions do not decrease if all
features are weakly increased, and strictly increase if some feature is increased.

A consequence of this monotonicity of rewards is that the budget necessary to elicit a vector $\bx$ is its $\ell_1$ norm $\|\bx\|_1$. This is shown in the following lemma. 

\begin{lemma}\label{lem:full_budget}
        If a vector $\bx$ is elicitable with budget $E$, then $\|\bx\|_1 = E$.
    \end{lemma}
    \begin{proof} 
        If $\bx$ is elicitable, then it must be feasible. So $\|\bx\|_1 \le E$ due to the budget constraint and $\bconstraints \bx \le 0$. For any monotone reward function $R$, there is a feature $F_j$ such that increasing the value of feature $F_j$ strictly increases the reward $R$. Since we assume that $\balpha$ has no zero rows (\Cref{ass:feature-maps}), there is a coordinate $i$ in the $j^{\text{th}}$ row of $\balpha$ that is non-zero. Additionally, by \Cref{ass:non-zero-feasible}, there is a vector $\by \ge 0$ with $\by_i > 0$ that satisfies $\bconstraints \by \le 0$. Consider a scaled version $\by'$ of $\by$ with $\ell_1$ norm less than $E - \|\bx\|_1$. That is $\by' = (E - \|\bx\|_1) \by/\|\by\|_1$. The vector $\bx' = \bx + \by'$ has a strictly higher value of feature $F_j$ and no lower values on other features. That is, $\balpha \bx' \ge \balpha \bx$ and  $(\balpha \bx')_j > (\balpha \bx)_j$. As a result, the reward function has a higher value on $\bx'$ than on $\bx$. $\bx'$ satisfies the budget constraint since $\|\bx\|_1 \le \|\bx\|_1 + \|\by'\|_1  \le E$. It also satisfies conic constraints $\bconstraints$ since both $\bx$ and $\by'$ satisfy $\bconstraints$. This shows that for any reward function, it is possible to construct another feasible vector with a higher reward if $\bx$ has $\ell_1$ norm less than the effort level $E$. Therefore, $\bx$ is not elicitable if $\|\bx\|_1 < E$.
    \end{proof}

\begin{lemma}[Single output elicitation necessary]\label{lem:single_output_elicitation_necessary}
Given a feature weights matrix $\balpha$ and conic constraints $\bconstraints$, an output vector $\bx \gneq 0$ is elicitable only if $\bx$ is feasible  and $\calB_{\calS(\bx), \calV(\bx)} \cap \{\bd: \balpha \bd \ge 0\}$ is empty. 
\end{lemma}
\begin{proof}[Proof of \Cref{lem:single_output_elicitation_necessary}]
We will show that if $\calB_{\calS(\bx), \calV(\bx)} \cap \{\balpha \bd \ge 0\}$ is non-empty then $\bx$ is not elicitable. If the intersection is non-empty, consider any $\bd \in \calB_{\calS(\bx), \calV(\bx)} \cap \{\balpha \bd \ge 0\}$, we will use this $\bd$ to construct a feasible output vector $\by$ with $\|\by\|_1 \le \|\bx\|_1$ and having strictly higher reward than $\bx$ for every monotone reward function of the features. This vector $\by$ we construct is $\by = \|\bx\|_1(\bx + \lambda \bd) / \|\bx + \lambda \bd\|_1$ for an appropriate choice of $\lambda$ that we will describe shortly. 

    First consider the vector $\by' = \bx + \lambda \bd$. Note that $\by'$ is feasible on all conic and non-negativity constraints that are binding at $\bx$ due to $\bd$'s membership in $\{\bd: C_{\calV(\bx)} \bd \le 0\} \cap \{\bd: \bd_{\calS(\bx)^c} \ge 0\}$. We can choose $\lambda$ to be small enough so that $\by'$ continues to meet all non-binding constraints. That is choose $\lambda < \min_{j \in  \calV(\bx)^c, C_j \bd > 0 } - \boldsymbol{C}_{j} \bx /\boldsymbol{C}_{j} \bd$ and $\min_{i \in  \calS(\bx), \bd_i < 0 } - \bx_i /\bd_i$. This establishes that we have a positive choice of $\lambda$ making $\by'$ satisfy the nonnegativity and conic constraints. Additionally, we have that $\mathds{1}^t\by' = \|\by'\|_1 = \|\bx\|_1 + \lambda \mathds{1}^t d < \|\bx\|_1 = 1$. That is, $\by'$ satisfies any budget constraint that is satisfied by $\bx$, and satisfies with a strictly larger margin. Hence $\by'$ is feasible relative to the conic constraints and budget constraints from any level of specified budget.

    We also have that $\balpha^t \by' = \balpha^t (\bx + \lambda \bd) \ge \balpha^t \bx$ since $\balpha^t \bd \ge 0$. Hence $\by'$ satisfies feasibility constraints and has at least as high values on all features. By \Cref{lem:full_budget}, $y'$ is not elicitable with budget $\|\bx\|_1$ since $\|\by'\|_1 < \|\bx\|_1$. This means that for any monotone reward function $R$, there is a feasible vector $\bx'$ with $\|\bx'\| \le \|\bx\|_1$ having $R(\bx') > R(\by')$. Since $\balpha \by' \ge \balpha \bx$, $R(\by') \ge R(\bx)$. In particular, $\bx'$ has a strictly higher reward than $\bx$. This means that $\bx$ cannot be elicited with effort level $\|\bx\|_1$ which is the only possible level at which $\bx$ can be elicited by \Cref{lem:full_budget}. So, $\bx$ cannot be elicited if $\mathcal{B}_{\calS(\bx), \calV(\bx)} \cap \{\balpha \bd \ge 0\} \neq \emptyset$.       
\end{proof}
\newcommand{\ones}{\boldsymbol{1}}
\begin{lemma}[Single output elicitation sufficient]\label{lem:single_output_elicitation_sufficient}
Given feature weights matrix $\balpha$ and conic constraints, $\bconstraints$, an output vector $\bx \ge 0$ is elicitable if $\bx$ satisfies the conic constraints i.e., $\bconstraints \bx \le 0$ and  $\calB_{\calS(\bx), \calV(\bx)} \cap \{\bd: \balpha \bd \ge 0\}$ is empty. 
\end{lemma}

\begin{proof}
Let us denote $S := \calS(\bx)$ and $V := \calV(\bx)$. And let $D_{\balpha}$ be the set of feature-improving directions $\{\balpha \bd \ge 0\}$. Under the condition that $\bx$ is feasible and $\calB_{S,V} \cap D_{\balpha}$ is empty, we will show that $\bx$ is elicitable by explicitly constructing a reward function that elicits $\bx$ with budget $\|\bx\|_1$.

\paragraph{Dual certificate of empty intersection.} The condition of empty intersection of the sets $\calB_{S,V}$ and $D_{\balpha}$ can equivalently be written as the infeasibility of the system :
\begin{equation}\label{eq:infeasible-system}
C_V d \le 0,\qquad I_{S^c} d \ge 0,\qquad \alpha^\top d \ge 0,\qquad \ones^\top d < 0,
\end{equation}
where \(I_{S^c} \in \R^{|S^c|\times M}\) is the subset of rows of the $M \times M$ identity matrix indexed by the set $S^c$.

By Motzkin’s transposition theorem, infeasibility of \eqref{eq:infeasible-system} implies the existence of dual variables
\[
\gamma \in \R^{|V|}_{\ge 0},\quad
\lambda \in \R^{|S^c|}_{\ge 0},\quad
\nu \in \R^{N}_{\ge 0},\quad
\tau > 0
\]
such that 
\begin{equation}\label{eq:motzkin-combo}
C_V^\top \gamma \;-\; I_{S^c}^\top \lambda \;+\; \tau \mathbf{1}  \;-\; \balpha \, \nu \;=\; 0
\end{equation}

\paragraph{Reward function construction.} 
We will now define a reward function that is linear in the features to elicit a vector $\bx$. We will later show how this reward function along with budget $\|\bx\|_1$ elicits the vector $\bx$ when the empty intersection condition is satisfied. The reward function is 
\[
R(\bu) \;=\; \sum_{i=1}^N \beta_i f_i((\balpha^\top \bu)_i)
\qquad\text{with}\qquad
\beta_i \;:=\; \frac{\nu_i}{f_i'\bigl((\balpha^\top \bx)_i\bigr)},
\]
which is well-defined since each \(f_i\) is strictly increasing, hence \(f_i'((\balpha^\top \bx)_i)>0\).
Because each \(f_i\) is concave and increasing, \(R\) is concave. Its gradient at \(x\) is
\[
\nabla R(\bx)
\;=\;
\sum_{i=1}^N \beta_i f_i'\bigl((\balpha^\top \bx)_i\bigr)\,\alpha_{\cdot, i}
\;=\;
\alpha\, \nu,
\]
where \(\alpha_{\cdot, i}\) is the \(i\)-th column of \(\alpha\).

\paragraph{Elicitability.} 
From \Cref{lem:full_budget} we know that a vector $\bx$ can only be elicited with a budget of $E = \|\bx\|_1$. So let us consider  the constrained reward maximization program with this budget $E$ and a reward function $R$. We will show that when $\calB_{S,V}$ and $D_{\balpha}$ have empty intersection, the dual certificate of this empty intersection that we computed before along with $\bx$ satisfies the KKT conditions of the reward-maximizing optimization program.
\[
\max_{\bu \in \R^M}\;\; R(\boldsymbol{f}(\balpha^\top \bu))
\quad\text{s.t.}\quad
\bconstraints \bu \le 0,\;\; \bu \ge 0,\;\; \ones^\top \bu \le E.
\]
This is a concave program, and its Lagrangian is
\[
\mathcal{L}(\bu,\lambda_0,\boldsymbol{\mu},\boldsymbol{\tilde\gamma})
\;=\;
R(\boldsymbol{f}(\balpha^\top \bu)) \;+\; \lambda_0\,(E-\ones^\top \bu) \;+\; \boldsymbol{\mu}^\top \bu \;-\; \boldsymbol{\tilde\gamma}^\top (\bconstraints \bu),
\]
with dual variables \(\lambda_0 \ge 0\), \(\boldsymbol{\mu} \ge 0\), \(\boldsymbol{\tilde\gamma} \ge 0\).
Evaluate the KKT conditions at \(\bu=\bx\) with the choice of dual variables from the dual certificate of the empty intersection condition. That is, the dual variables are
\[
\lambda_0 := \tau,\qquad
\mu_{S} := 0,\;\; \mu_{S^c} := \lambda,\qquad
\tilde\gamma_{V} := \gamma,\;\; \tilde\gamma_{V^c} := 0.
\]
Primal feasibility holds due to feasibility of $\bx$ by definition of \(S,V\).
Complementary slackness holds since \(x_j=0\) for \(j\in S^c\) and \((Cx)_\ell=0\) for \(\ell \in V\), while \(\mu_S=0\).
For stationarity,
\[
\nabla R(x) \;-\; \lambda_0 \ones \;+\; \mu \;-\; C^\top \tilde\gamma
\;=\;
\alpha \nu \;-\; \tau \ones \;+\; I_{S^c}^\top \lambda \;-\; C_V^\top \gamma
\;=\; 0
\]
by \eqref{eq:motzkin-combo}.

Since \(R\) is concave and the constraints are linear, the KKT conditions are sufficient to certify optimality; hence \(x\) maximizes \(R\) over the feasible region and is therefore elicitable.
\end{proof}

\begin{proof}[Proof of \Cref{prop:single_output_characterization}]
This follows from \Cref{lem:single_output_elicitation_necessary}, \Cref{lem:single_output_elicitation_sufficient} showing necessity and sufficiency of the condition of non-emptiness of $\calB_{\calS(\bx), \calV(\bx)} \cap \{\balpha \bd \ge 0\}$ for elicitability of $\bx$.
\end{proof}

An immediate consequence of the condition characterizing elicitability of a single output vector from \Cref{prop:single_output_characterization} is that the role of the budget is only in scaling the $\ell_1$ norm of elicitable output vectors. 

\begin{corollary}\label{lem:only_directions}
    A vector $\bx$ is elicitable with some budget $E$ if and only $\bx/\|\bx\|_1$ is elicitable with budget $1$ for the same reward function.
\end{corollary}
\begin{proof}
    The condition characterizing the elicitability of $\bx$ given a feature weights matrix $\balpha$ is whether or not the sets $\calB_{\calS(\bx), \calV(\bx)}$ and $\{\bd: \balpha \bd \ge 0\}$ intersect (\Cref{thm:fixed_feature_map}). Note that the sets $\calB_{\calS(\bx), \calV(\bx)}$ and $\calB_{\calS(\bx/\|\bx\|_1), \calV(\bx/\|\bx\|_1)}$ are equal because the supports and binding sets of $\bx$ and $\bx/\|\bx\|_1$ are the same. As a result, $\bx$ is elicitable if and only if $\bx/\|\bx\|_1$ is elicitable. From \Cref{lem:full_budget}, $\bx$, if elicitable, is elicitable with budget $\|\bx\|_1$ and similarly, $\bx / \|\bx\|_1$, if elicitable, is elicitable with budget $1$.
\end{proof}

\subsection{Proof of \Cref{thm:fixed_feature_map}}\label{proof:fixed_feature_map}
\begin{proof}[Proof of \Cref{thm:fixed_feature_map}]
    This follows directly from the characterization of elicitability of a single output vector $\bx$ for a given feature weights matrix $\balpha$ provided by \Cref{prop:single_output_characterization}. The condition characterizing whether an aggregation operation $\aggop$ is elicitability-expanding given a feature weights matrix $\balpha$ is the condition that each $\bx^{(k)}$ is elicitable under $\balpha$ and $\bx^{(A)}$ is not elicitable. Each of these conditions are provided by \Cref{prop:single_output_characterization}. The condition for $\bx^{(k)}$ to be elicitable under $\bx^{(k)}$ to be elicitable is that $\calB_{\calS(\bx^{(k)}), \calV(\bx^{(k)})}$ and $\{\bd: \balpha \bd \ge 0\}$ have empty intersection and the condition for $\bx^{(A)}$ to not be elicitable is that $\calB_{\calS(\bx^{(A)}), \calV(\bx^{(A)})}$ and $\{\bd: \balpha \bd \ge 0\}$ have non-empty intersection.
\end{proof}

\subsection{Proof of \Cref{prop:alllargesubsets}}\label{proof:alllargesubsets}

\begin{proof}[Proof of \Cref{prop:alllargesubsets}]

The assumptions are 
(i) no conic constraint is binding for these vectors,
and (ii) $\bx^{(A)}$ is not full support (i.e., $\calS(\bx^{(A)}) \neq [M]$) or there exists $j \in [M]$ such that no $\bx^{(k)}$ has support $[M] \setminus \{j\}$. 

By Theorem \ref{thm:necessary} and Theorem \ref{thm:sufficient}, the aggregation operation is elicitability-expanding if and only if the power-characterizing condition is satisfied. Because of assumption (i), $\calV(\bx^{(k)}) = \emptyset$ for every $k \in [K]$, so $\bgamma^{(k)} \in \preals^{|\calV(\bx^{(k)})|}$ is trivially zero in the second case of \Cref{def:characterizing_condition}. The condition therefore reduces to either feasibility expansion or a strengthened form of support expansion (involving only $\blambda^{(k)}$). It suffices to show that under assumptions (i) and (ii), strengthened support expansion is equivalent to support expansion.

We know that (ii) coupled with support expansion implies that 
either
\[
\calS(\bx^{(A)}) \not\subseteq \calS(\bx^{(k)})\ \ \text{for all }k\in[K]
\quad\text{and}\quad
\calS(\bx^{(A)})\neq [M],
\]
or 
\[\calS(\bx^{(A)}) =  [M] \;\; \text{and } \;\;\calS(\bx^{(k)}) \neq [M] \forall  k \in [K] \;\;  \text{and} \;\;  \text{ there exists } j\in[M] \text{ s.t. }  \calS(\bx^{(k)}) \neq [M]\setminus\{j\} \forall k \in [K].\]

\paragraph{A useful equivalent form of (ii) coupled with support expansion.} We first show that (ii) coupled with support expansion is equivalent to the existence of indices
$j(k)\in \calS(\bx^{(A)})\setminus \calS(\bx^{(k)})$ for each $k\in[K]$ such that the set
\[
J \coloneqq \bigcup_{k\in[K]}\{j(k)\}
\]
is a \emph{strict} subset of $[M]$.
Indeed, for any selection of $j(k)$ we always have $J\subseteq \calS(\bx^{(A)})$.
If $\calS(\bx^{(A)})\neq [M]$, then necessarily $J\subsetneq [M]$.
Otherwise, $\calS(\bx^{(A)})=[M]$, and the existence of some $j$ with
$[M]\setminus\{j\}\not\subseteq \calS(\bx^{(k)})$ for all $k$
is equivalent to being able to choose all $j(k)$ from $[M]\setminus\{j\}$,
which forces $J\subseteq [M]\setminus\{j\}\subsetneq [M]$.

\paragraph{(ii) coupled with support expansion $\Rightarrow$ Strengthened support expansion.}
Assume the above condition holds: there exists $J\subsetneq [M]$ and
indices $j(k)\in \calS(\bx^{(A)})\setminus \calS(\bx^{(k)})$ with
$J=\bigcup_{k\in[K]}\{j(k)\}$.
Fix any $c>0$ and define $\bd\in\mathbb{R}^M$ coordinate-wise by
\[
\bd_j \coloneqq
\begin{cases}
-1, & j\in J,\\
c, & j\notin J,
\end{cases}
\]
where $c$ is a constant chosen such that 
\[\frac{|J|-1}{M-|J|} < c < \frac{|J|}{M-|J|}.\]
Since $J\subseteq \calS(\bx^{(A)})$, setting $\bd_j<0$ on $J$ does not violate feasibility for
coordinates outside the support, and we have $\bd_{\calS(\bx^{(A)})^c}\ge 0$.
Moreover,
\[
\mathds{1}^\top \bd
= -|J| + (M-|J|)c < 0.
\]
Thus $\bd\in \calB_{\calS(\bx^{(A)}),\emptyset}$.
Finally, for each $k\in[K]$ we have $\bd_{j(k)}=-1$ and therefore
\[
-|\mathds{1}^\top \bd| = -|J| + (M-|J|)c > -1 = \bd_{j(k)},
\]
so strengthened support expansion holds.

\paragraph{Strengthened support expansion $\Rightarrow$ support expansion.} This follows by the same argument as Proposition \ref{lem:stronger_implies_weaker}.
\end{proof}

\subsection{Proof of equivalence between the versions of the power-characterizing condition defined in \Cref{def:characterizing_condition} and \Cref{def:alt_condition}}\label{proof:equivalence}

To prove \Cref{thm:necessary} and \Cref{thm:sufficient}, we will use an alternate but equivalent way of expressing the power-characterizing condition defined in \Cref{def:characterizing_condition}. This equivalent condition is defined in \Cref{def:alt_condition}. The equivalence between both conditions is stated in the following proposition.

\begin{proposition}
\label{prop:equivalence}
    The conditions defined in \Cref{def:characterizing_condition} and \Cref{def:alt_condition} are equivalent.
\end{proposition}
\begin{proof}

For ease of notation, let $V_k := \calV(\bx^{(k)})$, $S_k := \calS(\bx^{(k)})$ for $k \in [K]$, and let $V_A := \calV(\bx^{(A)})$, $S_A := \calS(\bx^{(A)})$. Both definitions agree in the feasibility-expansion case, so it suffices to compare their second cases. We will show the following per-$k$ equivalence: for any fixed $\bd \in \calB_{S_A, V_A}$ with $\bd \not \le 0$ and any $k \in [K]$,
\begin{enumerate}[label=\textnormal{(\arabic*)}, leftmargin=*]
    \item \label{cond:empty} $\bigl\{\bu + \lambda \bd : \bu \in \preals^M,\ \lambda \ge 0\bigr\} \cap \calB_{S_k, V_k} = \varnothing$, \quad if and only if
    \item \label{cond:certificate} there exist $\bgamma^{(k)} \in \preals^{|V_k|}$ and $\blambda^{(k)} \in \preals^{|S_k^c|}$ such that, writing $\bw^{(k)} := \bigl(\bgamma^{(k)}\bigr)^{\!\top} \bconstraints_{V_k} - \bigl(\blambda^{(k)}\bigr)^{\!\top} I_{S_k^c}$,
    \[\bw^{(k)} \bd \;+\; \bigl|\mathds{1}^\top \bd\bigr| \cdot \min_{j \in [M]} \min\!\bigl(\bw^{(k)}_j,\, 0\bigr) \;>\; 0.\]
\end{enumerate}
The per-$k$ equivalence implies equivalence of the two definitions: the second case of \Cref{def:alt_condition} requires \ref{cond:empty} to hold for every $k \in [K]$ for a single $\bd$, and the second case of \Cref{def:characterizing_condition} requires \ref{cond:certificate} to hold for every $k \in [K]$ for a single $\bd$.

\paragraph{Reformulating empty intersection as feasibility of an LP.}
The intersection in \ref{cond:empty} is non-empty if and only if there exist $\bu \in \preals^M$ and $\lambda \ge 0$ such that
\[\bconstraints_{V_k}(\bu + \lambda \bd) \le \mathbf{0},\qquad -I_{S_k^c}(\bu + \lambda \bd) \le \mathbf{0},\qquad \mathds{1}^\top(\bu + \lambda \bd) < 0.\]
Both \ref{cond:empty} and the inequality in \ref{cond:certificate} are positively homogeneous in $\bd$: scaling $\bd \mapsto t \bd$ for $t > 0$ rescales $\lambda$ in \ref{cond:empty} and rescales both terms equally in \ref{cond:certificate}. We may therefore normalize $\bd$ so that $\mathds{1}^\top \bd = -1$; under this normalization, $|\mathds{1}^\top \bd| = 1$. Note that $\lambda$ must be strictly positive: $\bu \ge \mathbf{0}$ implies $\mathds{1}^\top \bu \ge 0$, while $\mathds{1}^\top(\bu + \lambda \bd) < 0$ together with $\mathds{1}^\top \bd = -1$ forces $\lambda > 0$. Setting $\bv := \bu / \lambda \in \preals^M$ and dividing the inequalities by $\lambda$, the non-empty intersection is equivalent to the existence of $\bv \in \preals^M$ with $\mathds{1}^\top \bv < 1$ such that
\[\bconstraints_{V_k}(\bd + \bv) \le \mathbf{0} \qquad \text{and} \qquad -I_{S_k^c}(\bd + \bv) \le \mathbf{0}.\]

\paragraph{Dualizing the inequality system.}
A finite system of linear inequalities $A \by \le \mathbf{0}$ holds if and only if every non-negative weighted combination of its rows holds, i.e., $\bgamma^\top A \by \le 0$ for every $\bgamma \ge \mathbf{0}$. Applying this to the system above, the non-empty intersection is equivalent to the existence of $\bv \in \preals^M$ with $\mathds{1}^\top \bv \le 1$ such that
\[\bigl(\bgamma^{(k)\top} \bconstraints_{V_k} - \blambda^{(k)\top} I_{S_k^c}\bigr)(\bd + \bv) \;\le\; 0 \qquad \forall \bgamma^{(k)} \in \preals^{|V_k|},\ \blambda^{(k)} \in \preals^{|S_k^c|}.\]
We have replaced the strict inequality $\mathds{1}^\top \bv < 1$ with the closed inequality $\le 1$; this is without loss of generality because the value of the objective is continuous in $\bv$ on a closed feasible set. We may also restrict $\bgamma^{(k)}, \blambda^{(k)}$ to bounded norm $\|\bgamma^{(k)}\|_1 \le 1$, $\|\blambda^{(k)}\|_1 \le 1$ without loss of generality, since the objective is positively homogeneous in $(\bgamma^{(k)}, \blambda^{(k)})$. So the non-empty intersection is equivalent to
\[\inf_{\substack{\bv \in \preals^M \\ \mathds{1}^\top \bv \le 1}}\ \sup_{\substack{\bgamma^{(k)} \in \preals^{|V_k|},\ \|\bgamma^{(k)}\|_1 \le 1 \\ \blambda^{(k)} \in \preals^{|S_k^c|},\ \|\blambda^{(k)}\|_1 \le 1}}\ \bw^{(k)}(\bd + \bv) \;\le\; 0,\]
where $\bw^{(k)} := \bgamma^{(k)\top} \bconstraints_{V_k} - \blambda^{(k)\top} I_{S_k^c}$.

\paragraph{Applying the minimax theorem.}
The objective $\bw^{(k)}(\bd + \bv)$ is bilinear in $(\bgamma^{(k)}, \blambda^{(k)})$ and $\bv$, and both feasible sets are convex and compact. Sion's minimax theorem applies, so we may swap the order of $\inf$ and $\sup$:
\[\sup_{\substack{\bgamma^{(k)} \in \preals^{|V_k|},\ \|\bgamma^{(k)}\|_1 \le 1 \\ \blambda^{(k)} \in \preals^{|S_k^c|},\ \|\blambda^{(k)}\|_1 \le 1}}\ \inf_{\substack{\bv \in \preals^M \\ \mathds{1}^\top \bv \le 1}}\ \bw^{(k)}(\bd + \bv) \;\le\; 0.\]

\paragraph{Closed form of the inner infimum.}
For fixed $\bgamma^{(k)}, \blambda^{(k)}$ (and hence fixed $\bw^{(k)}$), we compute
\[\inf_{\bv \in \preals^M,\ \mathds{1}^\top \bv \le 1} \bw^{(k)} \bv.\]
If $\bw^{(k)}$ has any strictly negative coordinate, the infimum is $\min_j \bw^{(k)}_j$, achieved by placing all of $\bv$'s mass on the most negative coordinate of $\bw^{(k)}$. If $\bw^{(k)} \ge \mathbf{0}$, the infimum is $0$, achieved at $\bv = \mathbf{0}$. In both cases, $\inf_{\bv} \bw^{(k)} \bv = \min_{j \in [M]} \min(\bw^{(k)}_j, 0)$, so
\[\inf_{\substack{\bv \in \preals^M \\ \mathds{1}^\top \bv \le 1}} \bw^{(k)}(\bd + \bv) \;=\; \bw^{(k)} \bd + \min_{j \in [M]} \min(\bw^{(k)}_j, 0).\]

\paragraph{Identifying the certificate.}
Combining, the non-empty intersection is equivalent to
\[\sup_{\substack{\bgamma^{(k)} \in \preals^{|V_k|},\ \|\bgamma^{(k)}\|_1 \le 1 \\ \blambda^{(k)} \in \preals^{|S_k^c|},\ \|\blambda^{(k)}\|_1 \le 1}} \Bigl(\bw^{(k)} \bd + \min_{j \in [M]} \min(\bw^{(k)}_j, 0)\Bigr) \;\le\; 0,\]
i.e., $\bw^{(k)} \bd + \min_{j \in [M]} \min(\bw^{(k)}_j, 0) \le 0$ for every bounded-norm $\bgamma^{(k)}, \blambda^{(k)} \ge \mathbf{0}$. Negating and using positive homogeneity in $(\bgamma^{(k)}, \blambda^{(k)})$ to drop the bounded-norm restriction, the empty intersection \ref{cond:empty} (under the normalization $\mathds{1}^\top \bd = -1$) is equivalent to: there exist $\bgamma^{(k)} \in \preals^{|V_k|}$ and $\blambda^{(k)} \in \preals^{|S_k^c|}$ such that $\bw^{(k)} \bd + \min_{j \in [M]} \min(\bw^{(k)}_j, 0) > 0$. Undoing the normalization (using positive homogeneity in $\bd$ to restore the factor $|\mathds{1}^\top \bd|$), \ref{cond:empty} is equivalent to: there exist $\bgamma^{(k)} \in \preals^{|V_k|}$ and $\blambda^{(k)} \in \preals^{|S_k^c|}$ such that
\[\bw^{(k)} \bd \;+\; \bigl|\mathds{1}^\top \bd\bigr| \cdot \min_{j \in [M]} \min\!\bigl(\bw^{(k)}_j,\, 0\bigr) \;>\; 0,\]
which is exactly \ref{cond:certificate}.
\end{proof}

\subsection{Proof of Theorem \ref{thm:necessary}}\label{proof:necessary}

Using the equivalence of the power-characterizing condition in \Cref{def:characterizing_condition} and the alternative power-characterizing condition in \Cref{def:alt_condition} that we established in \Cref{prop:equivalence}, we will show the necessity of the power-characterizing condition for elicitability-expansion by showing the necessity of the alternate condition. In the proof of Theorem \ref{thm:necessary}, we will use the single-agent characterization of elicitability (\Cref{lem:single_output_elicitation_sufficient}) rather than the multi-agent characterization of elicitability-expansion (\Cref{thm:fixed_feature_map}). 

\begin{proof}[Proof of Theorem \ref{thm:necessary}]

We will prove the contrapositive. Suppose that the power-characterizing condition  is not satisfied. By Proposition \ref{prop:equivalence}, this means that the alternate power-characterizing condition (\Cref{def:alt_condition}) is also not satisfied. Violation of the alternate power-characterizing condition means that $\aggop$ is not feasibility-expanding. That is, $\bx^{(A)}$ and each $\bx^{(k)}$ for $k \in [K]$ is feasible.

We will show that $\aggop$ is not elicitability-expanding by showing that for any feature-weights matrix $\balpha$ that makes $\bx^{(A)}$ not elicitable, there exists some $k \in [K]$ such that $\bx^{(k)}$ is also not elicitable under $\balpha$.
 
For any feature weights matrix $\balpha$ that makes $\bx^{(A)}$ inelicitable, by Lemma \ref{lem:single_output_elicitation_sufficient}, there is a $\bd^{(A)} \in \mathcal{B}_{\calS(\bx^{(A)}), \calV(\bx^{(A)})}$ such that $\balpha \bd^{(A)} \ge 0$. This $\bd$ must have some positive entry i.e., $\bd^{(A)} \not \le 0$. Since $\mathds{1}^\top \bd^{(A)} < 0$,  $\bd^{(A)}$ must have some strictly negative coordinate. 
%This would violate $\balpha \bd^{(A)} \ge 0$ since we assume $\balpha$ has no zero rows. 

Since the alternate power-characterizing condition is violated, there exists $\bx^{(k)}$ with $\mathcal{B}_{\calS(\bx^{(k)}), \calV(\bx^{(k)})}$ having non-empty intersection with $\{\bu + \lambda {\bd}^{(A)}\}$. It suffices to show that $\bx^{(k)}$ is not elicitable under feature mapping $\balpha$. To see this, let $\bd^{(k)}$ denote an element of the intersection $\mathcal{B}_{\calS(\bx^{(k)}), \calV(\bx^{(k)})} \cap \{\bu + \lambda \bd^{(A)}\}$. We can then write $\bd^{(k)} = \bu + \lambda \bd^{(A)}$. Note that $\balpha \bd_i = \balpha \bu + \lambda \balpha \bd^{(A)}$. We know that $\balpha \bu \ge 0$ since $\bu \ge 0$ and $\balpha$ has non-negative entries. Additionally, $\balpha \bd^{(A)} \ge 0$. Hence $\balpha \bd^{(k)} \ge 0$ for $\bd^{(k)} \in \mathcal{B}_{\calS(\bx^{(k)}), \calV(\bx^{(k)})}$. By Lemma \ref{lem:single_output_elicitation_necessary}, this means that $\bx^{(k)}$ is not elicitable.

\end{proof}

\subsection{Proof of Theorem \ref{thm:sufficient}}

\begin{proof}[Proof of Theorem \ref{thm:sufficient}]\label{proof:sufficient}
    Suppose the power-characterizing condition is satisfied. By Proposition \ref{prop:equivalence}, this means that the alternate power-characterizing condition is also satisfied. Then we know that we are in one of two cases. 

    \paragraph{Case 1: $\aggop$ implements feasibility expansion.} Consider a feature mapping with a single feature and all dimensions contribute equal weights of one to this feature. All output vectors with the same $\ell_1$ norm result in the same reward for all reward functions, and thus all feasible outcomes are elicitable. That is any output vector is elicitable if and only if it is feasible. A feasible output is elicitable with budget equal to its $\ell_1$ norm. Under this construction, feasibility-expansion implies elicitability-expansion. 

    \paragraph{Case 2: there exists $\bd^{(A)} \in \mathcal{B}_{S(\bx^{(A)}, \mathcal{V}(\bx^{(A)})}$ with $\bd^{(A)} \not \le 0$ such that for all $\bu \ge 0, \lambda \ge 0$, $\bu + \lambda \bd^{(A)} \not \in \mathcal{B}_{\calS(\bx^{(k)}), \calV(\bx^{(k)})}$ for $k \in [K]$.} We will construct a feature mapping $\balpha$ based on $\bd^{(A)}$ such that the set of directions weakly increasing feature values i.e., the set $D_{\balpha} = \{\bd: \balpha \bd \ge 0\}$ is a subset of $\{\bu + \lambda \bd^{(A)}: \bu \ge 0, \lambda \ge 0\}$. By \Cref{thm:fixed_feature_map}, this implies that $\bx^{(A)}$ is not elicitable but for all other outputs $\bx^{(k)}$, $D_{\balpha} \cap \mathcal{B}_{S(\bx^{(k)}, \mathcal{V}(\bx^{(k)})}$ is empty and hence each $\bx^{(k)}$ is elicitable under $\balpha$.

    To complete this argument, we will explicitly construct such an $\balpha$ based on $\bd^{(A)}$. Let $P_0 = \{i \in [m]: \bd^{(A)}_i > 0\}$ denote the positive coordinates of $\bd^{(A)}$ and let $N_0 =  \{i \in [m]: \bd^{(A)}_i \le 0\}$ denote the negative or zero coordinates. Note that $P_0$ is non-empty since $\bd^{(A)} \not \le 0$. We construct two sets of features:
    \begin{itemize}[leftmargin=*]
        \item For every $p \in P_0$, there is a corresponding feature $F_p$ whose row in $\balpha$ is the vector $e_{p}$ which is the vector with 1 at coordinate $p$ and zero everywhere else. That is, dimension $p$ has weight 1 on feature $F_p$ and all other dimensions have zero weight. 
        
    \item The next set of features are defined for every pair  $p \in P_0$, $q \in N_0$. This feature $F_{p,q}$ has a corresponding row in $\balpha$ that is the vector $\bd_p^{(A)} e_q - \bd_q^{(A)} e_p$. That is, the only dimensions with possible non-zero weights to $F_{p,q}$ are dimensions $p, q$. The weight from  dimension $p$ is $|\bd^{(A)}_q|$ and the weight from dimension $q$ is $|\bd^{(A)}_p|$.
    \end{itemize}

    Now let us show that the set $D_{\balpha} = \{\bd: \balpha \bd \ge 0\}$ is a subset of $B_A = \{\bu + \lambda \bd^{(A)}\}$. Take any $\bd \in D_{\balpha}$. 

    For every $p \in P_0$, since $\bd$ weakly improves value of $F_p$, it holds that $\bd_p \ge 0$. By ensuring that $\lambda \le \bd_{p} / \bd^{(A)}_p$ for all $p \in P_0$, we can ensure that $\bd_p - \lambda \bd^{(A)}_p \ge 0$.

    For every $p \in P_0, q \in N_0$, since $\bd$ weakly improves value of $F_{p,q}$, it holds that 
    $-\bd_p \bd^{(A)}_q + \bd_q \bd^{(A)}_p \ge 0$. In other words, $\bd_q \ge \bd_p \bd^{(A)}_q / \bd^{(A)}_p$. By choosing $\lambda$ less than $\bd_p^{(A)} \bd_q / \bd_p \bd^{(A)}_{q}$ for every $q \in N_0$, we get $\bd_q - \lambda \bd^{(A)}_q \ge 0$ for every $q \in N_0$.
    
\end{proof}

\subsection{How addition and intersection aggregation rules can or cannot implement the mechanisms } \label{sec:aggregation_rules}

In \Cref{subsec:examples}, we showed examples of intersection and addition aggregation rules expanding elicitability by implementing each of the feasibility-expansion, support expansion or binding set contraction mechanisms.  In this part, we will discuss the mechanisms that each aggregation rule can or cannot implement. These results are summarized in \Cref{tab:aggregation-expansion}.

\definecolor{darkgreen}{RGB}{0,100,0}
\begin{table}[h!]
\centering
\renewcommand{\arraystretch}{1.3}
\setlength{\tabcolsep}{8pt}
\begin{tabular}{|l|c|c|c|}
\hline
 & \makecell[c]{\textbf{Feasibility} \\ \textbf{Expansion}}
 & \makecell[c]{\textbf{Support} \\ \textbf{Expansion}}
 & \makecell[c]{\textbf{Binding Set} \\ \textbf{Contraction}} \\
\hline
\textbf{Intersection aggregation}
  & \makecell[c]{\textcolor{darkgreen}{\checkmark} \\ (\Cref{ex:power_feasibility})}
  & \makecell[c]{$\textcolor{red}{\times}$ \\ (\Cref{prop:intersection_support})}
  & \makecell[c]{\textcolor{darkgreen}{\checkmark} \\ (\Cref{ex:power_other})} \\
\hline
\textbf{Addition aggregation}
  & \makecell[c]{$\textcolor{red}{\times}$ \\ (\Cref{prop:add_feasibility})}
  & \makecell[c]{\textcolor{darkgreen}{\checkmark} \\(\Cref{ex:power_support})}
  & \makecell[c]{\textcolor{darkgreen}{\checkmark} \\ (\Cref{ex:add_expand_wo_support})} \\
\hline
\end{tabular}
\caption{Implementability of mechanisms in Section \ref{subsec:examples} for the intersection aggregation rule and addition aggregation rule. The symbol \textcolor{darkgreen}{\checkmark} denotes that there exists a problem instance where the aggregation rule implements that mechanism. The symbol $\textcolor{red}{\times}$ denotes that the aggregation rule does not implement the mechanism for any problem instance.  }
\label{tab:aggregation-expansion}
\end{table}

Intersection aggregation does not implement support expansion for any problem instance, as the following result formalizes.
\begin{proposition}[Intersection does not expand support]\label{prop:intersection_support}
Consider any aggregation operation of the form $\bx^{(1)}, \ldots, \bx^{(K)} \rightarrow \bx^{(A)} = \intaggfun$. For any $k \in [K]$, this aggregation operation does not implement support expansion relative to $k$. 
\end{proposition}
\begin{proof}
    Proposition \ref{prop:intersection_support} follows from the fact that the support of $\intaggfun$ is always a subset of the support of each $\bx^{(k)}$ for $k \in [K]$.
\end{proof}

Intersection aggregation can implement feasibility-expansion as shown in \Cref{ex:power_feasibility} and binding set-contraction as shown in \Cref{ex:power_other}. In fact, these examples go one step further and demonstrate that elicitability expansion is achievable via these mechanisms.

Addition aggregation does not implement feasibility expansion for any problem instance, as the following result formalizes. 
\begin{proposition}[Addition cannot expand feasibility]\label{prop:add_feasibility}
Consider constraints $\bconstraints$. 
Any aggregation operation of the form $\bx^{(1)}, \ldots, \bx^{(K)} \rightarrow \bx^{(A)} = \addaggfun$ does not implement feasibility expansion relative to $\bconstraints$.  
\end{proposition}
\begin{proof}
   Proposition \ref{prop:add_feasibility} directly follows from the fact that the constraint set $\bconstraints$ is conic. 
\end{proof}

On the other hand, addition aggregation operations can implement the other two mechanisms. \Cref{ex:power_support} already constructed a problem instance where addition aggregation implements support expansion. The next example constructs a problem instance where addition aggregation can implement binding set contraction (\Cref{def:bind_contract}) and achieve elicitability-expansion for some feature mapping.

\begin{example}[Addition can result in binding set contraction]\label{ex:add_expand_wo_support} Consider the constraint matrix
\[
C = \begin{pmatrix}
1 & -1 & 0 \\
1 & -\tfrac{1}{4} & -1
\end{pmatrix},
\]
and consider vectors $\bx^{(1)} = (1,1,2)$ and $\bx^{(2)} = (2,4,1)$. Note that they are both feasible and $\bx^{(1)}$ is binding in the first constraint and $\bx^{(2)}$ in the second. Their sum is $\addaggfun(\bx^{(1)}, \bx^{(2)}; [1, 1]) \rightarrow \bx^{(A)} = (3,5,3)$, which is also feasible but does not have any binding constraints.  
\end{example}

\input{arxiv_sections/visualizations}

%% file: arxiv_sections/proofs_special_cases.tex
\subsection{Proofs in \Cref{sec:special_cases}}
\subsubsection{Proof of \Cref{thm:weaker_necessary}}\label{proof:weaker_necessary}

\begin{proof}[Proof of \Cref{thm:weaker_necessary}]
We show this as a corollary of Theorem \ref{thm:necessary}. We will prove this by showing that when both of the conditions in \Cref{thm:weaker_necessary} are violated, the power-characterizing condition (\Cref{def:characterizing_condition}) is violated and hence $\aggop$ cannot be elicitability-expanding. The violation of the conditions in \Cref{thm:weaker_necessary} correspond to lack of implementation of any of the natural mechanisms.

One way of satisfying the power-characterizing condition is through implementing feasibility-expansion. Violating the conditions in \Cref{thm:weaker_necessary} means feasibility-expansion is not implemented.  We will show that the other way of satisfying the power-characterizing condition also does not hold.
    
When the second condition of \Cref{thm:weaker_necessary} theorem is violated, there exists $k \in [K]$ with respect to which $\aggop$ is neither support-expanding nor binding-set contracting. That is, there is a $k$ such that $\calV(\bx^{(k)}) \subseteq \calV(\bx^{(A)})$ and $\calS(\bx^{(k)}) \supseteq \calS(\bx^{(A)})$. As a result, the rows of $C_{\calV(\bx^{(k)})}$ are a subset of the rows in $C_{\calV(\bx^{(A)})}$ and similarly, the rows in $\bd_{\calS(\bx^{(k)})^c}$ are a subset of the rows in  $\bd_{\calS(\bx^{(A)})^c}$. So every $\bd$ in the feasibility-improving and budget-reducing directions set of $\bx^{(A)}$ i.e., $\bd \in \{C_{\calV(\bx^{(A)})}\bd \le 0, \bd_{\calS(\bx^{(A)})^c} \ge 0, \mathds{1}^\top \bd = -1\}$ satisfies the inequalities $C_{\calV(\bx^{(k)})}(\bd) \le 0$ and $\bd_{\calS(\bx^{(k)})^c} \ge 0$. Such a $\bd$ also satisfies all weighted combinations of the inequalities. Hence for any ${\bgamma}^{(k)} \in \preals^{|\calV_k|}$, $({\bgamma^{(k) \top}}) C_{\calV(\bx^{(k)})} \bd - \|(\bgamma^{(k)\top} C_{\calV(\bx^{(k)})})_-\|_{\infty} \le 0$. That is, the second condition of the power-characterizing condition is also violated.   
\end{proof}

%% file: arxiv_sections/visualizations.tex
\section{Empirical setup for Section \ref{subsec:casestudy}}\label{sec:visualizations}

\paragraph{Task and Setup.} We study a citation task where the system designer seeks a list of 10 influential LLM papers spanning five perspectives: (1) ML theory, (2) NLP/CL, (3) cognitive science, (4) AI alignment and human–AI interaction, and (5) multi-agent systems. The system designer issues five prompts, each targeting one perspective, to gpt-4o-mini-2024-07-18 and then aggregates the resulting lists. 
We prompt another LLM (also gpt-4o-mini-2024-07-18) to aggregate these five lists, instantiating aggregation rules that are inspired by intersection aggregation $\intagg$ and union aggregation $\addagg$. Specifically, the model is prompted with \textit{aggregation instructions} along with the five different lists of 10 references, and produces an aggregated list of 10 references. The \textit{intersection-style aggregation instructions} ask for references which are central and broadly relevant across all five perspectives, thus approximating intersection even when the literal overlap of references is empty. The \textit{addition-style} aggregation instructions ask for references that jointly cover and reflect the combined topical space of all five perspectives. We defer the specific prompts and other details of the empirical setup to \Cref{sec:visualizations}. 

\paragraph{Model output generation.} The outputs are generated using gpt-4o-mini-2024-07-18 with the temperature set to $1.0$. These are the five prompts that are used to produce model outputs:
\begin{enumerate}[leftmargin=*]
    \item ``From a machine learning theory perspective, list 10 influential papers that have shaped our current understanding of large language models.''
    \item ``From the perspective of natural language processing and computational linguistics, list 10 key research papers that have been most influential in the development of modern large language models.''
    \item ``From a cognitive science and psycholinguistics standpoint, list 10 important papers that inform our understanding of how large language models represent, process, or acquire linguistic and conceptual structure.''
    \item ``From the standpoint of AI alignment and human–AI interaction, list 10 important papers that have shaped how large language models are aligned, instructed, or trained with feedback.''
    \item ``From a multi-agent and game-theoretic perspective, list 10 influential papers that contribute to the development or understanding of large language models''
\end{enumerate}
These prompts produce five outputs \(X_1,\dots,X_5\), each a list of 10 papers tailored to its respective perspective. Next, we pass the concatenated outputs \((X_1,\dots,X_5)\) to gpt-4o-mini-2024-07-18 by prompting the model with \emph{aggregation instructions} followed by the concatenation of the 5 lists of papers, where each list is preceded by ``List of papers: [insert output number]''. The intersection-style and addition-style aggregation operations are performed using the following \emph{aggregation instructions}. 
\begin{itemize}[leftmargin=*]
    \item \textit{Addition-style aggregation:}
``Each of the following lists contains influential papers on large language models in specializing in different areas: machine learning theory, natural language processing, computational linguistics, AI alignment, human–AI interaction, and multi-agent systems. Based on these lists, generate a new list of 10 papers that reflects the union of their themes and coverage. Your list should be freshly generated (not a literal set union), but it should include papers that plausibly come from any of the provided lists, covering as much of the combined topical space as possible."
\item \textit{Intersection-style aggregation:} ``Each of the following lists contains influential papers on large language models in specializing in different areas: machine learning theory, natural language processing, computational linguistics, AI alignment, human–AI interaction, and multi-agent systems. Based on these lists, generate a new list of 10 papers that reflects their intersection. That is, papers belonging to many of these areas of specialization. Your list should be freshly generated (not a literal intersection), selecting papers that could plausibly appear in all of the lists. If the literal intersection is empty, still generate the best possible list of papers that are central, broadly relevant, and thematically compatible with all lists.''
\end{itemize}
These aggregation prompts produce outputs $X_{\text{addition}},X_{\text{intersection}}$. 

\paragraph{Output vector computation.} We now describe in more detail how we compute the embeddings shown in Figure \ref{fig:embedding-grid}. We embed and visualize the set
\(\{X_1,\dots,X_5,X_{\text{addition}},X_{\text{intersection}}\}\). We calculate the 768-dimensional embeddings using all-mpnet-base-v2 \citep{reimers2021sentencebert}, which is built into the sentence-transformers package in pytorch. To make these embeddings fit into our framework, we translate them to the nonnegative orthant by applying an additive shift $\mathbf{s} \in \mathbb{R}_{\ge 0}^{768}$ To do this, we compute the embeddings of the 805 gpt-4o-mini-2024-07-18 outputs from the helpful-base dataset in AlpacaEval \citep{alpaca_eval}. The additive shift $\mathbf{s}$ is taken to be negative of the minimum coordinate along each dimension in this set of 805 embeddings. We translate all 5 outputs and the aggregated outputs by adding $\mathbf{s}$. We compute the variance across the 5 translated outputs vectors along each of the 768 dimensions, and select the top 2 and top 3 dimensions according to variance. We also compute the $\ell_2$-distance between outputs, which is invariant to the additive shift.

%% file: arxiv_sections/appendix_empirical_arxiv.tex
\section{Other empirical details for Section \ref{sec:empirical}}\label{appendix:empiricaldetails}

\normalsize{}

All of our models use the latest version of GPT-4o-mini as of January 2026. We set the temperature to be 0.7 for generating the output vectors, and set the temperature to be 0 for the LLM-as-a-judge calls to elicit highest probability tokens from the LLM.

\paragraph{Prompting.} A prompt is defined by an inclusion set $G^{inc} \subseteq G$, exclusion set $G^{exc} \subseteq G$ that is disjoint from the inclusion set $G^{inc}$, and operators $\text{op}^{\text{inc}}, \text{op}^{\text{exc}} \in \left\{\text{and}, \text{or}\right\}$. If the inclusion list has more than one element, we choose an inclusion list operator $op_{inc}$ that is one of $\{\mathrm{and}, \mathrm{or}\}$ to combine the topics in the list for inclusion. Similarly, if the exclusion list has more than one element, we choose an exclusion list operator $op_{exc}$ that is $\mathrm{AND}$ or $\mathrm{OR}$. The prompt based on $G^{inc} = \{g^{inc}_1, \ldots, g^{inc}_a\}, G_{exc} = \{g^{exc}_1, \ldots, g^{exc}_b\}$, $op_{inc}, op_{exc}$ is ``List up to 20 papers on $g^{inc}_1$ $op_{inc}$ $\ldots$ $op_{inc}$  $g^{inc}_a$, but EXCLUDE any papers about $g^{exc}_1$ $op_{exc}$ $\ldots$ $op_{exc}$ $g^{exc}_b$. For example, the prompt with $G^{inc} = \{\text{Topic 1}, \text{Topic 2}\}$, $G^{exc} = \{\text{Topic 3}, \text{Topic 4}\}$, $op_{inc}: \mathrm{AND}$, $op_{exc}: \mathrm{OR}$ is \textit{List up to 20 papers on Topic 1 AND Topic 2, but EXCLUDE any papers about Topic 3 or Topic 4. }

\paragraph{LLM-as-judge setup.} We use an LLM-as-judge to measure whether a list $L$ exhibits a given topic $T$ by prompting it.

\begin{verbatim}
For each of the following research papers, determine whether it belongs 
to each topic.
Treat each paper independently - do not let other papers in the list 
influence your classification.

Papers:
{paper_list}

Topics:
{topic_list}

For each paper, independently answer yes/no for each topic.

Output format:
1. {topic_keys_str.replace(', ', ': yes/no, ')}: yes/no
2. {topic_keys_str.replace(', ', ': yes/no, ')}: yes/no
...

SUMMARY (count of "yes" for each topic):
{chr(10).join(f'{key}: [count]' for key in topics.keys())}
\end{verbatim}

\paragraph{Aggregation.} To perform aggregation, we first reformat the titles produced by an LLMs response by making everything lower case, removing punctuations and removing any extra whitespaces. We further remove duplicates by considering two titles equivalent if one is a substring of another. 
The aggregation rule $\mathcal{A}^{\cap}$ intersects the lists and $\mathcal{A}^{\cup}$ takes the union of the lists. Note that in contrast with Appendix \ref{sec:visualizations}, we consider this form of literal set intersection and union as it allows us to consider a non-LLM based aggregation that does not inherit the behavior of the LLMs.

\paragraph{Feasibility and elicitability.} To analyze feasibility and elicitability, we brute force over all possible prompt specifications. This brute force search enables us to identify the closest output vector $x^*(x^{(A)})$ to the aggregated output vector. 

\paragraph{Confidence intervals.} For each output vector, we average over 30 trials and report a coordinate-wise 95\% confidence interval. We use a Wilson confidence interval. To measure a confidence interval for the $\ell_1$ distance between two output vectors, we use the confidence intervals for the two vectors themselves, and then
we compute minimum possible and maximum possible $\ell_1$ distance for those confidence sets. We use these confidence sets to compute $\bx^*_O(\bx^{(A)})$ and $\bx^*_P(\bx^{(A)})$. Specifically, for each candidate output vector $x$, we compute the \textit{lower} bound in the confidence set for the $\ell_1$ distance between $\bx$ and $\bx^{(A)}$, and then we choose the $\bx$ that minimizes this lower bound.

\section{Additional experiments}\label{appendix:additionalexperiments}

Building on the setup in Section \ref{sec:empirical}, we conduct the following additional experiments where we vary our model configuration in three ways. 
\begin{enumerate}
    \item Experimental setup \textbf{E1} considers different settings for the temperature (0.3, 0.5, 0.7), and we report our findings in Tables \ref{tab:empirical_results_4omini_temp_03}-\ref{tab:empirical_results_4omini_temp_07}.
    \item Experimental setup \textbf{E2} considers different models (GPT-5.4 at temperatures 0 and 0.7, and GPT-5-mini at default temperature), and we report our findings in Tables \ref{tab:empirical_results_54_temp_00}-\ref{tab:empirical_results_5mini_temp_none}.
    \item Experimental setup \textbf{E3} considers the aggregation of heterogeneous models (GPT-4o-mini and GPT-5-mini; GPT-4o-mini and GPT-5.4; GPT-5-mini and GPT-5.4), and we report our findings in Tables \ref{tab:empirical_results_4omini_temp_00_x_54_temp_00}-\ref{tab:empirical_results_54_temp_00_x_5mini_temp_none}. 
\end{enumerate}

We summarize our key findings. 
\begin{itemize}
    \item We find that support expansion readily generalizes. In fact, these mechanisms hold for the same instances (i.e., task, aggregation operation, and prompt/output specification) as in our original experiments in all of the cases.
    \item We find that binding-set contraction generalizes in nearly all cases. The exception is one case (Table \ref{tab:empirical_results_5mini_temp_none}; \textbf{E2} for GPT-5-mini), where intersection aggregation produces an empty list ($\bx^{(A)} = \mathbf{0}$), even after switching to a slightly modified prompt specification intended to encourage a non-empty intersection. As with the feasibility expansion failures below, this failure mode is driven by intersection of the produced lists being empty. In all other cases, binding-set contraction holds for the same task, aggregation operation, output specification, and prompt specifications as in our original experiments.
    \item We find that feasibility expansion also generalizes to \textbf{E1} and \textbf{E2}. Again, these mechanisms hold for the same tasks and instances as our original experiments.
However, feasibility expansion is no longer exhibited by the instance that we constructed for \textbf{E3} in two out of the three cases. The failure mode is that intersection aggregation produces an empty list likely due to model heterogeneity. We defer constructing instances that exhibit feasibility expansion for other models to future work.
\end{itemize} 
Altogether, these experiments provide further support for these mechanisms, and additionally demonstrate the robustness of our experiments across almost all the generalized settings.

\FloatBarrier
\subsubsection*{Experiment 1 (E1): Changing temperature}

\begin{table}[h]
     \centering
     \begin{minipage}{0.48\textwidth}
     \centering
     \resizebox{\textwidth}{!}{%
     \begin{tabular}{lcccc}
     \toprule
     & $x_1$ & $x_2$ & $x_3$ \\
     \midrule
     $\bx^{(1)}$  & $[0.63, 0.70]$ & $[0.00, 0.00]$ & $[0.04, 0.08]$ \\
     $\bx^{(2)}$  & $[0.00, 0.00]$ & $[0.75, 0.82]$ & $[0.03, 0.06]$ \\
     \midrule
     $\bx^{(A)}$  & $[0.34, 0.41]$ & $[0.43, 0.51]$ & $[0.03, 0.06]$ \\
     $\bx^*_P(\bx^{(A)})$  & $[0.00, 0.01]$ & $[0.78, 0.88]$ & $[0.01, 0.05]$ \\
     \bottomrule
     \end{tabular}%
     }
     \subcaption{Support Expansion}
    % \label{tab:support_expansion}
     \end{minipage}%
     \hfill
     \begin{minipage}{0.48\textwidth}
     \centering
     \resizebox{\textwidth}{!}{%
     \begin{tabular}{lcccc}
     \toprule
     & $x_1$ & $x_2$ & $x_3$ \\
     \midrule
     $\bx^{(1)}$  & $[0.66, 0.73]$ & $[0.35, 0.42]$ & $[0.01, 0.02]$ \\
     $\bx^{(2)}$  & $[0.65, 0.72]$ & $[0.00, 0.00]$ & $[0.32, 0.40]$ \\
     \midrule
     $\bx^{(A)}$  & $[0.70, 0.77]$ & $[0.00, 0.00]$ & $[0.14, 0.20]$ \\
     $\bx^*_P(\bx^{(A)})$  & $[0.63, 0.75]$ & $[0.00, 0.01]$ & $[0.29, 0.42]$ \\
     \bottomrule
     \end{tabular}%
     }
     \subcaption{Feasibility Expansion}
    % \label{tab:feasibility_expansion}
     \end{minipage}

     \vspace{1em}
     \begin{minipage}{0.6\textwidth}
     \centering
     \resizebox{\textwidth}{!}{%
     \begin{tabular}{lccccc}
     \toprule
     & $x_1$ & $x_2$ & $x_3$ & $x_4$ & $x_5$ \\
     \midrule
     $\bx^{(1)}$ & $[0.77, 0.84]$ & $[0.09, 0.15]$ & $[0.00, 0.00]$ & $[0.00, 0.02]$ & $[0.02, 0.05]$ \\
     $\bx^{(2)}$ & $[0.17, 0.23]$ & $[0.00, 0.00]$ & $[0.64, 0.72]$ & $[0.04, 0.08]$ & $[0.02, 0.05]$ \\
     \midrule
     $\bx^{(A)}$  & $[0.93, 0.96]$ & $[0.00, 0.00]$ & $[0.00, 0.00]$ & $[0.01, 0.03]$ & $[0.00, 0.00]$ \\
     $\bx^*_P(\bx^{(A)})$  & $[0.78, 0.88]$ & $[0.07, 0.15]$ & $[0.00, 0.01]$ & $[0.00, 0.01]$ & $[0.02, 0.07]$ \\
     \bottomrule
     \end{tabular}%
     }
     \subcaption{Binding-set contraction}
   %  \label{tab:binding_contraction}
     \end{minipage}
     \hfill
     \begin{minipage}{0.38\textwidth}
     \centering
     \resizebox{\textwidth}{!}{%
     \begin{tabular}{lc}
     \toprule
     \shortstack[l]{Experiment} & \shortstack{$\|\bx^*_P(\bx^{(A)}) - \bx^{(A)}\|_1$} \\
     \midrule
     Support expansion & [0.59, 0.91] \\
     Feasibility expansion & [0.09, 0.43] \\
     Binding-set contraction & [0.12, 0.45] \\
     \bottomrule
     \end{tabular}%
     }
     \subcaption{$\ell_1$ distance from $\bx^{(A)}$}
     %\label{tab:empirical_angular_distance_summary}
     \end{minipage}
     \caption{Results for \textbf{gpt-4o-mini} and \textbf{temperature 0.3}. The empirical setup and construction of the instance (i.e., task, aggregation operation, and prompt/output specification) is the same as in Section 5. These results demonstrate support expansion, binding-set contraction, and feasibility expansion.
     }
     \label{tab:empirical_results_4omini_temp_03}
     \end{table}

\begin{table}[h]
     \centering
     \begin{minipage}{0.48\textwidth}
     \centering
     \resizebox{\textwidth}{!}{%
     \begin{tabular}{lcccc}
     \toprule
     & $x_1$ & $x_2$ & $x_3$ \\
     \midrule
     $\bx^{(1)}$  & $[0.65, 0.72]$ & $[0.00, 0.00]$ & $[0.03, 0.07]$ \\
     $\bx^{(2)}$  & $[0.00, 0.00]$ & $[0.73, 0.80]$ & $[0.02, 0.05]$ \\
     \midrule
     $\bx^{(A)}$  & $[0.34, 0.42]$ & $[0.44, 0.52]$ & $[0.03, 0.06]$ \\
     $\bx^*_P(\bx^{(A)})$  & $[0.00, 0.01]$ & $[0.76, 0.86]$ & $[0.03, 0.09]$ \\
     \bottomrule
     \end{tabular}%
     }
     \subcaption{Support Expansion}
    % \label{tab:support_expansion}
     \end{minipage}%
     \hfill
     \begin{minipage}{0.48\textwidth}
     \centering
     \resizebox{\textwidth}{!}{%
     \begin{tabular}{lcccc}
     \toprule
     & $x_1$ & $x_2$ & $x_3$ \\
     \midrule
     $\bx^{(1)}$  & $[0.67, 0.74]$ & $[0.36, 0.44]$ & $[0.04, 0.07]$ \\
     $\bx^{(2)}$  & $[0.65, 0.72]$ & $[0.00, 0.00]$ & $[0.33, 0.41]$ \\
     \midrule
     $\bx^{(A)}$  & $[0.56, 0.64]$ & $[0.00, 0.00]$ & $[0.14, 0.20]$ \\
     $\bx^*_P(\bx^{(A)})$  & $[0.61, 0.74]$ & $[0.00, 0.01]$ & $[0.30, 0.43]$ \\
     \bottomrule
     \end{tabular}%
     }
     \subcaption{Feasibility Expansion}
  %   \label{tab:feasibility_expansion}
     \end{minipage}

     \vspace{1em}
     \begin{minipage}{0.6\textwidth}
     \centering
     \resizebox{\textwidth}{!}{%
     \begin{tabular}{lccccc}
     \toprule
     & $x_1$ & $x_2$ & $x_3$ & $x_4$ & $x_5$ \\
     \midrule
     $\bx^{(1)}$ & $[0.80, 0.86]$ & $[0.09, 0.14]$ & $[0.00, 0.00]$ & $[0.00, 0.02]$ & $[0.01, 0.04]$ \\
     $\bx^{(2)}$ & $[0.18, 0.24]$ & $[0.00, 0.01]$ & $[0.62, 0.69]$ & $[0.07, 0.11]$ & $[0.02, 0.05]$ \\
     \midrule
     $\bx^{(A)}$  & $[0.91, 0.95]$ & $[0.00, 0.00]$ & $[0.00, 0.00]$ & $[0.00, 0.00]$ & $[0.00, 0.00]$ \\
     $\bx^*_P(\bx^{(A)})$  & $[0.78, 0.88]$ & $[0.05, 0.13]$ & $[0.00, 0.01]$ & $[0.00, 0.01]$ & $[0.00, 0.04]$ \\
     \bottomrule
     \end{tabular}%
     }
     \subcaption{Binding-set contraction}
  %   \label{tab:binding_contraction}
     \end{minipage}
     \hfill
     \begin{minipage}{0.38\textwidth}
     \centering
     \resizebox{\textwidth}{!}{%
     \begin{tabular}{lc}
     \toprule
     \shortstack[l]{Experiment} & \shortstack{$\|\bx^*_P(\bx^{(A)}) - \bx^{(A)}\|_1$} \\
     \midrule
     Support expansion & [0.56, 0.91] \\
     Feasibility expansion & [0.10, 0.49] \\
     Binding-set contraction & [0.08, 0.37] \\
     \bottomrule
     \end{tabular}%
     }
     \subcaption{$\ell_1$ distance from $\bx^{(A)}$}
     %\label{tab:empirical_angular_distance_summary}
     \end{minipage}
     \caption{Results for \textbf{gpt-4o-mini} and \textbf{temperature 0.5}.
     The empirical setup and construction of the instance (i.e., task, aggregation operation, and prompt/output specification) is the same as in Section 5. These results demonstrate support expansion, binding-set contraction, and feasibility expansion.}
     \label{tab:empirical_results_4omini_temp_05}
     \end{table}

\begin{table}[h]
     \centering
     \begin{minipage}{0.48\textwidth}
     \centering
     \resizebox{\textwidth}{!}{%
     \begin{tabular}{lcccc}
     \toprule
     & $x_1$ & $x_2$ & $x_3$ \\
     \midrule
     $\bx^{(1)}$  & $[0.70, 0.77]$ & $[0.00, 0.00]$ & $[0.04, 0.08]$ \\
     $\bx^{(2)}$  & $[0.00, 0.00]$ & $[0.74, 0.81]$ & $[0.03, 0.06]$ \\
     \midrule
     $\bx^{(A)}$  & $[0.39, 0.47]$ & $[0.43, 0.51]$ & $[0.03, 0.06]$ \\
     $\bx^*_P(\bx^{(A)})$  & $[0.68, 0.80]$ & $[0.00, 0.01]$ & $[0.05, 0.12]$ \\
     \bottomrule
     \end{tabular}%
     }
     \subcaption{Support Expansion}
  %   \label{tab:support_expansion}
     \end{minipage}%
     \hfill
     \begin{minipage}{0.48\textwidth}
     \centering
     \resizebox{\textwidth}{!}{%
     \begin{tabular}{lcccc}
     \toprule
     & $x_1$ & $x_2$ & $x_3$ \\
     \midrule
     $\bx^{(1)}$  & $[0.65, 0.72]$ & $[0.36, 0.44]$ & $[0.06, 0.10]$ \\
     $\bx^{(2)}$  & $[0.65, 0.72]$ & $[0.00, 0.00]$ & $[0.35, 0.42]$ \\
     \midrule
     $\bx^{(A)}$  & $[0.63, 0.70]$ & $[0.00, 0.00]$ & $[0.19, 0.25]$ \\
     $\bx^*_P(\bx^{(A)})$  & $[0.60, 0.73]$ & $[0.00, 0.01]$ & $[0.33, 0.46]$ \\
     \bottomrule
     \end{tabular}%
     }
     \subcaption{Feasibility Expansion}
   %  \label{tab:feasibility_expansion}
     \end{minipage}

     \vspace{1em}
     \begin{minipage}{0.6\textwidth}
     \centering
     \resizebox{\textwidth}{!}{%
     \begin{tabular}{lccccc}
     \toprule
     & $x_1$ & $x_2$ & $x_3$ & $x_4$ & $x_5$ \\
     \midrule
     $\bx^{(1)}$ & $[0.79, 0.85]$ & $[0.09, 0.14]$ & $[0.00, 0.02]$ & $[0.00, 0.02]$ & $[0.03, 0.06]$ \\
     $\bx^{(2)}$ & $[0.16, 0.22]$ & $[0.00, 0.01]$ & $[0.61, 0.69]$ & $[0.04, 0.08]$ & $[0.02, 0.05]$ \\
     \midrule
     $\bx^{(A)}$  & $[0.84, 0.89]$ & $[0.00, 0.00]$ & $[0.00, 0.00]$ & $[0.00, 0.00]$ & $[0.00, 0.00]$ \\
     $\bx^*_P(\bx^{(A)})$  & $[0.78, 0.88]$ & $[0.07, 0.16]$ & $[0.00, 0.01]$ & $[0.00, 0.03]$ & $[0.00, 0.04]$ \\
     \bottomrule
     \end{tabular}%
     }
     \subcaption{Binding-set contraction}
   %  \label{tab:binding_contraction}
     \end{minipage}
     \hfill
     \begin{minipage}{0.38\textwidth}
     \centering
     \resizebox{\textwidth}{!}{%
     \begin{tabular}{lc}
     \toprule
     \shortstack[l]{Experiment} & \shortstack{$\|\bx^*_P(\bx^{(A)}) - \bx^{(A)}\|_1$} \\
     \midrule
     Support expansion & [0.63, 1.02] \\
     Feasibility expansion & [0.07, 0.39] \\
     Binding-set contraction & [0.07, 0.35] \\
     \bottomrule
     \end{tabular}%
     }
     \subcaption{$\ell_1$ distance from $\bx^{(A)}$}
     %\label{tab:empirical_angular_distance_summary}
     \end{minipage}
     \caption{Results for \textbf{gpt-4o-mini} and \textbf{temperature 0.7} (the default model and temperature used in the main text). This is the same data as in Table \ref{tab:empirical_results}.}
     \label{tab:empirical_results_4omini_temp_07}
     \end{table}

\FloatBarrier
\subsubsection*{Experiment 2 (E2): Changing model}
\begin{table}[h]
     \centering
     \begin{minipage}{0.48\textwidth}
     \centering
     \resizebox{\textwidth}{!}{%
     \begin{tabular}{lcccc}
     \toprule
     & $x_1$ & $x_2$ & $x_3$ \\
     \midrule
     $\bx^{(1)}$  & $[0.65, 0.78]$ & $[0.00, 0.01]$ & $[0.00, 0.03]$ \\
     $\bx^{(2)}$  & $[0.00, 0.03]$ & $[0.40, 0.53]$ & $[0.16, 0.28]$ \\
     \midrule
     $\bx^{(A)}$  & $[0.34, 0.48]$ & $[0.42, 0.56]$ & $[0.02, 0.08]$ \\
     $\bx^*_P(\bx^{(A)})$  & $[0.24, 0.31]$ & $[0.29, 0.36]$ & $[0.21, 0.27]$ \\
     \bottomrule
     \end{tabular}%
     }
     \subcaption{Support Expansion}
  %   \label{tab:support_expansion}
     \end{minipage}%
     \hfill
     \begin{minipage}{0.48\textwidth}
     \centering
     \resizebox{\textwidth}{!}{%
     \begin{tabular}{lcccc}
     \toprule
     & $x_1$ & $x_2$ & $x_3$ \\
     \midrule
     $\bx^{(1)}$  & $[0.47, 0.61]$ & $[0.45, 0.59]$ & $[0.20, 0.32]$ \\
     $\bx^{(2)}$  & $[0.36, 0.50]$ & $[0.00, 0.01]$ & $[0.70, 0.81]$ \\
     \midrule
     $\bx^{(A)}$  & $[0.99, 1.00]$ & $[0.00, 0.01]$ & $[0.74, 0.85]$ \\
     $\bx^*_P(\bx^{(A)})$  & $[0.96, 0.99]$ & $[0.05, 0.12]$ & $[0.70, 0.82]$ \\
     \bottomrule
     \end{tabular}%
     }
     \subcaption{Feasibility Expansion}
  %   \label{tab:feasibility_expansion}
     \end{minipage}

     \vspace{1em}
     \begin{minipage}{0.6\textwidth}
     \centering
     \resizebox{\textwidth}{!}{%
     \begin{tabular}{lccccc}
     \toprule
     & $x_1$ & $x_2$ & $x_3$ & $x_4$ & $x_5$ \\
     \midrule
     $\bx^{(1)}$ & $[0.73, 0.85]$ & $[0.20, 0.32]$ & $[0.00, 0.01]$ & $[0.00, 0.01]$ & $[0.00, 0.01]$ \\
     $\bx^{(2)}$ & $[0.31, 0.45]$ & $[0.00, 0.01]$ & $[0.51, 0.65]$ & $[0.09, 0.18]$ & $[0.00, 0.01]$ \\
     \midrule
     $\bx^{(A)}$  & $[0.15, 0.26]$ & $[0.00, 0.01]$ & $[0.00, 0.01]$ & $[0.00, 0.01]$ & $[0.00, 0.01]$ \\
     $\bx^*_P(\bx^{(A)})$  & $[0.82, 0.91]$ & $[0.00, 0.04]$ & $[0.00, 0.04]$ & $[0.00, 0.01]$ & $[0.07, 0.15]$ \\
     \bottomrule
     \end{tabular}%
     }
     \subcaption{Binding-set contraction}
   %  \label{tab:binding_contraction}
     \end{minipage}
     \hfill
     \begin{minipage}{0.38\textwidth}
     \centering
     \resizebox{\textwidth}{!}{%
     \begin{tabular}{lc}
     \toprule
     \shortstack[l]{Experiment} & \shortstack{$\|\bx^*_P(\bx^{(A)}) - \bx^{(A)}\|_1$} \\
     \midrule
     Support expansion & [0.22, 0.76] \\
     Feasibility expansion & [0.03, 0.31] \\
     Binding-set contraction & [0.61, 0.99] \\
     \bottomrule
     \end{tabular}%
     }
     \subcaption{$\ell_1$ distance from $\bx^{(A)}$}
     %\label{tab:empirical_angular_distance_summary}
     \end{minipage}
     \caption{Results for \textbf{gpt-5.4} and a temperature of $0$. The empirical setup and construction of the instance (i.e., task, aggregation operation, and prompt/output specification) is the same as in Section 5. These results demonstrate support expansion, binding-set contraction, and feasibility expansion.
}
     \label{tab:empirical_results_54_temp_00}
     \end{table}

\begin{table}[h]
\centering
\begin{minipage}{0.48\textwidth}
\centering
\resizebox{\textwidth}{!}{%
\begin{tabular}{lcccc}
\toprule
& $x_1$ & $x_2$ & $x_3$ \\
\midrule
$\bx^{(1)}$  & $[0.72, 0.79]$ & $[0.00, 0.00]$ & $[0.01, 0.04]$ \\
$\bx^{(2)}$  & $[0.00, 0.00]$ & $[0.46, 0.54]$ & $[0.17, 0.23]$ \\
\midrule
$\bx^{(A)}$  & $[0.40, 0.48]$ & $[0.44, 0.52]$ & $[0.03, 0.07]$ \\
$\bx^*_P(\bx^{(A)})$  & $[0.29, 0.37]$ & $[0.37, 0.45]$ & $[0.15, 0.21]$ \\
\bottomrule
\end{tabular}%
}
\subcaption{Support Expansion}
\end{minipage}%
\hfill
\begin{minipage}{0.48\textwidth}
\centering
\resizebox{\textwidth}{!}{%
\begin{tabular}{lcccc}
\toprule
& $x_1$ & $x_2$ & $x_3$ \\
\midrule
$\bx^{(1)}$  & $[0.49, 0.57]$ & $[0.48, 0.56]$ & $[0.17, 0.24]$ \\
$\bx^{(2)}$  & $[0.37, 0.45]$ & $[0.00, 0.00]$ & $[0.77, 0.83]$ \\
\midrule
$\bx^{(A)}$  & $[0.98, 1.00]$ & $[0.11, 0.16]$ & $[0.88, 0.93]$ \\
$\bx^*_P(\bx^{(A)})$  & $[0.93, 0.96]$ & $[0.02, 0.04]$ & $[0.68, 0.75]$ \\
\bottomrule
\end{tabular}%
}
\subcaption{Feasibility Expansion}
\end{minipage}

\vspace{1em}
\begin{minipage}{0.6\textwidth}
\centering
\resizebox{\textwidth}{!}{%
\begin{tabular}{lccccc}
\toprule
& $x_1$ & $x_2$ & $x_3$ & $x_4$ & $x_5$ \\
\midrule
$\bx^{(1)}$ & $[0.83, 0.89]$ & $[0.12, 0.17]$ & $[0.00, 0.00]$ & $[0.00, 0.00]$ & $[0.00, 0.02]$ \\
$\bx^{(2)}$ & $[0.33, 0.41]$ & $[0.00, 0.00]$ & $[0.55, 0.63]$ & $[0.11, 0.16]$ & $[0.01, 0.02]$ \\
\midrule
$\bx^{(A)}$  & $[0.30, 0.37]$ & $[0.00, 0.00]$ & $[0.00, 0.00]$ & $[0.00, 0.00]$ & $[0.00, 0.00]$ \\
$\bx^*_P(\bx^{(A)})$  & $[0.71, 0.78]$ & $[0.02, 0.04]$ & $[0.01, 0.03]$ & $[0.00, 0.01]$ & $[0.14, 0.20]$ \\
\bottomrule
\end{tabular}%
}
\subcaption{Binding-set contraction}
\end{minipage}
\hfill
\begin{minipage}{0.38\textwidth}
\centering
\resizebox{\textwidth}{!}{%
\begin{tabular}{lc}
\toprule
\shortstack[l]{Experiment} & \shortstack{$\|\bx^*_P(\bx^{(A)}) - \bx^{(A)}\|_1$} \\
\midrule
Support expansion & [0.12, 0.52] \\
Feasibility expansion & [0.22, 0.47] \\
Binding-set contraction & [0.49, 0.77] \\
\bottomrule
\end{tabular}%
}
\subcaption{$\ell_1$ distance from $\bx^{(A)}$}
\end{minipage}
\caption{Results for \textbf{gpt-5.4} and \textbf{temperature 0.7}. The empirical setup and construction of the instance (i.e., task, aggregation operation, and prompt/output specification) is the same as in Section \ref{sec:empirical}. These results demonstrate support expansion, binding-set contraction, and feasibility expansion.}
\label{tab:empirical_results_54_temp_07}
\end{table}

\begin{table}[h]
     \centering
     \begin{minipage}{0.48\textwidth}
     \centering
     \resizebox{\textwidth}{!}{%
     \begin{tabular}{lcccc}
     \toprule
     & $x_1$ & $x_2$ & $x_3$ \\
     \midrule
     $\bx^{(1)}$  & $[0.83, 0.89]$ & $[0.00, 0.00]$ & $[0.05, 0.09]$ \\
     $\bx^{(2)}$  & $[0.00, 0.00]$ & $[0.40, 0.48]$ & $[0.26, 0.33]$ \\
     \midrule
     $\bx^{(A)}$  & $[0.41, 0.49]$ & $[0.40, 0.48]$ & $[0.12, 0.18]$ \\
     $\bx^*_P(\bx^{(A)})$  & $[0.36, 0.49]$ & $[0.13, 0.23]$ & $[0.34, 0.47]$ \\
     \bottomrule
     \end{tabular}%
     }
     \subcaption{Support Expansion}
   %  \label{tab:support_expansion}
     \end{minipage}%
     \hfill
     \begin{minipage}{0.48\textwidth}
     \centering
     \resizebox{\textwidth}{!}{%
     \begin{tabular}{lcccc}
     \toprule
     & $x_1$ & $x_2$ & $x_3$ \\
     \midrule
     $\bx^{(1)}$  & $[0.33, 0.40]$ & $[0.71, 0.78]$ & $[0.01, 0.03]$ \\
     $\bx^{(2)}$  & $[0.32, 0.40]$ & $[0.00, 0.02]$ & $[0.79, 0.85]$ \\
     \midrule
     $\bx^{(A)}$  & $[0.02, 0.05]$ & $[0.02, 0.05]$ & $[0.02, 0.05]$ \\
     $\bx^*_P(\bx^{(A)})$  & $[0.00, 0.01]$ & $[0.04, 0.10]$ & $[0.01, 0.04]$ \\
     \bottomrule
     \end{tabular}%
     }
     \subcaption{Feasibility Expansion}
  %  \label{tab:feasibility_expansion}
     \end{minipage}

     \vspace{1em}
     \begin{minipage}{0.6\textwidth}
     \centering
     \resizebox{\textwidth}{!}{%
     \begin{tabular}{lccccc}
     \toprule
     & $x_1$ & $x_2$ & $x_3$ & $x_4$ & $x_5$ \\
     \midrule
     $\bx^{(1)}$ & $[0.60, 0.68]$ & $[0.28, 0.35]$ & $[0.03, 0.06]$ & $[0.00, 0.00]$ & $[0.08, 0.13]$ \\
     $\bx^{(2)}$ & $[0.25, 0.32]$ & $[0.00, 0.01]$ & $[0.66, 0.74]$ & $[0.02, 0.05]$ & $[0.01, 0.03]$ \\
     \midrule
     $\bx^{(A)}$  & $[0.00, 0.00]$ & $[0.00, 0.00]$ & $[0.00, 0.00]$ & $[0.00, 0.00]$ & $[0.00, 0.00]$ \\
     $\bx^*_P(\bx^{(A)})$  & $[0.00, 0.00]$ & $[0.00, 0.00]$ & $[0.00, 0.00]$ & $[0.00, 0.00]$ & $[0.00, 0.00]$ \\
     \bottomrule
     \end{tabular}%
     }
     \subcaption{Binding-set contraction}
   %  \label{tab:binding_contraction}
     \end{minipage}
     \hfill
     \begin{minipage}{0.38\textwidth}
     \centering
     \resizebox{\textwidth}{!}{%
     \begin{tabular}{lc}
     \toprule
     \shortstack[l]{Experiment} & \shortstack{$\|\bx^*_P(\bx^{(A)}) - \bx^{(A)}\|_1$} \\
     \midrule
     Support expansion & [0.32, 0.83] \\
     Feasibility expansion & [0.01, 0.18] \\
     Binding-set contraction & [0.00, 0.00] \\
     \bottomrule
     \end{tabular}%
     }
     \subcaption{$\ell_1$ distance from $\bx^{(A)}$}
     %\label{tab:empirical_angular_distance_summary}
     \end{minipage}
     \caption{Results for \textbf{gpt-5-mini} with the default temperature. The empirical setup and construction of the instance (i.e., task, aggregation operation, and output specification) is the same as in Section 5. For support expansion and feasibility expansion, we use the same prompt specification as in Section \ref{sec:empirical}. For binding-set contraction, the prompts used to generate $x^{(1)}$ and $x^{(2)}$ slightly differ: we set $G_1^{\text{inc}} = \left\{\mathcal{T}_1^P\right\}, G_1^{\text{exc}} = \left\{\mathcal{T}_2^P\right\}$, and we set $G_2^{\text{inc}} = \left\{\mathcal{T}_1^P, \mathcal{T}_2^P\right\}, \text{op}_2^{\text{inc}} = \left\{\text{or}\right\}$. These results demonstrate support expansion and feasibility expansion (modestly), but not binding-set contraction: intersection aggregation produces an empty list ($\bx^{(A)} = \mathbf{0}$), even with the modified prompt specifications described above.}
     \label{tab:empirical_results_5mini_temp_none}
     \end{table}

\FloatBarrier
\subsubsection*{Experiment 3 (E3): Combining heterogeneous models}

\begin{table}[h]
\centering
\begin{minipage}{0.48\textwidth}
\centering
\resizebox{\textwidth}{!}{%
\begin{tabular}{lccc}
\toprule
& $x_{1}$ & $x_{2}$ & $x_{3}$ \\
\midrule
$\bx^{(1)}$  & $[0.62, 0.71]$ & $[0.00, 0.01]$ & $[0.01, 0.04]$ \\
$\bx^{(2)}$  & $[0.00, 0.01]$ & $[0.80, 0.87]$ & $[0.01, 0.03]$ \\
\midrule
$\bx^{(A)}$  & $[0.29, 0.38]$ & $[0.45, 0.54]$ & $[0.01, 0.03]$ \\
$\bx^*_P(\bx^{(A)})$  & $[0.24, 0.31]$ & $[0.29, 0.36]$ & $[0.21, 0.27]$ \\
\bottomrule
\end{tabular}%
}
\subcaption{Support Expansion}
%\label{tab:support_expansion}
\end{minipage}%
\hfill
\begin{minipage}{0.48\textwidth}
\centering
\resizebox{\textwidth}{!}{%
\begin{tabular}{lccc}
\toprule
& $x_{1}$ & $x_{2}$ & $x_{3}$ \\
\midrule
$\bx^{(1)}$  & $[0.66, 0.75]$ & $[0.32, 0.42]$ & $[0.01, 0.04]$ \\
$\bx^{(2)}$  & $[0.37, 0.47]$ & $[0.00, 0.01]$ & $[0.72, 0.80]$ \\
\midrule
$\bx^{(A)}$  & $[0.00, 0.01]$ & $[0.00, 0.01]$ & $[0.00, 0.01]$ \\
$\bx^*_P(\bx^{(A)})$  & $[0.01, 0.04]$ & $[0.00, 0.01]$ & $[0.00, 0.01]$ \\
\bottomrule
\end{tabular}%
}
\subcaption{Feasibility Expansion}
%\label{tab:feasibility_expansion}
\end{minipage}

\vspace{1em}
\begin{minipage}{0.6\textwidth}
\centering
\resizebox{\textwidth}{!}{%
\begin{tabular}{lccccc}
\toprule
& $x_{1}$ & $x_{2}$ & $x_{3}$ & $x_{4}$ & $x_{5}$ \\
\midrule
$\bx^{(1)}$  & $[0.80, 0.87]$ & $[0.08, 0.14]$ & $[0.00, 0.01]$ & $[0.00, 0.01]$ & $[0.00, 0.02]$ \\
$\bx^{(2)}$  & $[0.30, 0.39]$ & $[0.00, 0.01]$ & $[0.56, 0.66]$ & $[0.09, 0.16]$ & $[0.00, 0.02]$ \\
\midrule
$\bx^{(A)}$  & $[0.21, 0.29]$ & $[0.00, 0.01]$ & $[0.00, 0.01]$ & $[0.00, 0.01]$ & $[0.00, 0.01]$ \\
$\bx^*_P(\bx^{(A)})$  & $[0.02, 0.08]$ & $[0.00, 0.01]$ & $[0.01, 0.04]$ & $[0.00, 0.03]$ & $[0.09, 0.18]$ \\
\bottomrule
\end{tabular}%
}
\subcaption{Binding-set contraction}
%\label{tab:binding_contraction}
\end{minipage}
\hfill
\begin{minipage}{0.38\textwidth}
\centering
\resizebox{\textwidth}{!}{%
\begin{tabular}{lc}
\toprule
\shortstack[l]{Experiment} & \shortstack{$\|\bx^*_P(\bx^{(A)}) - \bx^{(A)}\|_1$} \\
\midrule
Support expansion & [0.26, 0.66] \\
Feasibility expansion & [0.00, 0.07] \\
Binding-set contraction & [0.21, 0.54] \\
\bottomrule
\end{tabular}%
}
\subcaption{$\ell_1$ distance from $\bx^{(A)}$}
%\label{tab:empirical_angular_distance_summary}
\end{minipage}

\caption{Results for aggregating two different models: \textbf{gpt-4o-mini} (temp 0) and \textbf{gpt-5.4} (temp 0). {Support expansion and binding set contraction seems to work. The empirical setup and construction of the instance (i.e., task, aggregation operation, and prompt/output specification) is the same as in Section 5. $\bx^*_P(\bx^{(A)})$ is chosen via brute force search over prompt topics and over the two models.  These results demonstrate support expansion, binding-set contraction. Feasibility expansion does not appear to occur for this instance due to intersection resulting in empty lists.}
}
\label{tab:empirical_results_4omini_temp_00_x_54_temp_00}
\end{table}

\begin{table}[h]
\centering
\begin{minipage}{0.48\textwidth}
\centering
\resizebox{\textwidth}{!}{%
\begin{tabular}{lccc}
\toprule
& $x_{1}$ & $x_{2}$ & $x_{3}$ \\
\midrule
$\bx^{(1)}$  & $[0.70, 0.78]$ & $[0.00, 0.01]$ & $[0.04, 0.09]$ \\
$\bx^{(2)}$  & $[0.00, 0.01]$ & $[0.76, 0.83]$ & $[0.01, 0.04]$ \\
\midrule
$\bx^{(A)}$  & $[0.32, 0.41]$ & $[0.45, 0.55]$ & $[0.03, 0.08]$ \\
$\bx^*_P(\bx^{(A)})$  & $[0.00, 0.01]$ & $[0.42, 0.55]$ & $[0.21, 0.33]$ \\
\bottomrule
\end{tabular}%
}
\subcaption{Support Expansion}
%\label{tab:support_expansion}
\end{minipage}%
\hfill
\begin{minipage}{0.48\textwidth}
\centering
\resizebox{\textwidth}{!}{%
\begin{tabular}{lccc}
\toprule
& $x_{1}$ & $x_{2}$ & $x_{3}$ \\
\midrule
$\bx^{(1)}$  & $[0.66, 0.75]$ & $[0.31, 0.41]$ & $[0.01, 0.05]$ \\
$\bx^{(2)}$  & $[0.30, 0.40]$ & $[0.00, 0.02]$ & $[0.74, 0.82]$ \\
\midrule
$\bx^{(A)}$  & $[0.00, 0.01]$ & $[0.00, 0.01]$ & $[0.00, 0.01]$ \\
$\bx^*_P(\bx^{(A)})$  & $[0.01, 0.04]$ & $[0.00, 0.01]$ & $[0.00, 0.01]$ \\
\bottomrule
\end{tabular}%
}
\subcaption{Feasibility Expansion}
%\label{tab:feasibility_expansion}
\end{minipage}

\vspace{1em}
\begin{minipage}{0.6\textwidth}
\centering
\resizebox{\textwidth}{!}{%
\begin{tabular}{lccccc}
\toprule
& $x_{1}$ & $x_{2}$ & $x_{3}$ & $x_{4}$ & $x_{5}$ \\
\midrule
$\bx^{(1)}$  & $[0.59, 0.69]$ & $[0.28, 0.37]$ & $[0.01, 0.04]$ & $[0.00, 0.01]$ & $[0.07, 0.12]$ \\
$\bx^{(2)}$  & $[0.15, 0.22]$ & $[0.00, 0.01]$ & $[0.68, 0.76]$ & $[0.05, 0.10]$ & $[0.00, 0.02]$ \\
\midrule
$\bx^{(A)}$  & $[0.99, 1.00]$ & $[0.00, 0.01]$ & $[0.00, 0.01]$ & $[0.00, 0.01]$ & $[0.00, 0.01]$ \\
$\bx^*_P(\bx^{(A)})$  & $[0.78, 0.88]$ & $[0.07, 0.16]$ & $[0.00, 0.01]$ & $[0.00, 0.03]$ & $[0.00, 0.04]$ \\
\bottomrule
\end{tabular}%
}
\subcaption{Binding-set contraction}
%\label{tab:binding_contraction}
\end{minipage}
\hfill
\begin{minipage}{0.38\textwidth}
\centering
\resizebox{\textwidth}{!}{%
\begin{tabular}{lc}
\toprule
\shortstack[l]{Experiment} & \shortstack{$\|\bx^*_P(\bx^{(A)}) - \bx^{(A)}\|_1$} \\
\midrule
Support expansion & [0.43, 0.84] \\
Feasibility expansion & [0.00, 0.07] \\
Binding-set contraction & [0.18, 0.46] \\
\bottomrule
\end{tabular}%
}
\subcaption{$\ell_1$ distance from $\bx^{(A)}$}
%\label{tab:empirical_angular_distance_summary}
\end{minipage}
\caption{Results for aggregating two different models: \textbf{gpt-4o-mini} (temp 0) and \textbf{gpt-5-mini} (temp default).  The empirical setup and construction of the instance (i.e., task, aggregation operation, and prompt/output specification) is the same as in Section 5. $\bx^*_P(\bx^{(A)})$ is chosen via brute force search over prompt topics and over the two models. These results demonstrate support expansion, binding-set contraction. Feasibility expansion does not appear to occur for this instance due to intersection resulting in empty lists.
}
\label{tab:empirical_results_4omini_temp_00_x_5mini_temp_none}
\end{table}

\begin{table}[h]
\centering
\begin{minipage}{0.48\textwidth}
\centering
\resizebox{\textwidth}{!}{%
\begin{tabular}{lccc}
\toprule
& $x_{1}$ & $x_{2}$ & $x_{3}$ \\
\midrule
$\bx^{(1)}$  & $[0.70, 0.78]$ & $[0.00, 0.01]$ & $[0.04, 0.09]$ \\
$\bx^{(2)}$  & $[0.00, 0.01]$ & $[0.36, 0.45]$ & $[0.15, 0.23]$ \\
\midrule
$\bx^{(A)}$  & $[0.32, 0.41]$ & $[0.45, 0.54]$ & $[0.05, 0.10]$ \\
$\bx^*_P(\bx^{(A)})$  & $[0.24, 0.31]$ & $[0.29, 0.36]$ & $[0.21, 0.27]$ \\
\bottomrule
\end{tabular}%
}
\subcaption{Support Expansion}
% \label{tab:support_expansion}
\end{minipage}%
\hfill
\begin{minipage}{0.48\textwidth}
\centering
\resizebox{\textwidth}{!}{%
\begin{tabular}{lccc}
\toprule
& $x_{1}$ & $x_{2}$ & $x_{3}$ \\
\midrule
$\bx^{(1)}$  & $[0.34, 0.43]$ & $[0.72, 0.80]$ & $[0.01, 0.03]$ \\
$\bx^{(2)}$  & $[0.39, 0.48]$ & $[0.00, 0.01]$ & $[0.69, 0.77]$ \\
\midrule
$\bx^{(A)}$  & $[0.30, 0.40]$ & $[0.21, 0.29]$ & $[0.21, 0.29]$ \\
$\bx^*_P(\bx^{(A)})$  & $[0.41, 0.55]$ & $[0.25, 0.37]$ & $[0.36, 0.50]$ \\
\bottomrule
\end{tabular}%
}
\subcaption{Feasibility Expansion}
% \label{tab:feasibility_expansion}
\end{minipage}

\vspace{1em}
\begin{minipage}{0.6\textwidth}
\centering
\resizebox{\textwidth}{!}{%
\begin{tabular}{lccccc}
\toprule
& $x_{1}$ & $x_{2}$ & $x_{3}$ & $x_{4}$ & $x_{5}$ \\
\midrule
$\bx^{(1)}$  & $[0.62, 0.71]$ & $[0.25, 0.34]$ & $[0.02, 0.05]$ & $[0.00, 0.01]$ & $[0.05, 0.10]$ \\
$\bx^{(2)}$  & $[0.32, 0.41]$ & $[0.00, 0.01]$ & $[0.55, 0.64]$ & $[0.09, 0.16]$ & $[0.00, 0.01]$ \\
\midrule
$\bx^{(A)}$  & $[0.16, 0.24]$ & $[0.00, 0.01]$ & $[0.00, 0.01]$ & $[0.00, 0.01]$ & $[0.00, 0.01]$ \\
$\bx^*_P(\bx^{(A)})$  & $[0.15, 0.26]$ & $[0.00, 0.01]$ & $[0.00, 0.01]$ & $[0.00, 0.01]$ & $[0.41, 0.55]$ \\
\bottomrule
\end{tabular}%
}
\subcaption{Binding-set contraction}
% \label{tab:binding_contraction}
\end{minipage}
\hfill
\begin{minipage}{0.38\textwidth}
\centering
\resizebox{\textwidth}{!}{%
\begin{tabular}{lc}
\toprule
\shortstack[l]{Experiment} & \shortstack{$\|\bx^*_P(\bx^{(A)}) - \bx^{(A)}\|_1$} \\
\midrule
Support expansion & [0.20, 0.66] \\
Feasibility expansion & [0.08, 0.70] \\
Binding-set contraction & [0.40, 0.69] \\
\bottomrule
\end{tabular}%
}
\subcaption{$\ell_1$ distance from $\bx^{(A)}$}
% \label{tab:empirical_angular_distance_summary}
\end{minipage}
\caption{Results for aggregating two different models:  \textbf{gpt-5.4} (temp 0) and \textbf{gpt-5-mini} (temp default). The empirical setup and construction of the instance (i.e., task, aggregation operation, and prompt/output specification) is the same as in Section 5. $\bx^*_P(\bx^{(A)})$ is chosen via brute force search over prompt topics and over the two models. These results demonstrate support expansion, binding-set contraction, and feasibility expansion.
}
\label{tab:empirical_results_54_temp_00_x_5mini_temp_none}
\end{table}